\documentclass[10pt,journal,compsoc]{IEEEtran}

\ifCLASSOPTIONcompsoc
\usepackage[nocompress]{cite}
\else
\usepackage{cite}
\fi

\ifCLASSINFOpdf
  \usepackage[pdftex]{graphicx}
  \graphicspath{{fig/}}
  \DeclareGraphicsExtensions{.pdf,.jpeg,.png}
\else
\fi

\usepackage[cmex10]{amsmath}
\usepackage{amssymb} 
\usepackage{amsthm} 

\usepackage{algorithm}
\usepackage{algorithmic}

\usepackage{array}

\usepackage{mdwmath}
\usepackage{mdwtab}
\usepackage{eqparbox}

\usepackage{comment}
\usepackage{xcolor}

\ifCLASSOPTIONcompsoc
  \usepackage[caption=false,font=footnotesize,labelfont=sf,textfont=sf]{subfig}
\else
  \usepackage[caption=false,font=footnotesize]{subfig}
\fi
\usepackage{multirow}
\usepackage{color,soul}
\usepackage{comment}
\usepackage{enumitem}
\usepackage{framed}

\usepackage{tabularx} 

\newcolumntype{L}[1]{>{\hsize=#1\hsize\raggedright\arraybackslash}X}%
\newcolumntype{R}[1]{>{\hsize=#1\hsize\raggedleft\arraybackslash}X}%
\newcolumntype{C}[1]{>{\hsize=#1\hsize\centering\arraybackslash}X}%

\newtheorem{theorem}{Theorem}

\newtheorem{result}{Result}

\newcommand{\bR}{\mathbf{R}}
\newcommand{\bx}{\mathbf{x}}
\newcommand{\by}{\mathbf{y}}

\newcommand{\ba}{\mathbf{a}}

\newcommand{\bA}{\mathbf{A}}
\newcommand{\bB}{\mathbf{B}}
\newcommand{\bp}{\mathbf{p}}

\newcommand{\bu}{\mathbf{u}}

\newcommand{\bt}{\mathbf{t}}

\newcommand{\cI}{\mathcal{I}}

\newcommand{\cH}{\mathcal{H}}
\newcommand{\cX}{\mathcal{X}}
\newcommand{\cY}{\mathcal{Y}}

\def\exp{\operatorname*{exp\,}}

\def\SO3{SO(3)}

\newcommand{\bT}{\mathbf{T}}

\newcommand*\rot{\rotatebox{90}}

\def\exp{\operatorname*{exp\,}}

\usepackage{url}

\usepackage{mathtools}
\usepackage{multirow}
\usepackage{soul}
\usepackage{gensymb}
\usepackage{textcomp}
\usepackage{comment}
\usepackage[percent]{overpic}

\begin{document}
\title{Guaranteed Outlier Removal for Point Cloud Registration with Correspondences}

\author{\'{A}lvaro~{Parra~Bustos},
         Tat-Jun~Chin,~\IEEEmembership{Member,~IEEE}
\IEEEcompsocitemizethanks{
\IEEEcompsocthanksitem Parra Bustos and Chin are with School of Computer Science, The University of Adelaide.\protect\\
E-mail: \{alvaro.parrabustos, tat-jun.chin\}@adelaide.edu.au
}
}


\IEEEcompsoctitleabstractindextext{%
\begin{abstract}
An established approach for 3D point cloud registration is to estimate the registration function from 3D keypoint correspondences. Typically, a robust technique is required to conduct the estimation, since there are false correspondences or outliers. Current 3D keypoint techniques are much less accurate than their 2D counterparts, thus they tend to produce extremely high outlier rates. A large number of putative correspondences must thus be extracted to ensure that sufficient good correspondences are available. Both factors (high outlier rates, large data sizes) however cause existing robust techniques to require very high computational cost. In this paper, we present a novel preprocessing method called \emph{guaranteed outlier removal} for point cloud registration. Our method reduces the input to a smaller set, in a way that any rejected correspondence is guaranteed to not exist in the globally optimal solution. The reduction is performed using purely geometric operations which are deterministic and fast. Our method significantly reduces the population of outliers, such that further optimization can be performed quickly. Further, since only true outliers are removed, the globally optimal solution is preserved. On various synthetic and real data experiments, we demonstrate the effectiveness of our preprocessing method. \emph{Demo code is available as supplementary material}.

\end{abstract}

}

\maketitle

\IEEEdisplaynontitleabstractindextext

%
\IEEEpeerreviewmaketitle

\IEEEraisesectionheading{\section{Introduction}}

\IEEEPARstart{P}oint cloud registration is a core operation in computer vision and robotics. In general, point cloud registration is required whenever there is a need to integrate 3D measurements from different viewpoints or time steps. Given two point clouds $\cX$ and $\cY$, the aim is to find a transformation function $f$ that maps $\cX$ to the reference frame of $\cY$, in a way that the points are as ``aligned" as possible. In this work, we focus on rigidly moving point clouds, i.e., $f$ is a rotation or a Euclidean/rigid transformation.

A popular paradigm for point cloud registration is to estimate $f$ from a set of point correspondences or matches extracted using a 3D keypoint technique~\cite{tam13}. Let $\{(\bx_i, \by_i)\}_{i=1}^N$ denote a set of point matches between $\cX$ and $\cY$. If there are no false matches or outliers, the best $f$ in the least squares sense can usually be obtained analytically~\cite{horn87,arun87}. However, most automatic 3D keypoint matching methods~\cite{tombari13,rusu08,zhong09} invariably produce mismatches. Thus $f$ must be estimated robustly from the input correspondences.

To this end, many practitioners apply the maximum consensus approach, i.e., find the $f$ that is consistent up to threshold $\xi$ with as many of the input matches as possible
\begin{align}\label{eq:pcreg}
	\begin{aligned}
		&\underset{\Omega, \; \cI \subseteq \cH}{\text{maximize}}
		& & \left| \cI \right| \\ 
		&\text{subject to}
		& & \| f(\bx_i \mid \Omega) - \by_i\|\leq \xi, \; \forall i \in \cI,
	\end{aligned}
\end{align}
where $\Omega$ represents the parameters of $f$, $\cH = \{1,\dots,N\}$ are indices of the input data, and $\|\cdot\|$ denotes the Euclidean norm. The subset $\cI$ is often called a \emph{consensus set}. The optimal $\Omega^\ast$ to~\eqref{eq:pcreg} enables $f$ to be consistent with the largest possible consensus set $\cI^\ast$. Hereafter, we make no distinction between the indices and the data referred to by the indices. Also, primarily for convenience, we call the data in $\cI^\ast$ \emph{true inliers}, and the data in $\cH\setminus \cI^\ast$ \emph{true outliers}.
 
RANSAC~\cite{fischler81} is the standard approach to solve maximum consensus problems. In the context of~\eqref{eq:pcreg}, RANSAC randomly samples minimal subsets of size $m$ ($m = 2$ for rotation, $m = 3$ for rigid transformation) and estimates $\Omega$ from the minimal subsets (usually using least squares). The $\Omega$ with the largest consensus set is returned as the solution. The number of samples to generate is typically taken as
\begin{align}\label{eq:ransacstop}
\left\lceil \frac{\log (1-\rho)}{\log (1 - (1-\eta)^m)} \right\rceil,
\end{align}
where $\eta$ is the outlier proportion. Formula~\eqref{eq:ransacstop} is derived with the aim of retrieving at least one minimal subset that consists purely of inliers with probability $\rho$. Since $\eta$ is unknown initially, it is estimated on-the-fly based on the current best result $\tilde{\cI}$, thus~\eqref{eq:ransacstop} is merely a probabilistic \emph{lower bound} of the runtime of RANSAC.

\begin{figure*}[t]
\centering
\subfloat[Input correspondence set $\cH$.]{\includegraphics[width=0.16\textwidth]{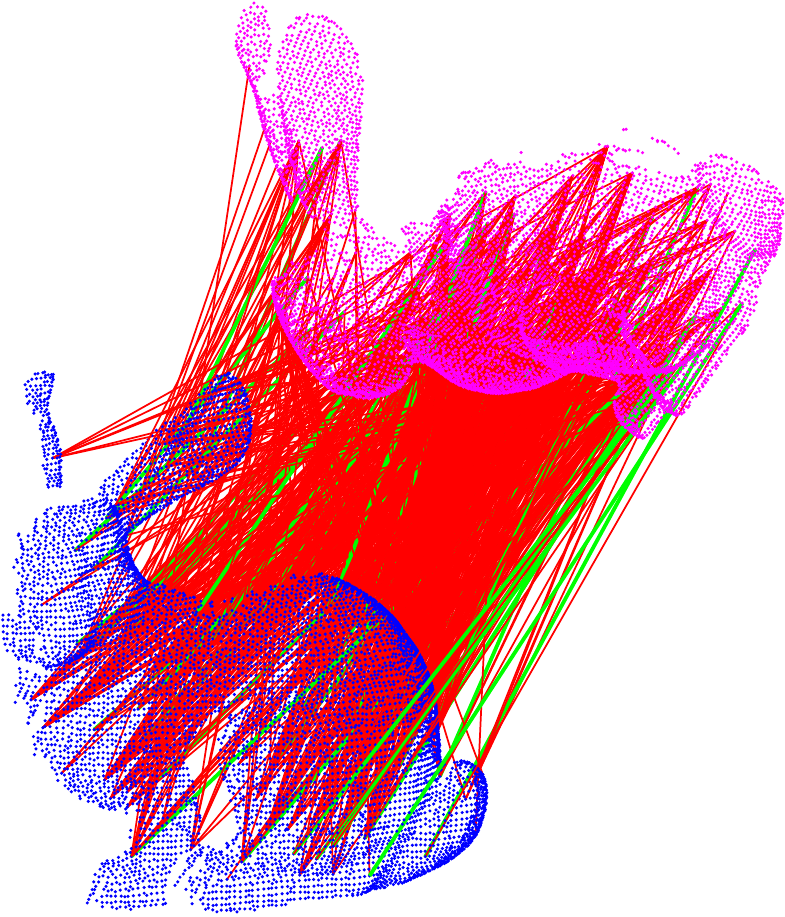}}
\hspace{5em}
\subfloat[Reduced correspondence set $\cH^\prime$ by GORE.]{
\begin{minipage}[t]{0.16\textwidth}\centering 
	\includegraphics[width=.97\textwidth]{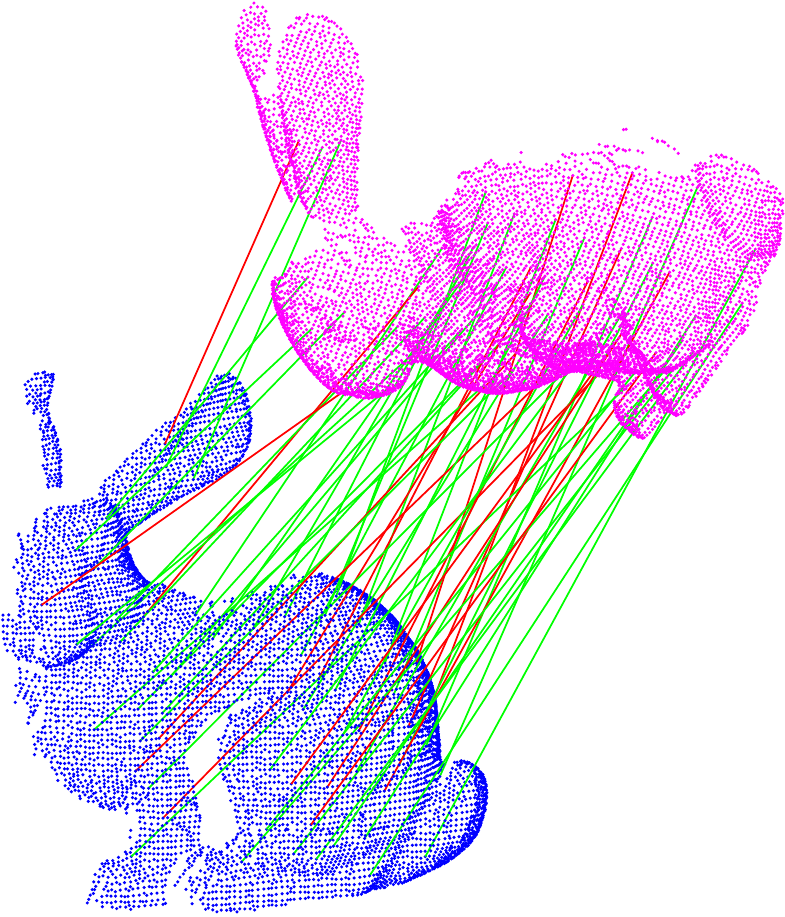}
\end{minipage}
}
\hspace{5em}
\subfloat[RANSAC's alignment on $\cH^\prime$.]{\includegraphics[width=0.171\textwidth]{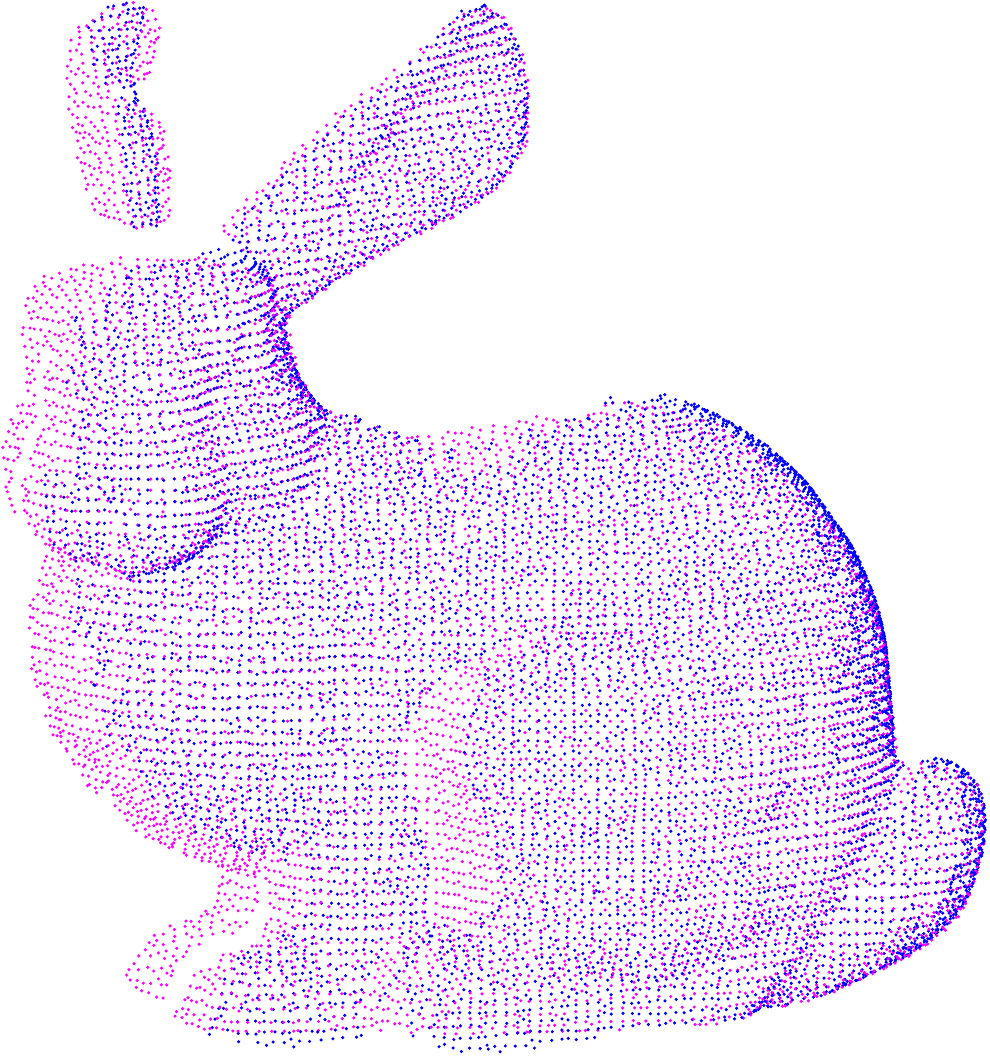}}
\caption{Illustrating the concept of GORE for robust point cloud registration~\eqref{eq:pcreg}. In the above diagrams, green indicates true inliers $\cI^\ast$ and red indicates true outliers $\cH \setminus \cI^\ast$. In this instance, the input $\cH$ contains $N = 2000$ correspondences with outlier rate $\eta = 0.98$. In $4.62$ seconds, GORE reduced the input to a smaller correspondence set $\cH^\prime$ of size $60$ with $\eta = 0.28$. Using $\rho = 0.99$ in the stopping criterion~\eqref{eq:ransacstop}, executing RANSAC on $\cH^\prime$ terminated in $0.005$ seconds. Executing RANSAC on the original input $\cH$, however, required $132$ seconds to converge.}
	\label{fig:preview}
\end{figure*}

Although RANSAC is generally efficient, its runtime as predicted by~\eqref{eq:ransacstop} increases exponentially with the outlier ratio $\eta$. The sampling stage thus becomes a significant bottleneck in point cloud registration, since the accuracy of 3D keypoint matching methods is generally much lower than their 2D image counterparts such as SIFT~\cite{lowe04} and SURF~\cite{bay06} , due to inherent difficulties in the 3D case such as irregular point cloud density and lack of useful texture on the points. As we shall see later in Sec.~\ref{sec:results}, outlier ratios of more than $95\%$ is quite common in real point clouds.

Another fundamental limitation of RANSAC is that it does not provide the optimal solution in general. Formally, let $\tilde{\Omega}$ be the result of RANSAC, and $\tilde{\cI} \subseteq \cH$ be the associated consensus set. We have that $|\tilde{\cI}| \le |\cI^\ast|$, and in general $\tilde{\cI} \nsubseteq \cI^\ast$, i.e., true inliers may be discarded by RANSAC.

There exist globally optimal algorithms for point cloud registration. Most of the global methods~\cite{olsson09,enqvist08,hartley09,bazin12,yang13,parrabustos16_pami} employ branch-and-bound (BnB)~\cite{horst03} to systematically search the parameter space ($SO(3)$ for rotation, $SE(3)$ for rigid transformation) to deterministically optimize their respective objective function. Another class of global algorithms~{\cite{olsson08,enqvist12,enqvist14}} leverage the fact that the optimal solution to a robust  objective function can be obtained as the solution of the same problem on a $d$-subset of the $N$ input correspondences $\cH$. For example, for 6 DoF rigid registration~\eqref{eq:pcreg}, $d$ is equal to $6$. To find $\Omega^\ast$, the methods~{\cite{olsson08,enqvist12,enqvist14}} enumerate all {$\binom{N}{d}$} subsets of {$\cH$} and solve~\eqref{eq:pcreg} on each $d$-subset analytically (note that this approach differs from ``standard" RANSAC which solves for $\Omega$ via least squares on sampled minimal subsets of size $m$). 

A general weakness of global algorithms is their high computational cost, especially on data with large sizes $N$ and high outlier contamination rates. Theoretically, the runtime of BnB increases exponentially with the input size. The significant cost of BnB is also exacerbated by large outlier ratios, since outliers tend to increase local optima. As we will show in Sec.~\ref{sec:results}, runtimes in excess of several hours can be consumed by BnB when invoked on real correspondence sets of moderate size ($N = 1000$). In the second group of global algorithms~\cite{olsson08,enqvist12,enqvist14}, the number of unique subsets to test is impracticably large for practical input sizes, e.g., for $N = 1000$ there are $\ge 100$ Trillion $6$-subsets. Thus, such methods are feasible only on very small instances.

\subsection{Our contributions}

We propose a novel \emph{guaranteed outlier removal (GORE)} technique for robust point cloud registration~\eqref{eq:pcreg}. Specifically, our method is able to efficiently reduce $\cH$ to a subset $\cH^\prime$ of point matches, in a way that any correspondence $(\bx_i, \by_i)$ discarded by reducing $\cH$ to $\cH^\prime$ is a true outlier, i.e., any $(\bx_i, \by_i)$ that is removed does not belong to $\cI^\ast$. More formally, our method ensures that
\begin{align}\label{eq:preprop}
\cI^\ast \subseteq \cH^\prime \subseteq \cH.
\end{align}
Note that this fundamentally differs from the objective of RANSAC, which is to return a suboptimal consensus set $\tilde{\cI}$. See Fig.~\ref{fig:preview} for an illustration of the concept of GORE.

Based on computationally simple geometric consistency checks, our GORE technique is deterministic and efficient, e.g., the result in Fig.~\ref{fig:preview} was produced in $4.62$ seconds. Though GORE may not remove all true outliers, i.e., $\cH^\prime$ may not be a consensus set, it is able to aggressively reduce the population of true outliers\footnote{ Note that removing all true outliers amounts to solving the original maximum consensus problem~\eqref{eq:pcreg}, which is intractable. Thus, except in trivial instances, GORE does not generally remove all true outliers.}. On real data, almost $90\%$ of the true outliers can be eliminated (see results in Table~\ref{tab:comparison6dof} for mining and remote sensing datasets). We pose our technique as an \emph{efficient preprocessor} to robust point cloud registration, in that GORE reduces the amount of data and the ratio of outliers in the input $\cH$, without removing true inliers.

The output $\cH^\prime$ can be used in the following ways:
\begin{itemize}[parsep=1pt]
\item Since the ratio of outliers is much lower in $\cH^\prime$, approximate maximum consensus algorithms such as RANSAC, whose runtime increases exponentially with outlier rate, can also benefit. Practically, preprocessing with GORE enables RANSAC to return good solutions in more reasonable times.
\item Due to condition~\eqref{eq:preprop}, the globally optimal solution $\Omega^\ast$ is preserved in $\cH^\prime$. Thus, one can conduct further optimization on $\cH^\prime$ to find $\Omega^\ast$. By applying GORE to drastically reduce the amount of data and outlier ratio, the \emph{total} runtime can be reduced significantly (by more than an order of magnitude on real data).
\end{itemize}
Comprehensive experimental evaluation is provided in Sec.~\ref{sec:results} to demonstrate the above claims\footnote{Matlab implementation is available as supplementary material.}.

As a by-product, a sub-optimal solution $\tilde{\Omega}$ is also returned by GORE. We empirically show that this sub-optimal transformation represents an approximate solution of comparable quality to the result of RANSAC. However, contrary to RANSAC, our method is deterministic.


\textbf{Previous work.} This work is an extension of our previous conference paper~\cite{parrabustos15-1} which proposed GORE for 3 DoF rotational registration. Here, we propose a novel broader GORE algorithm for 6 DoF Euclidean registration, which significantly increases the practical usefulness of GORE.

\subsection{Related work: preprocessing for robust estimation}\label{sec:survey}

A priori reducing the data size before conducting optimization is an age-old idea in computer science. In the context of robust estimation problems such as~\eqref{eq:pcreg}, one could first execute RANSAC with a large $\xi$ to remove hopefully most of the outlying data, before performing a more careful optimization (e.g., using BnB) on the remaining data. However, such a heuristic ``pipeline" does not provide any optimality guarantees, and worse, the rough initial culling could unintentionally remove many of the useful data such that a decent result is no longer attainable.

More closely in the area of robust registration from point correspondences, preprocessing methods that quickly reject outliers based on geometric consistency have been proposed. One of the earliest work with this flavor is by Adam et al.~\cite{adam01} who targeted the epipolar geometry estimation problem. Given two images with putative feature correspondences, they seek a rotation on one of the images such that the ``flow vectors" of the correspondences are as aligned as possible. Correspondences that are not aligned are removed as outliers. Adam et al.~showed that their method can significantly lower the outlier rate in the input data. However, their technique is heuristic, thus there are no strict guarantees on the preservation of solution optimality.

Another relevant work is by Sv\"{a}rm et al.~\cite{svarm14}, who proposed a technique for camera localization from 2D-3D correspondences. In their work, the usage of gravitational sensors reduces camera localization to a 3 DoF problem (2D translation and 1D rotation). Their approach also conducts a guaranteed outlier rejection scheme for the 2D-3D point matches, before a globally optimal algorithm is invoked. The underlying geometric consistency check is based on analysing the planar projection of backprojected viewing cones using Minkowski sums. Since our target problems (3 DoF rotational and 6 DoF rigid registration) differ from Sv\"{a}rm et al.'s, the core geometric motivations and operations of our work are vastly different from theirs.

Another previous work~\cite{chin16} investigated the usage of mixed integer linear programming (MILP) to perform GORE on robust linear regression. Apart from differences in intended applications (their method cannot be utilized for rotational and rigid registration), \cite{chin16} also ``outsourced" the calculation of bounds required for GORE to a MILP solver. This contrasts greatly with the present work which studies and exploits the problem geometry for consistency checks.

\subsection{Recent work on point cloud registration}

Note that we focus on correspondence-based registration, which differs from registering ``raw" point clouds~\cite{chen99,aiger08}.

Zhou et al.~\cite{zhou16} conducted robust registration under the outlier removal framework of~\cite{black96}. Their idea is to alleviate the effects of local optima by minimizing a continuation of increasingly good approximations of a robust objective function. The method locally solves each objective function from the previous objective value. Hence, there are no global optimality guarantees.

Petrelli and Di Stefano~\cite{petrelli16} used an heuristic based on the alignment of local reference frames~\cite{tombari10} as a preprocessor of RANSAC. The method relies on a unique local reference frame for 3D features, such that every correspondence induces a local transformation. A suboptimal consensus is obtained by using a 3D Hough voting scheme over a transformed representative point (of the point cloud to be aligned) when applying the induced transformations. Other heuristics based on unique local reference frames could also be applied, e.g.: clustering the induced transformations~\cite{mian10}. Unlike GORE, these heuristics do not guarantee preservation of solution optimality. As observed in~\cite{petrelli16}, its method may fail under the presence of incorrect local frames. 




Albarelli et al.~\cite{albarelli10,albarelli09,rodala12,rodola13} proposed a point cloud registration method based on evolutionary game theory~\cite{rotabulo11}. Their objective is to find the subset of correspondences that maximize average pairwise consistency, defined by a payoff matrix---thus, their notion of "inliers" differs from ours~\eqref{eq:pcreg}. The solution is obtained as an \emph{evolutionary stable state} of an inlier selection process. Their algorithm is stochastic and has no global optimality guarantees, unlike GORE. In Sec.~\ref{sec:albarelli}, we will compare GORE with~\cite{albarelli10} as a representative of a state-of-the-art correspondence-based point cloud registration method.

\section{Theory of GORE}

The robust registration problem~\eqref{eq:pcreg} can be rewritten as
\begin{align}\label{eq:mcs2}
\underset{k \in \cH}{\text{maximize}} \; \; p_k,
\end{align}
where $p_k$ is defined as the maximum objective value of the subproblem $P_k$, for $k = 1,\dots,N$:
\begin{align}\label{eq:fk}
\begin{aligned}
&\underset{\Omega_k,\; \cI_k \subseteq \cH \setminus \{ k \}}{\text{maximize}}
& & \left| \cI_k \right| +1 \\ 
& \text{subject to}
& & \|f(\bx_i \mid \Omega_k)- \by_i\| \leq \xi, \; \forall i \in \cI_k, \\
& & & \|f(\bx_k \mid \Omega_k)- \by_k\|\leq \xi.
\end{aligned}
\tag{$P_k$}
\end{align}
In words, $P_k$ seeks the transformation $f$ that agrees with as many of the data as possible, given that $f$ \emph{must align} $(\bx_k,\by_k)$. Our reformulation~\eqref{eq:mcs2} does not make the original problem~\eqref{eq:pcreg} any easier - its utility derives from showing how an upper bound on the quality $p_k$ of the subproblem $P_k$ allows to identify true outliers.

Let $l\leq |\cI^*|$ be a lower bound for the solution of the original problem~\eqref{eq:pcreg}. Note that any suboptimal solution $\tilde{\cI}$ to~\eqref{eq:pcreg} can give rise to $l$. Let $\hat{p}_k$ be an upper bound on the quality of $P_k$, i.e., $\hat{p}_k \geq p_k$. Given the lower and upper bound values, the following result can be established.

\begin{theorem}\label{theo:reject}
If $\hat{p}_k < l$, then $(\bx_k,\by_k)$ is a true outlier, i.e., $k$ does not exist in the solution $\cI^\ast$ to~\eqref{eq:pcreg}.
\end{theorem}
\begin{proof}
The proof is by contradiction. If $k$ is in $\cI^\ast$, then we must have that $p_k = |\cI^\ast|$. However, if we are given that $\hat{p}_k < l$, then $p_k < l \le |\cI^\ast|$, which contradicts the previous condition. Hence, $k$ cannot exist in $\cI^\ast$.
\end{proof}

The primary computation performed by GORE is to obtain $l$ and $\hat{p}_k$ for $k = 1,\dots,N$, such that the test based on Theorem~\ref{theo:reject} can be attempted to reject true outliers. In Sec.~\ref{sec:gore-r}, we describe how to calculate $\hat{p}_k$ efficiently for rotational registration. In Sec.~\ref{sec:gore-main}, a novel algorithm to efficiently calculate $\hat{p}_k$ for 6 DoF rigid registration is described. Apart from the calculations of $\hat{p}_k$, in the following sections, we also present heuristic procedures that rapidly tighten the bounds $l$ and $\hat{p}_k$ as the rejection tests are being iterated over $k$. As a by-product of our GORE algorithms, good suboptimal solutions (of quality $l$) are also produced.

\section{GORE for rotational registration}\label{sec:gore-r}

To promote better flow of the ideas, we first present GORE for rotational registration. The algorithm presented here will form a ``kernel" for the 6 DoF case in Sec.~\ref{sec:gore-main}.

As a specialization of~\eqref{eq:pcreg}, the maximum consensus formulation for rotational registration is 
\begin{align}\label{eq:rotsearch}
\begin{aligned}
&\underset{\bR \in SO(3), \; \cI \subseteq \cH}{\text{maximize}}
& & \left| \cI \right| \\ 
&\text{subject to}
& & \| \bR\bx_i - \by_i\|\leq \xi, \; \forall i \in \cI.
\end{aligned}
\end{align}
It is clear that, since rotational motion is more ``restricted", any correspondence $(\bx_i,\by_i)$ where
\begin{align}\label{eq:normprune}
\left| \| \bx_i \| - \| \by_i \| \right| > \xi
\end{align}
cannot be rotationally aligned and thus does not affect the result. We assume that such data have been removed.

We further rewrite~\eqref{eq:rotsearch} to become
\begin{align}\label{eq:rotsearch2}
\begin{aligned}
&\underset{\bR \in\SO3, \; \cI \subseteq \cH}{\text{maximize}}
& & \left| \cI \right|\\
&\text{subject to}
& & \angle(\bR \bx_i, \by_i) \leq \epsilon_i, \; \forall i \in \cI,
\end{aligned}
\end{align}
where we have changed the error metric from Euclidean to angular. Further, the inlier threshold $\epsilon_i$ (in radians) is now data dependent and is defined as
\begin{align}\label{eq:defineeps}
\epsilon_i = 
\begin{cases}
\arccos (\lambda_i)  & \text{if}\;\;  \lambda_i > -1 \\
\pi & \text{otherwise,} 
\end{cases}
\end{align}
where 
\begin{align}
	\lambda_i := \dfrac{ \|\bx_i\|^2 + \|\by_i\|^2 - \xi^2}{2 \|\bx_i\| \|\by_i\|}.
\end{align}
See Fig.~\ref{fig:eps} for a depiction of the principle behind calculating $\epsilon_i$. By construction, problems~\eqref{eq:rotsearch} and~\eqref{eq:rotsearch2} are equivalent. Further, since the norm of the points do not matter under the angular error, in the rest of Sec.~\ref{sec:gore-r}, we regard all points to have unit norm.

\begin{figure}[ht]\centering
\subfloat[]{\includegraphics[width=.54\linewidth]{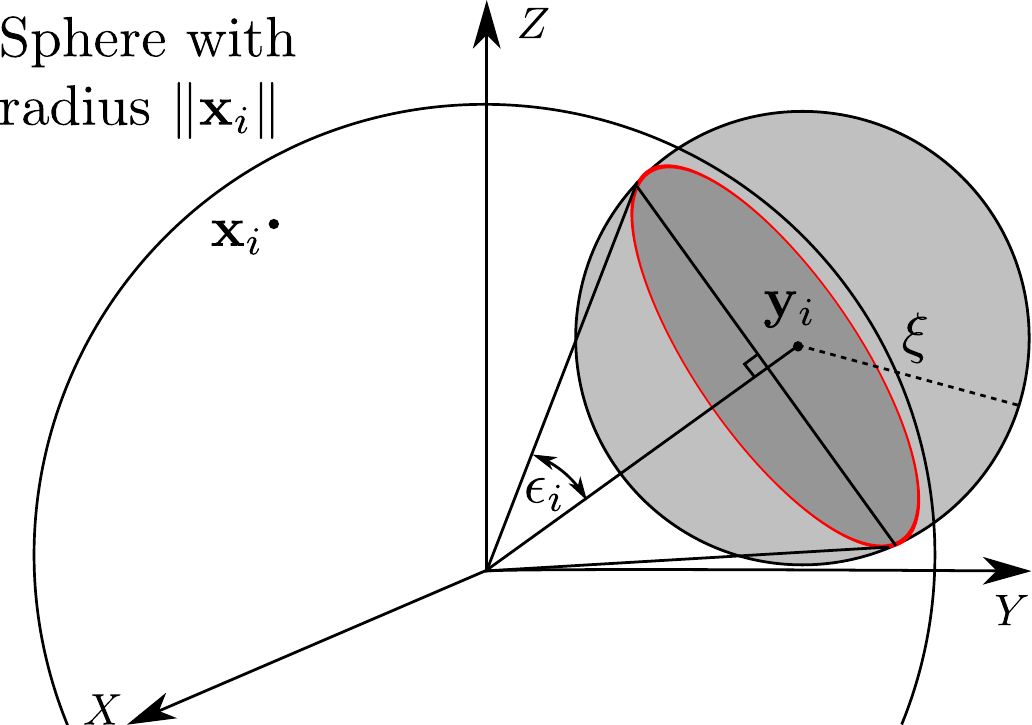}} \hfill
\subfloat[]{\includegraphics[width=.34\linewidth]{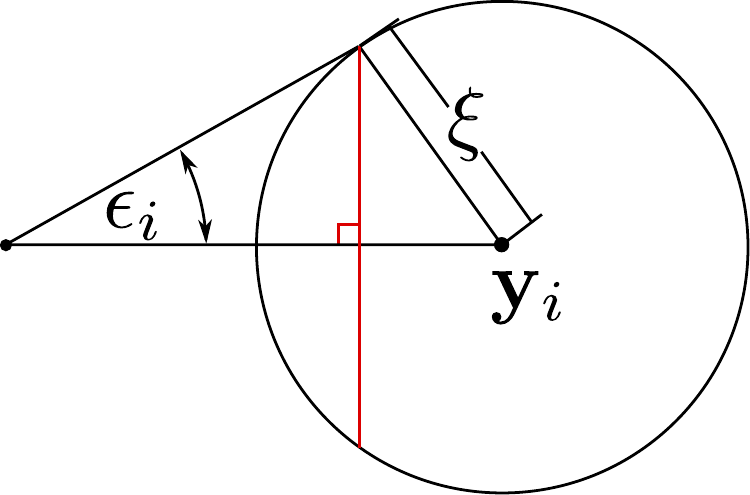}}
\caption{Relating the Euclidean and the angular error. (a) A solid ball intersects the surface of a sphere at a spherical patch, which has a circular outline in the sphere (highlighted in red). (b) Side view of the circular outline.}
\label{fig:eps}
\end{figure}

Specializing the subproblem~\ref{eq:fk} to rotational registration~\eqref{eq:rotsearch2}, we obtain
\begin{align}\label{eq:frk}
\begin{aligned}
&\underset{\bR_k\in\SO3,\; \cI_k \subseteq \cH \setminus \{ k \}}{\text{maximize}}
& & \left| \cI_k \right| +1 \\ 
& \text{subject to}
& & \angle(\bR_k \bx_i, \by_i)\leq \epsilon_i, \; \forall i \in \cI_k, \\
& & & \angle(\bR_k \bx_k, \by_k)\leq\epsilon_k.
\end{aligned}
\tag{$P^{rot}_k$}
\end{align} 
Secs.~\ref{sec:upbndrot} and~\ref{sec:stabbing} are devoted to describing an efficient algorithm for calculating an upper bound $\hat{p}_k$ to \ref{eq:frk}. In Sec.~\ref{sec:main-r}, an overall algorithm which encapsulates the bound calculation to conduct GORE will be presented.

\subsection{Efficient algorithm for upper bound}\label{sec:upbndrot}

\begin{figure*}[ht]
\centering
\subfloat[]{\includegraphics[scale=0.42]{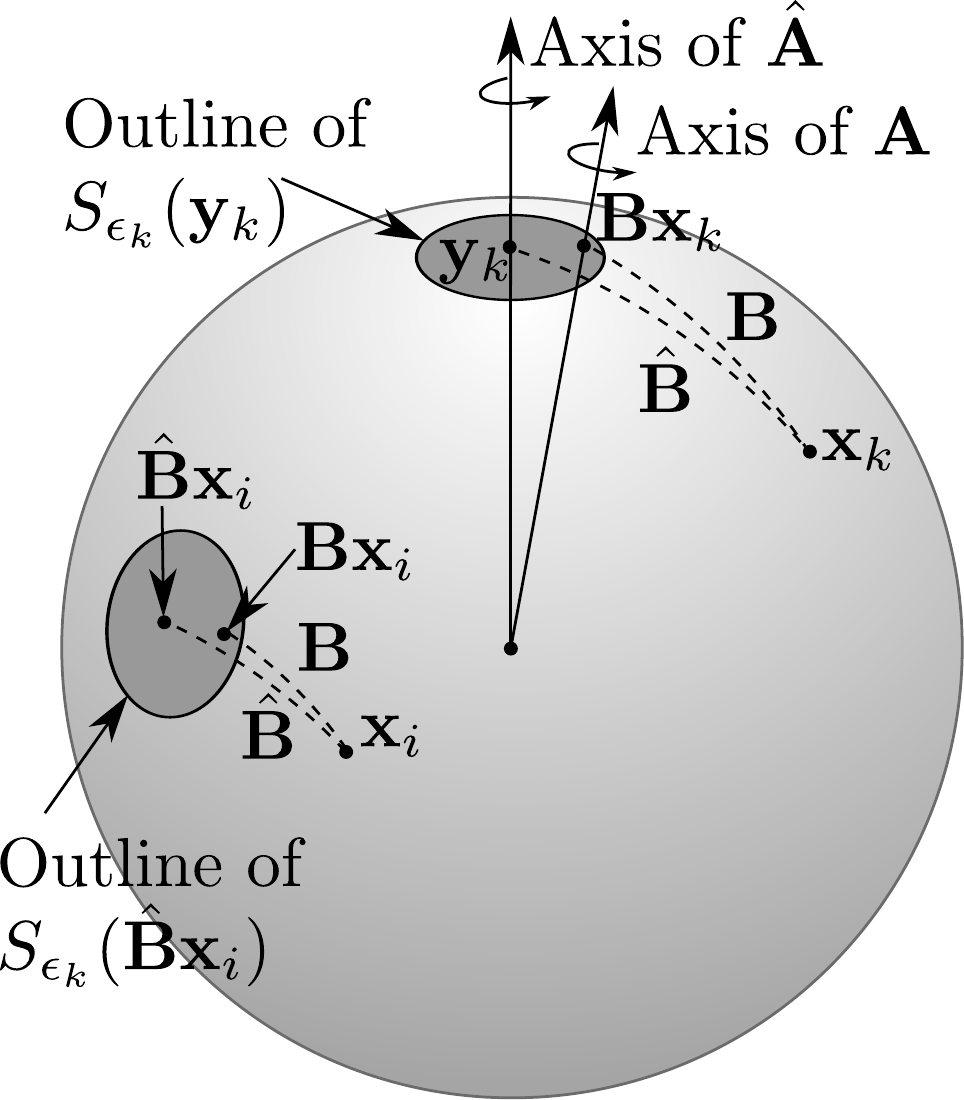}\label{fig:bound_a}}\hspace{.2cm}
\subfloat[]{\includegraphics[scale=0.42]{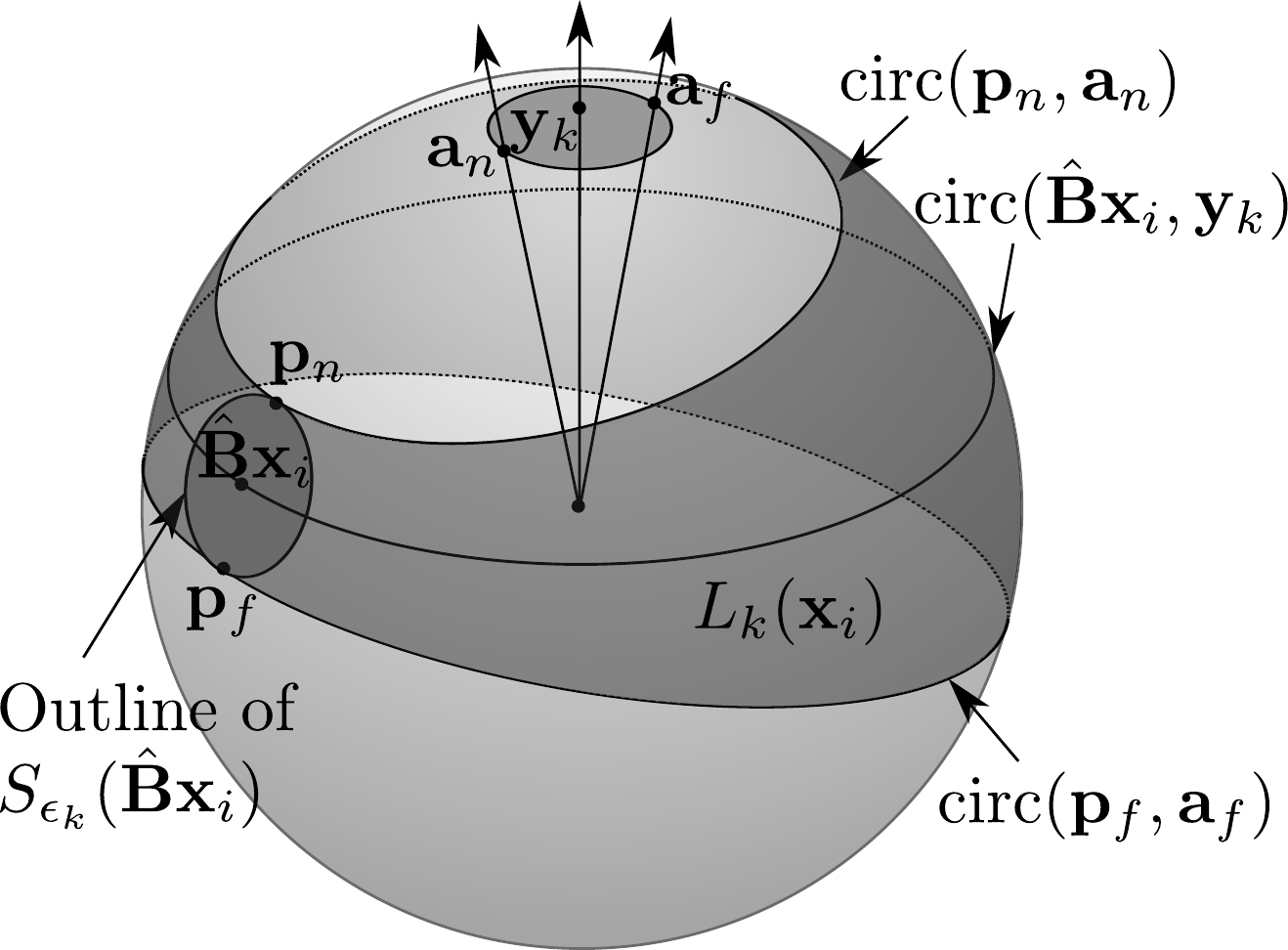}\label{fig:bound_c}}
\subfloat[]{\includegraphics[scale=0.42]{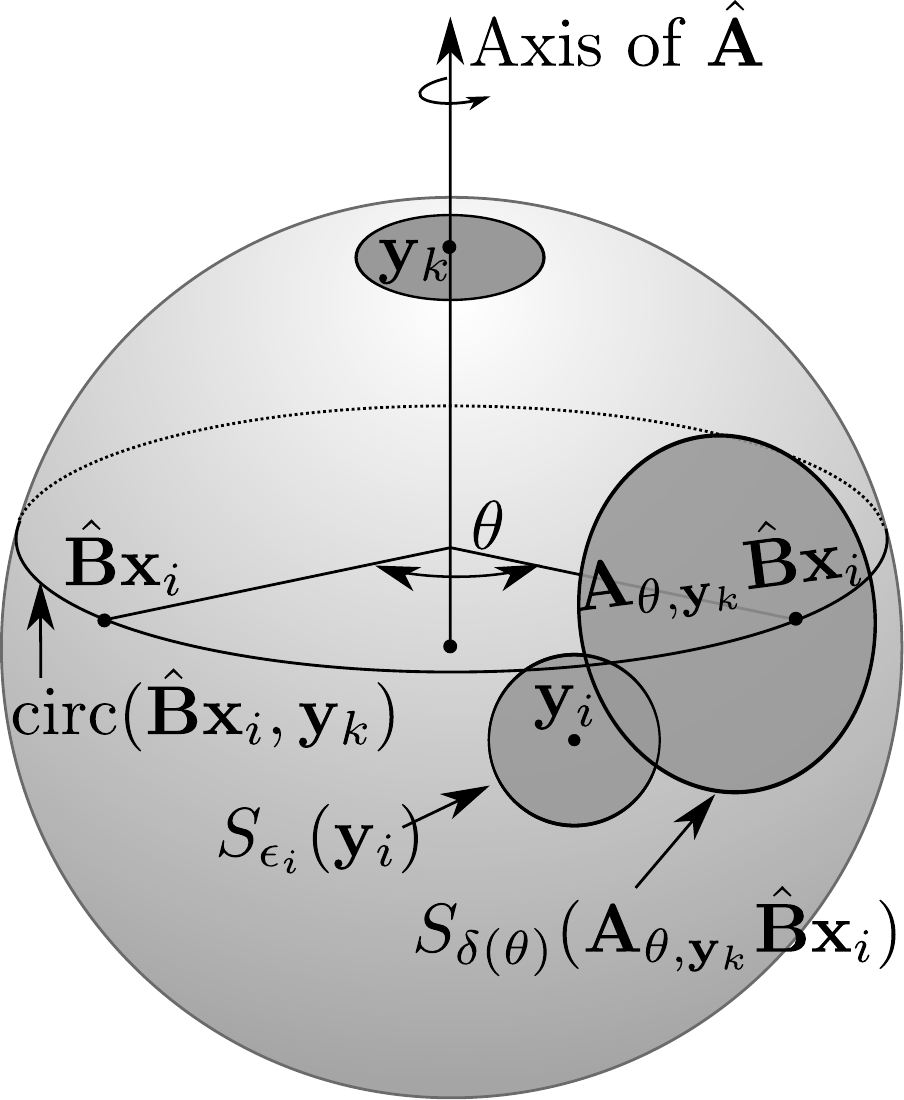}\label{fig:bound_b}}\hfill
\subfloat[]{\includegraphics[scale=0.42]{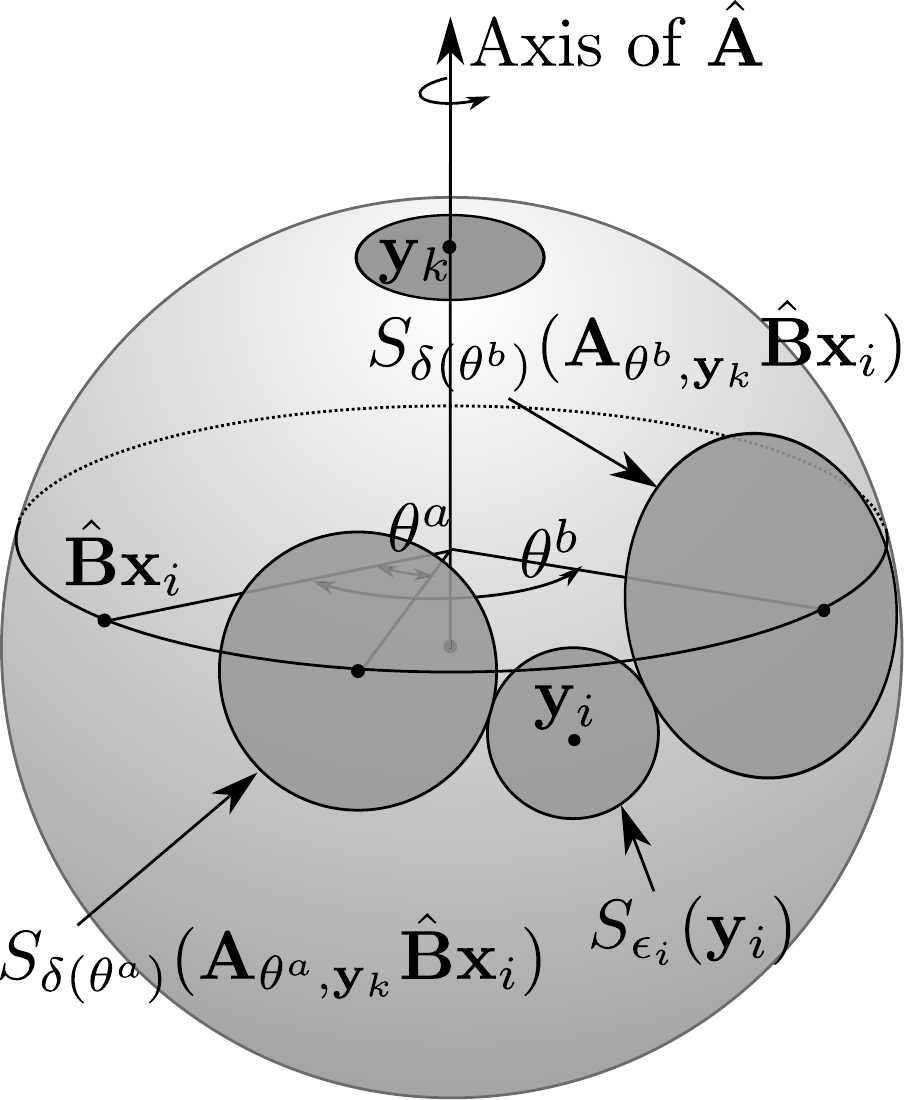}\label{fig:angint}} 
\caption{(a) Interpreting rotation $\bR_k$ according to~\eqref{eq:decompose}. (b) The uncertainty region $L_k(\bx_i)$~\eqref{eq:fea}. (c) This figure shows $S_{\delta(\theta)}(\bA_{\theta,\by_k}\hat{\bB}\bx_i)$ intersecting with $S_{\epsilon_i}(\by_i)$ for a particular $\theta$. (d) We wish to find a bounding interval $\Theta_i = [\theta^a, \theta^b] \subset [-\pi,\pi]$ on $\theta$ for which the intersection is non-empty. 
}
\label{fig:sphere}
\end{figure*}

Recall that any candidate rotation $\bR_k$ to solve \ref{eq:frk} must bring $\bx_k$ within angular distance $\epsilon_k$ from $\by_k$, i.e.,
\begin{align}\label{eq:constraint}
	\angle(\bR_k \bx_k, \by_k)\leq\epsilon_k.
\end{align}
We interpret $\bR_k$ by decomposing it into two rotations
\begin{align}\label{eq:decompose}
	\bR_k = \bA\bB,
\end{align}
where we define $\bB$ as a rotation that honors the condition
\begin{equation}\label{eq:bnd1}
\angle(\bB \bx_k, \by_k) \leq \epsilon_k,
\end{equation}
and $\bA$ as a rotation about axis $\bB \bx_k$. Since $\bA$ leaves $\bB\bx_k$ unchanged, the condition~\eqref{eq:bnd1} and hence constraint~\eqref{eq:constraint} are always satisfied. Fig.~\ref{fig:bound_a} illustrates this interpretation.

Solving \ref{eq:frk} thus amounts to finding the combination of the rotation $\bB$ (a 2 DoF problem, given~\eqref{eq:bnd1}) and the rotation angle of $\bA$ (a 1 DoF problem) that maximize the objective.

\subsubsection{The ideal case}

In the absence of noise and outliers, $\bx_i$ can be aligned exactly with $\by_i$ for all $i$. Based on~\eqref{eq:decompose}, we denote the rotation that solves \ref{eq:frk} under this ideal case as
\begin{align}
	\hat{\bR}_{k} = \hat{\bA} \hat{\bB},
\end{align}
which can be solved as follows (refer also to Fig.~\ref{fig:bound_a}). First, find a rotation $\hat{\bB}$ that aligns $\bx_k$ exactly with $\by_k$, i.e.,
\begin{align}\label{eq:ideal}
	\hat{\bB} \bx_k = \by_k.
\end{align}
For example, take $\hat{\bB}$ as the rotation that maps $\bx_k$ to $\by_k$ with the minimum geodesic motion. To solve for $\hat{\bA}$, take any $i \ne k$, then find the angle $\hat{\theta}$ of rotation about axis $\hat{\bB}\bx_k$ that maps $\hat{\bB}\bx_i$ to $\by_i$. Then,
\begin{align}
\hat{\bA} = \exp(\hat{\theta}[\hat{\bB}\bx_k]_{\times}),
\end{align}
where $\exp(\cdot)$ is the exponential map, and $[\bx]_\times$ returns the cross product matrix of $\bx$. The above steps affirm that rotation estimation requires a minimum of two point matches~\cite{horn87}.


\subsubsection{Uncertainty bound}

In the usual case, we must contend with noise and outliers. The aim of this section is to establish a bound on the position of $\bx_i$ when acted upon by the set of feasible rotations $\bR_k$, i.e., those that satisfy~\eqref{eq:constraint} for~\ref{eq:frk}.

The set of $\bB$ that maintain~\eqref{eq:bnd1} cause $\bB \bx_k$ to lie within a spherical region of angular radius $\epsilon_k$ centered at $\by_k$, i.e.,
\begin{align}
\bB\bx_k \in S_{\epsilon_k}(\by_k)
\end{align}
where
\begin{align}
S_{\epsilon}(\by) := \{ \bx \in \mathbb{R}^3 \mid  \| \bx\| = 1, \angle(\bx,\by) \le \epsilon \}.
\end{align}
Since $\bB \bx_k$ is the rotation axis of $\bA$, the interior of $S_{\epsilon_k}(\by_k)$ is also the set of possible rotation axes for $\bA$. Further, for any $i \neq k$, we can establish a second spherical region 
\begin{align}\label{eq:patch}
\bB\bx_i \in S_{\epsilon_k}(\hat{\bB}\bx_i)
\end{align}
such that the set of feasible $\bB$ cause $\bB \bx_i$ to lie in a spherical region; see Sec.~A in the supp.~material for its derivation.

Fig.~\ref{fig:bound_a} also shows $S_{\epsilon_k}(\by_k)$ and $S_{\epsilon_k}(\hat{\bB}\bx_i)$. The bound on $\bR_k\bx_i$ can thus be analyzed based on these two regions.


To make explicit the dependence of $\bA$ on a rotation axis $\ba$ and angle $\theta$, we now denote it as $\bA_{\theta,\ba}$, where
\begin{align}
	\bA_{\theta,\ba} = \exp(\theta[\ba]_\times).
\end{align}
Let $\bp$ be an arbitrary unit-norm point. Define
\begin{align}
	\text{circ}(\bp,\ba) := \{ \bA_{\theta,\ba}\bp \; | \; \theta \in [-\pi,\pi] \}
\end{align}
as the circle traced by $\bp$ when acted upon by rotation $\bA_{\theta,\ba}$ for all $\theta$ at a particular axis $\ba$.

The set of possible positions of $\bR_k \bx_i$ is then defined by
\begin{align}\label{eq:fea}
	L_k(\bx_i) := \{ \text{circ}(\bp,\ba) \; | \; \bp \in S_{\epsilon_k}(\hat{\bB}\bx_i), \ba \in S_{\epsilon_k}(\by_k) \}.
\end{align}
Fig.~\ref{fig:bound_c} illustrates this feasible region, which exists on the unit sphere. The region is bounded within the two circles
\begin{align}\label{eq:boundary}
	\text{circ}(\bp_n,\ba_n) \; \; \text{and} \; \; \text{circ}(\bp_f,\ba_f),
\end{align}
which are highlighted in Fig.~\ref{fig:bound_c}. Intuitively, $\bp_n$ and $\ba_n$ (resp.~$\bp_f$ and $\ba_f$) are the closest (resp.~farthest) pair of points from $S_{\epsilon_k}(\hat{\bB}\bx_i)$ and $S_{\epsilon_k}(\by_k)$. See Sec.~A in the supp.~material for their mathematical representations.

The derivations above lead to our first result.

\begin{result}\label{res:prune}
For any $i \ne k$, if $S_{\epsilon_i}(\by_i)$ does not intersect with $L_k(\bx_i)$, then $(\bx_i,\by_i)$ cannot be aligned by any rotation $\bR_k$ that satisfies~\eqref{eq:constraint}. The correspondence $(\bx_i,\by_i)$ can then be safely removed without affecting the result of~\ref{eq:frk}.
\end{result}

\subsubsection{Reducing the uncertainty}\label{sec:reduce}

For each point match $(\bx_i,\by_i)$ that survives the pruning by Result~\ref{res:prune}, we reduce its uncertainty bound~\eqref{eq:fea} into an \emph{angular interval}. This reduction is crucial for our efficient upper bound algorithm to be introduced in Sec.~\ref{sec:stabbing}.

Consider rotating a point $\bp$ in $S_{\epsilon_k}(\hat{\bB}\bx_i)$ with $\bA_{\theta,\bu}$ for a fixed angle $\theta$ and an axis $\bu \in S_{\epsilon_k}(\by_k)$. We wish to bound the possible locations of $\bA_{\theta,\bu}\bp$ given the uncertainty in $\bp$ and $\bu$. To this end, the following bound can be established (see Sec.~A in the supp.~material for its derivation)
\begin{align}
\max_{ \substack{  \bp \in S_{\epsilon_k}(\hat{\bB}\bx_i)     \\ \bu \in S_{\epsilon_k}(\by_k)         } } \angle(\bA_{\theta,\bu}\bp, \, \bA_{\theta,\by_k}\hat{\bB}\bx_i) 
\le \delta(\theta),\label{eq:bound2}
\end{align}
where 
\begin{align} \label{eq:delta}
\delta(\theta) := 2|\theta|\sin(\epsilon_k/2) + \epsilon_k.
\end{align}
The inequality~\eqref{eq:bound2} states that for a fixed $\theta$ and for all $\bu \in S_{\epsilon_k}(\by_k)$ and $\bB\bx_i \in S_{\epsilon_k}(\hat{\bB}\bx_i)$, the point $\bA_{\theta,\bu}\bB\bx_i$ lies in
\begin{align}\label{eq:bound3}
	S_{\delta(\theta)}(\bA_{\theta,\by_k}\hat{\bB}\bx_i).
\end{align}
Fig.~\ref{fig:bound_b} depicts this spherical region. Observe that for all $\theta \in [-\pi,\pi]$, the center of the region lies in $\text{circ}(\hat{\bB}\bx_i,\by_k)$. Intuitively, this is a circle of a fixed latitude on the globe when $\by_k$ is the ``North Pole". Further, the spherical region attains the largest angular radius at $\theta = \pm \pi$.

For a pair $(\bx_i,\by_i)$, we wish to obtain a bound $\Theta_i = [\theta^a, \theta^b]$ on the range of $\theta$ that enable $\bA_{\theta,\bu}\bB\bx_i$ to align with $\by_i$, given the uncertainties $\bu \in S_{\epsilon_k}(\by_k)$ and $\bB\bx_i \in S_{\epsilon_k}(\hat{\bB}\bx_i)$. This is analogous to seeking a bound on the $\theta$ that allows $S_{\delta(\theta)}(\bA_{\theta,\by_k}\hat{\bB}\bx_i)$ to ``touch" $S_{\epsilon_i}(\by_i)$; see Fig.~\ref{fig:angint}.

Analytically solving for $\Theta_i$ is non-trivial; Sec.~A (supplementary material) describes a fast approximate algorithm that is able to yield a tight bounding interval $\Theta_i$.

\begin{result}\label{res:interval}
	For any $i \ne k$, if $S_{\epsilon_i}(\by_i)$ intersects with $L_k(\bx_i)$, the range of angles $\theta$ such that $\angle(\bA_{\theta,\bu}\bB\bx_i,\by_i) \le \epsilon_i$ for all $\bu \in S_{\epsilon_k}(\by_k)$ and $\bB\bx_i \in S_{\epsilon_k}(\hat{\bB}\bx_i)$ is bounded by $\Theta_i$, where $\Theta_i$ is as defined in Sec.~\ref{sec:reduce} and computed as in Sec.~A (supp.~material).
\end{result}

\subsection{Interval stabbing}\label{sec:stabbing}

For problem~\ref{eq:frk}, on the input point matches that remain after pruning by the application of Result~\ref{res:prune}, we use Result~\ref{res:interval} to convert them into a set of angular intervals $\{ \Theta_j \}$, where each $\Theta_j = [\theta^a_j,\theta^b_j]$. We aim to find the largest number of point matches that can be aligned by the same rotation angle $\theta$. More formally, we seek the solution
\begin{align}\label{eq:stab}
	O_k = \underset{\theta \in [-\pi,\pi]}{\text{maximize}} \;\; \sum_j \mathbb{I}(\theta \in [\theta^a_j,\theta^b_j])
\end{align}
where $\mathbb{I}(\cdot)$ is an indicator function that returns $1$ if the input predicate is true and $0$ otherwise. This is the well-known \emph{interval stabbing} problem, for which efficient deterministic algorithms exist~\cite[Chap.~10]{berg08}. We take $\hat{p}_k = O_k + 1$ as an upper bound to the value $p_k$ to \ref{eq:frk}.

\begin{theorem}
$O_k + 1$ is an upper bound to the value of \ref{eq:frk}.
\end{theorem}
\begin{proof}
	By Result~\ref{res:interval}, each interval $\Theta_j$ is an over-estimation of the range of angles of rotation $\bA_{\theta,\bu}$ that permit the associated point match to be aligned. The number $O_k + 1$ must thus be greater than or equal to the maximum number of point matches that can be aligned under problem \ref{eq:frk}.
\end{proof}

As a by-product of interval stabbing, we derive
\begin{align}\label{eq:suboptrot}
	\tilde{\bR}_k = \bA_{\tilde{\theta},\hat{\bB}\bx_k}\hat{\bB},
\end{align}
where $\tilde{\theta}$ is an angle that globally solves~\eqref{eq:stab}. We can take $\tilde{\bR}_k$ as an approximate solution to the rotational registration problem~\eqref{eq:rotsearch2}. Aligning the input data with $\tilde{\bR}_k$ thus provides a lower bound $l$ to the problem~\eqref{eq:rotsearch2}.

\subsection{Main algorithm}\label{sec:main-r}

Algorithm~\ref{alg:gore-r} summarizes our GORE method for rotational registration~\eqref{eq:rotsearch2}. Given a set of input point matches $\cH$, our method iterates over each point match $(\bx_k,\by_k)$ and performs two operations: seek an improved lower bound $l$ to problem~\eqref{eq:rotsearch2} and an upper bound {$\hat{p}_k$} to subproblem \ref{eq:frk}; both steps are conducted simultaneously using our techniques in Secs.~\ref{sec:upbndrot} and~\ref{sec:stabbing}. Both values are then compared to try to reject the current point match as a true outlier. The output is a reduced set of point matches $\cH^\prime \subseteq \cH$ guaranteed to include the globally optimal solution $\cI^*$ to~\eqref{eq:rotsearch2}.

\begin{algorithm}[t]
	\begin{algorithmic}[1]
		\REQUIRE Point correspondences $\{(\bx_i,\by_i)\}^{N}_{i=1}$ (points are assumed to have unit norm), inlier thresholds  $\{\epsilon_i\}_{i=1}^N$.
		\STATE $\cH \leftarrow \{1,2,\dots,N\}$.
		\STATE $\cH' \leftarrow \cH$, $\mathcal{O} \leftarrow \cH$, $\mathcal{V} \leftarrow \emptyset$, and $l \leftarrow 0$.
		\FORALL{$k \in \mathcal{O}$}\label{gore-r:step:loop}
		\STATE $\mathcal{V} \leftarrow \mathcal{V} \cup \{k\}$.
		\STATE Compute upper bound $\hat{p}_k$ and suboptimal rotation $\tilde{\bR}_k$ (Secs.~\ref{sec:upbndrot} and~\ref{sec:stabbing}) for problem \ref{eq:frk} on data indexed by $\cH^\prime$.\label{gore-r:step:pk}
		\STATE $\mathcal{C}_k \leftarrow \{ i \in \cH^\prime  \mid \| f(\bx_i \mid \tilde{\bR}_k ) - \by_i\| \le \epsilon_i \}$.
		\STATE $l_k \leftarrow |\mathcal{C}_k|$.
		\IF{$l_k>l$} 
		\STATE $l \leftarrow l_k$, $\tilde{\bR}  \leftarrow \tilde{\bR}_k$.
		\STATE $\mathcal{O} \leftarrow \cH^\prime \setminus \mathcal{C}_k$.
		\ENDIF
		\IF{$\hat{p}_k<l$}
		\STATE $\cH' \leftarrow \cH' \setminus \{k\}$.
		\ENDIF
		\STATE $\mathcal{O} \leftarrow \mathcal{O} \setminus \mathcal{V}$.
		\ENDFOR
		\RETURN $ \cH^\prime$ and $\tilde{\bR}$.
	\end{algorithmic}
	\caption{GORE for rotational registration~\eqref{eq:rotsearch2}.}
	\label{alg:gore-r}
\end{algorithm}

The worst case time complexity can be established as follows: for each $k$, the bounding interval $\Theta_i$ for each $i \ne k$ is obtained in constant time. Given $N$ intervals, the stabbing problem~\eqref{eq:stab} can be solved in $\mathcal{O}(N\log N)$ time~\cite[Chap.~10]{berg08}. Thus, Line~\ref{gore-r:step:pk} in Algorithm~\ref{alg:gore-r} takes $\mathcal{O}(N\log N)$ time. In the worst case, Line~\ref{gore-r:step:pk} is performed $N$ times, and Algorithm~\ref{alg:gore-r} thus consumes $\mathcal{O}(N^2\log N)$ time. Overall, Algorithm~\ref{alg:gore-r} contains only very simple geometric operations. In the next subsection, we demonstrate the extreme efficiency of the algorithm in processing large input data sizes for robust rotational registration.

\subsection{Results for rotational registration}\label{sec:rotresults}

We first present experimental results for Algorithm~\ref{alg:gore-r}. Only synthetic data was used for the experiment here, since
\begin{itemize}
\item GORE for rotational alignment will primarily be used as a sub-routine in the more general 6 DoF GORE algorithm in Sec.~\ref{sec:gore-main}. A synthetic data experiment enables a more comprehensive analysis of the performance of Algorithm~\ref{alg:gore-r} as a ``black box".
\item In real-life applications such as surveying and robotics, the point clouds usually differ by more than a 3 DoF rotation. It is thus more cogent to test on real data for the 6 DoF GORE algorithm. Comprehensive evaluation on real data will be given in Sec.~\ref{sec:results}.
\end{itemize}
All techniques/algorithms used in this section were implemented in C++. Experiments were conducted on a standard PC with a 2.50GHz CPU. Our implementations are also provided in the supplementary material.

\subsubsection{Data generation and experimental setting}\label{sec:synth}

A data instance for~\eqref{eq:rotsearch} was generated as follows: $N$ points in a solid ball of radius $100$ were randomly produced to obtain set $\cX$. Set $\cX$ was randomly rotated to produce set $\cY$, which was then added with uniformly distributed noise on a sphere centred at the origin with radius $\xi = 0.5$. We ensured that all correspondences $(\bx_i,\by_i)$ satisfy
\begin{align}
\left| \| \bx_i \| - \| \by_i \| \right| \le \xi
\end{align}
such that no data can be pruned by simply comparing norms; see~\eqref{eq:normprune}. For a given outlier rate $\eta$, $\eta N$ point correspondences $(\bx_i,\by_i)$ were randomly chosen from $(\cX, \cY)$ and the $\bx_i$ was randomly rotated to produce outliers. Each data instance was then converted to the equivalent form~\eqref{eq:rotsearch2}.

In our experiments, $N=\{100,250,500\}$ and $\eta=\{0,0.025,\ldots, 0.95\}$ were used. For each $(N,\eta)$ combination, $1000$ data instances were generated. The following approaches/pipelines were run on each instance:
\begin{itemize}
	\item \textbf{RANSAC}: A confidence level of $\rho = 0.99$ was used for the stopping criterion~\cite{fischler81}. For each data instance, median runtime over $100$ runs were taken.
	\item \textbf{GORE}: Algorithm~\ref{alg:gore-r}. No particular ordering for the data was conducted beyond the order of generation.
	\item \textbf{BnB}: A simplification of the method of~\cite{parrabustos16_pami} by inputting the point correspondences instead of the raw point clouds $\cX$ and $\cY$. Discussing the BnB algorithm is beyond the scope of this paper. See~\cite{parrabustos16_pami} for details.
	\item \textbf{GORE+BnB}: Data remaining after GORE $\cH^\prime$ was fed to BnB. The lower bound of BnB was also initialized as the value of $l$ at the termination of Algorithm~\ref{alg:gore-r}.
	\item \textbf{GORE+RANSAC}: Data remaining after GORE $\cH^\prime$ was fed to RANSAC.
	\item \textbf{RGORE+BnB}: Same as {GORE+BnB}, but the initial value of $l$ in Algorithm~\ref{alg:gore-r} for GORE was obtained by first running RANSAC to yield a suboptimal result.
	\item \textbf{GORE+aBnB}: Same as {GORE+BnB}, but all the original data was given to BnB (GORE was only used to initialize the lower bound of BnB).
\end{itemize}

\subsubsection{Runtime comparisons}

\begin{figure*}
\scriptsize
\begin{tabularx}{\textwidth}{C{1} C{1} C{1}}
$N=100$ & $N=250$& $N=500$\\
\hline
\includegraphics[width=.32\textwidth]{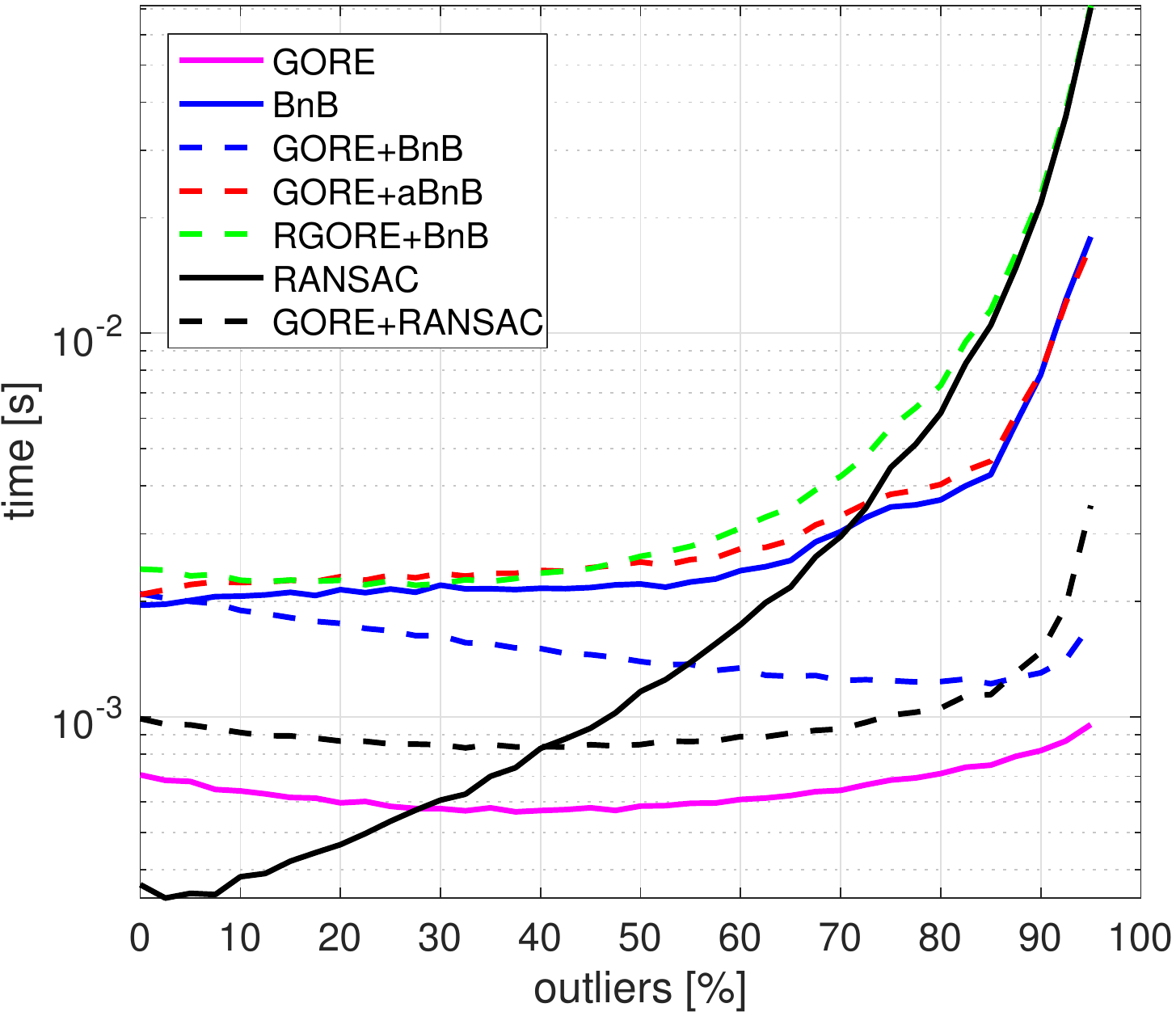}&
\includegraphics[width=.32\textwidth]{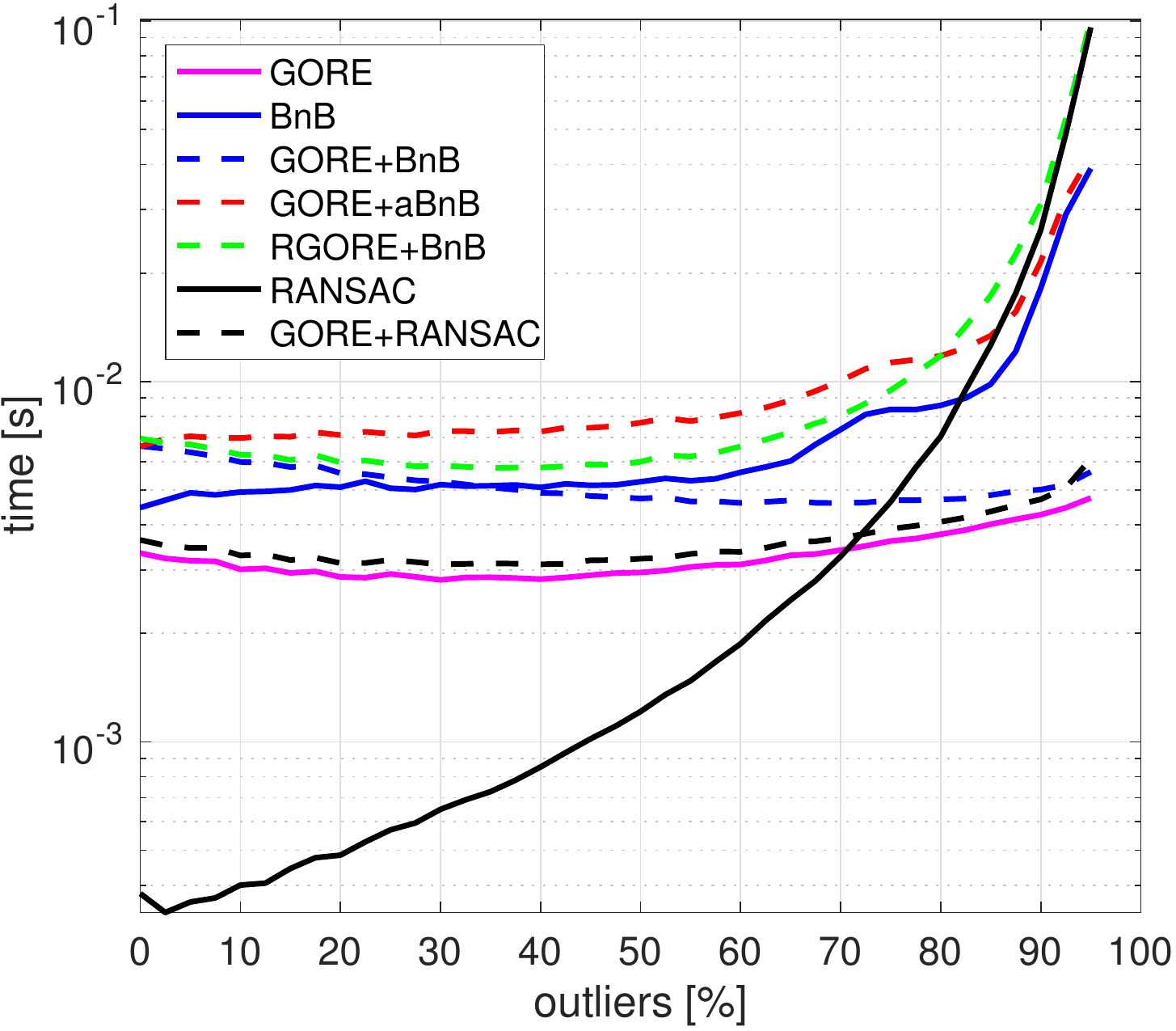}&
\includegraphics[width=.33\textwidth]{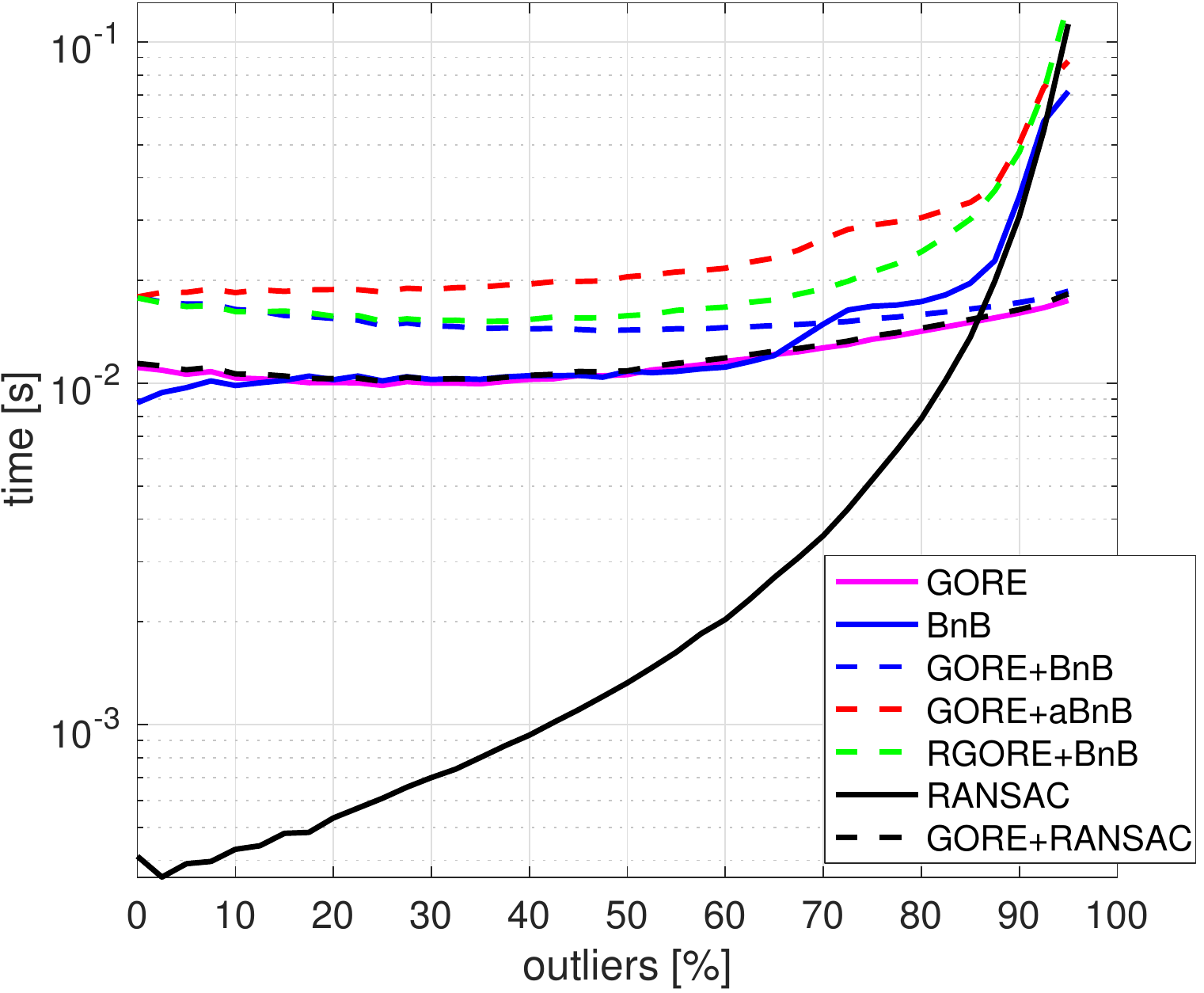}\\
\includegraphics[width=.33\textwidth]{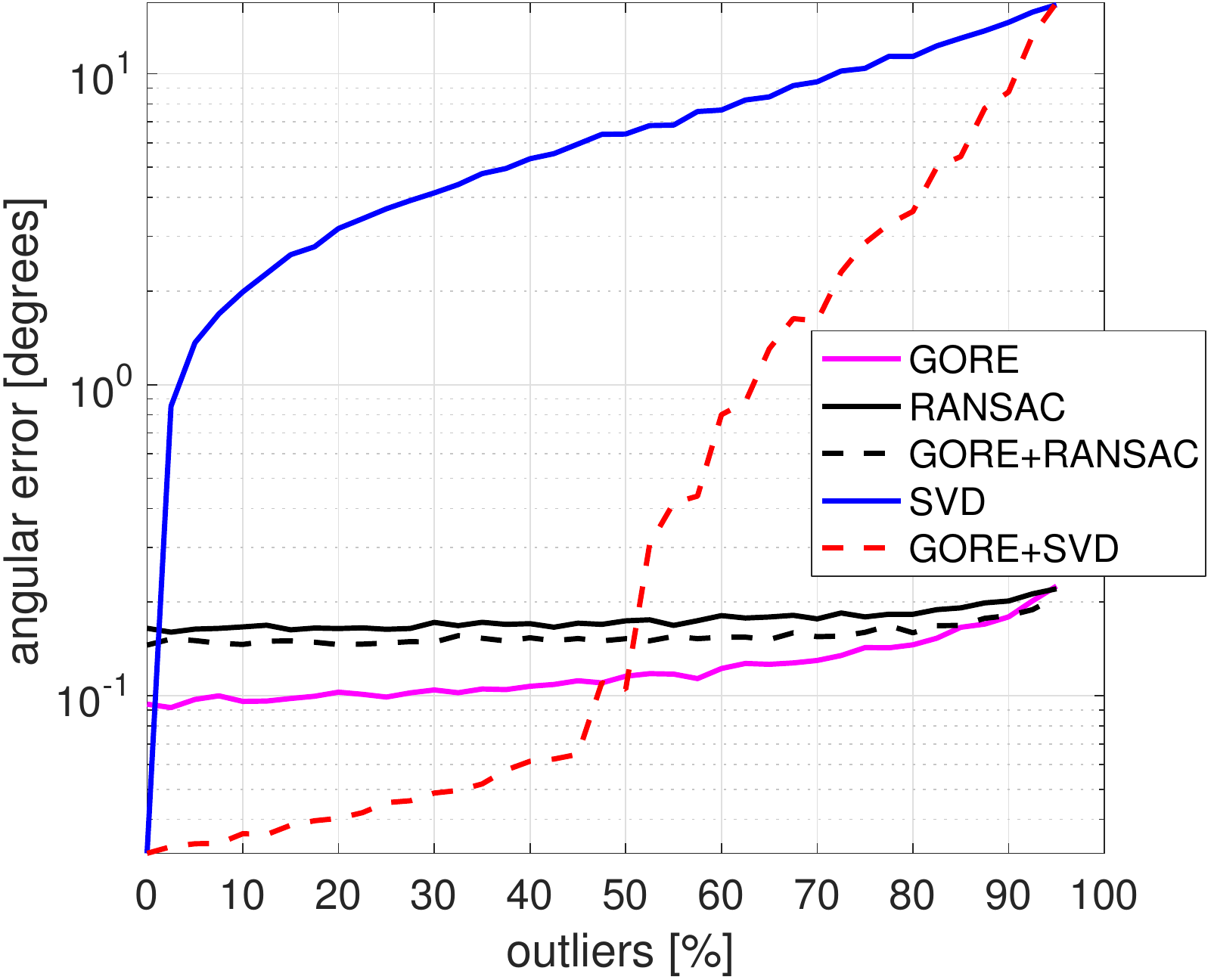}&
\includegraphics[width=.29\textwidth]{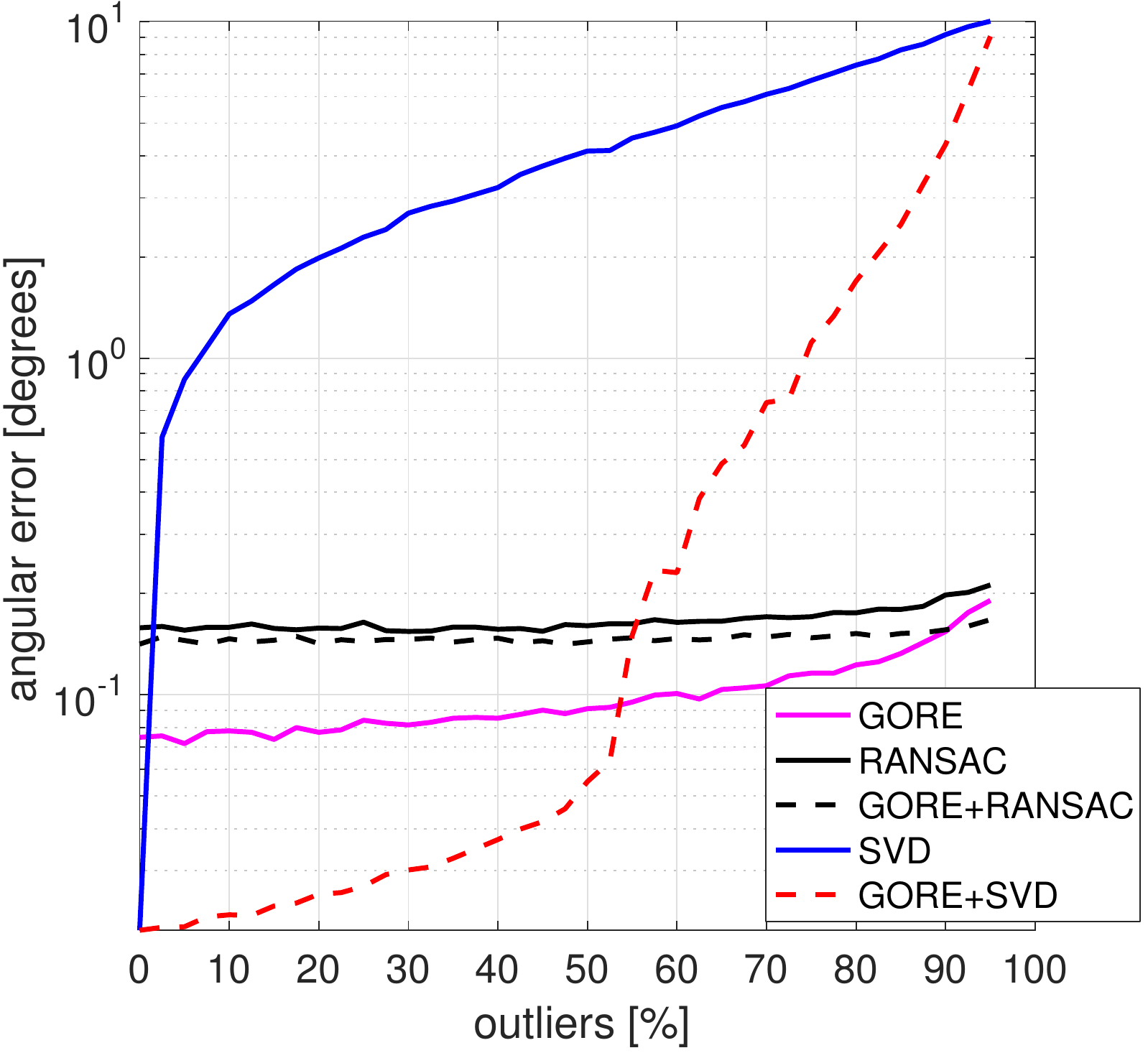}&
\includegraphics[width=.33\textwidth]{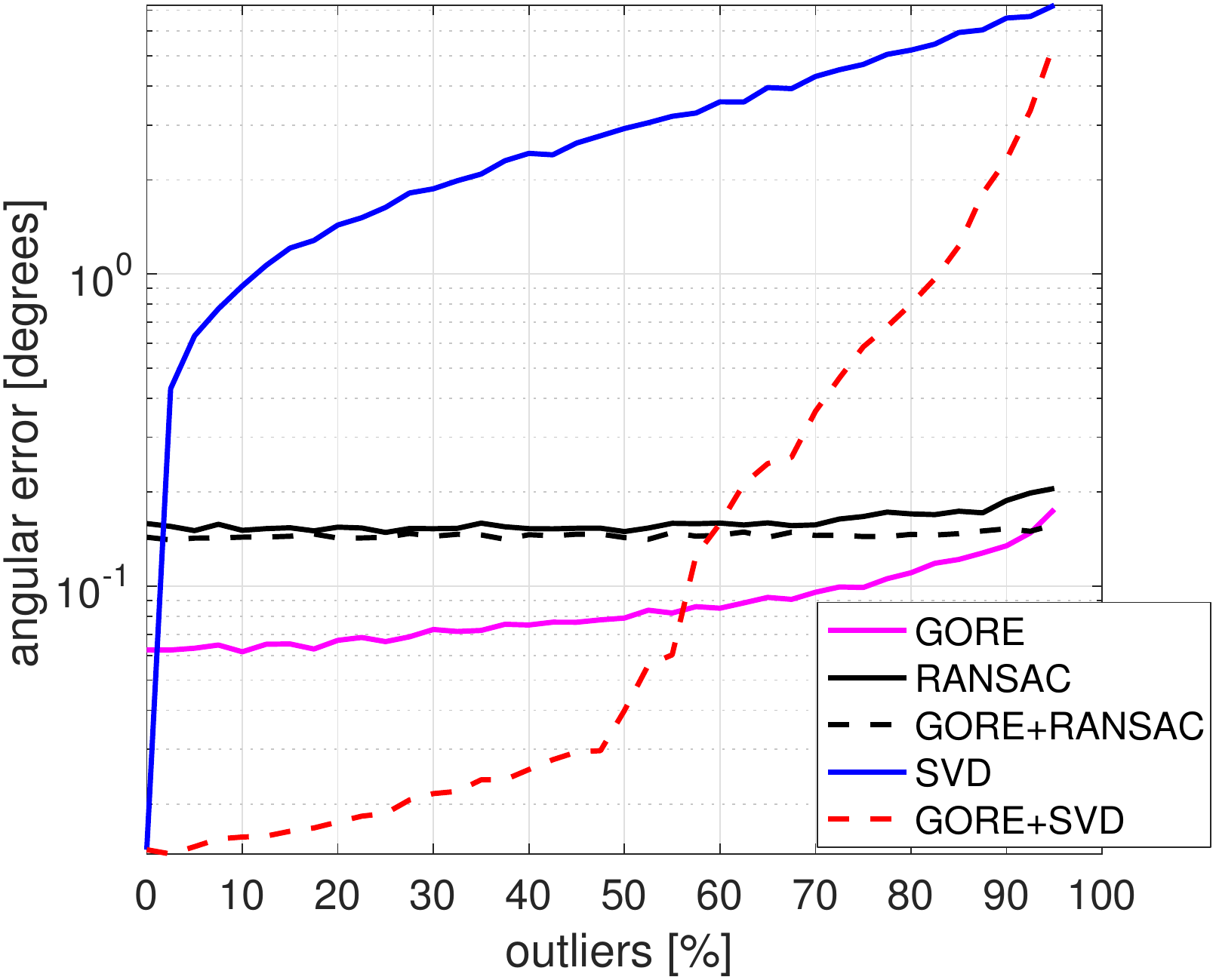}\\
\end{tabularx}
\caption{Results of rotation search on synthetic data. Top row: Runtimes of different optimization pipelines under different data size $N$ and outlier rates. Bottom row: Angular errors of estimated rotations w.r.t.~the globally optimal rotation.}
\label{fig:synthetic}
\end{figure*}

The first row of Fig.~\ref{fig:synthetic} shows the median \emph{total} runtime over all data instances for the methods. While RANSAC was faster than GORE at low outlier rates, as the outlier rate increased, the runtime of RANSAC increased exponentially. In contrast, the runtime of GORE grew at a much lower rate. The runtime of BnB also grew rapidly at higher outlier rates. Observe that at very high outlier rates ($\eta>80\%$), RANSAC consumed at least as much time as BnB.

For all $N$, the comparative trends between BnB versus GORE+BnB, and RANSAC versus GORE+RANSAC, were similar. At low outlier rates, preprocessing with GORE did not ``payoff" since there were few outliers to reject. However, as the outlier rate increased, the benefits of preprocessing with GORE steadily became obvious, i.e., we observe crossing of the curves of pipelines that preprocess with GORE and pipelines that do not. For all $N$, the performance gain is significant for $\ge 90\%$ outliers. As we will show in Sec.~\ref{sec:results}, outlier rates greater than $95\%$ are very common in real point clouds. We remind readers that in this experiment, the data were generated in such a way that the correspondences cannot be rejected by comparing norms, thus the outlier percentages displayed are ``actual".

The results of RGORE+BnB at low outlier rates show that initializing GORE with RANSAC only marginally reduced the total runtime. However, at high outlier rates, the total runtime increased dramatically following the exponentially slowing down of RANSAC.

The trend of BnB shows clearly that the dominating factor in speeding up BnB is in reducing the data amount, not in initializing BnB with a good lower bound. Hence, warm starting BnB with the suboptimal result $|\tilde{\cI}|$ of RANSAC will not reduce runtime (not to mention that at high outlier rates, the computation of RANSAC itself is a major burden).

\subsubsection{Evaluation of suboptimal rotation}

Here we provide empirical evidence that, although GORE cannot completely eliminate all outliers, the best suboptimal rotation $\tilde{\bR}_k$ calculated by Algorithm~\ref{alg:gore-r} is actually a good approximate solution. On each data instance generated above, we calculated the error of the best $\tilde{\bR}_k$ to the globally optimal solution $\bR^*$, where the error is measured by
\begin{align}
d\angle(\tilde{\bR}_k, \bR^*)=\|\log(\tilde{\bR}_k (\bR^*)^T)\|_2
\end{align}
with $\log(\cdot)$ the inverse of the exponential map. The distance is interpreted as the minimum geodesic motion between $\tilde{\bR}_k\bp$ and $\bR^*\bp$ where $\bp$ is an arbitrary point~\cite{hartley13}.

The second row in Fig.~\ref{fig:synthetic} shows the error of the $\tilde{\bR}_k$ returned by GORE. In the same diagrams, we show the error of the rotations obtained by RANSAC and GORE+RANSAC. The error of $\tilde{\bR}_k$ remained very low ($\le 0.05\degree$) even for high outlier rates, indicating that GORE is more accurate than RANSAC. This indicates the efficacy of GORE as an approximate rotational registration method. The results also show that a further reduction of RANSAC's error can be achieved by preprocessing with GORE.

As baselines, we obtained the error of the rotations estimated using SVD (least squares)~\cite{arun87} directly on the input data, and on the reduced data $\cH^\prime$ after GORE. As expected, SVD rotation estimation is easily biased by outliers. Also, the non-negligible error of GORE+SVD points to the presence of remaining outliers after GORE.

\section{GORE for Euclidean registration}\label{sec:gore-main}

In the context of 6 DoF rigid registration, the maximum consensus problem~\eqref{eq:pcreg} becomes
\begin{align}\label{eq:rigid}
\begin{aligned}
&\underset{\bT \in SE(3), \; \cI \subseteq \cH}{\text{maximize}}
& & \left| \cI \right| \\ 
&\text{subject to}
& & \| \bR\bx_i + \bt - \by_i\| \leq \xi, \; \forall i \in \cI,
\end{aligned}
\end{align}
where $\bT = (\bR,\bt)$ defines a rigid transformation. Specializing subproblem~\ref{eq:fk} to rigid registration yields
\begin{align}\label{eq:fkrig}
\begin{aligned}
&\underset{\bT_k \in SE(3),\; \cI_k \subseteq \cH \setminus \{ k \}}{\text{maximize}}
& & \left| \cI_k \right| +1 \\ 
& \text{subject to}
& & \| \bR_k\bx_i + \bt_k - \by_i\| \leq \xi, \; \forall i \in \cI_k, \\
& & & \|\bR_k\bx_k + \bt_k - \by_k\|\leq \xi,
\end{aligned}
\tag{$P^{rig}_k$}
\end{align}
where $\bT_k = (\bR_k,\bt_k)$. For the purpose of GORE, we need to compute an upper bound $\hat{p}_k$ to the value $p_k$ of \ref{eq:fkrig}.

\subsection{Efficient calculation of upper bound}\label{sec:6dofbound}

Define
\begin{align}\label{eq:shift}
\bx^{(k)}_i := \bx_i - \bx_k, \;\; \by^{(k)}_i := \by_i - \by_k,
\end{align}
and consider the following rotational registration problem
\begin{align}\label{eq:upbndfk}
\begin{aligned}
&\underset{\bR_k \in SO(3),\;  \cI_k \subseteq \cH \setminus \{ k \} }{\text{maximize}}
& & \left| \cI_k \right| +1 \\ 
& \text{subject to}
& & \|\bR_k \bx^{(k)}_i - \by^{(k)}_i \|\leq 2\xi, \; \forall i \in \cI_k.
\end{aligned}
\tag{$Q_k^{rig}$}
\end{align}
Intuitively, we obtain~\ref{eq:upbndfk} from~\ref{eq:fkrig} by translating the point clouds $\cX$ and $\cY$ such that the reference frame is centered at $(\bx_k,\by_k)$, then doubling the threshold to $2\xi$.

Define $q_k$ as the maximum objective value of problem~\ref{eq:upbndfk}. We first establish the following proposition, which permits the usage of $q_k$ as the required upper bound $\hat{p}_k$ to the value of subproblem~\ref{eq:fkrig}.

\begin{theorem}\label{theo2}
$q_k \ge p_k$, where $p_k$ is the value of~\ref{eq:fkrig}.
\end{theorem}
See Sec.~A in the supp.~material for a proof of Theorem~\ref{theo2}. We apply Theorem~\ref{theo2} in a ``two-step" procedure:
\begin{enumerate}
\item First, we convert \ref{eq:upbndfk} to a rotational registration problem of the form~\eqref{eq:rotsearch2} which uses angular errors. Then, we execute Algorithm~\ref{alg:gore-r} to conduct GORE. Let $\cH^\prime_k$ denote the remaining correspondences. We have that $|\cH^\prime_k| + 1 \ge q_k$. Therefore, if $|\cH_k^\prime| + 1 < l$, where $l$ is a lower bound to problem~\eqref{eq:rigid}, then
\begin{align}
p_k \le q_k < l
\end{align}
and we can immediately reject $(\bx_k,\by_k)$ as a true outlier to problem~\eqref{eq:rigid}.
\item If the above rejection is not successful, then on the remaining data $\cH^\prime_k$ (which can be much smaller than the original input $\cH \setminus \{ k\}$), we perform BnB optimization to exactly solve \ref{eq:upbndfk}. Then we compare $q_k$ and $l$ again to attempt the rejection. The convenience of this step is investigated in Sec.~\ref{sec:usageofbnb}.
\end{enumerate}
To conduct the BnB, we again apply the method of~\cite{parrabustos16_pami} with a simple modification to accept point correspondences instead of ``raw" point clouds without correspondences.

Let $\tilde{\bR}_k$ be the solution of~\ref{eq:upbndfk} (either a suboptimal solution output by Algorithm~\ref{alg:gore-r} for \ref{eq:upbndfk}, or a globally optimal solution). A candidate solution to~\eqref{eq:rigid} can be realised as
\begin{align}\label{eq:cand6dof}
\tilde{\bT}_k = (\tilde{\bR}_k, \; \by_k - \tilde{\bR}_k \bx_k),
\end{align}
which can then be used to attempt to improve $l$.

\subsection{Main algorithm}\label{sec:main}

Algorithm~\ref{alg:gore} summarizes GORE for 6 DoF Euclidean registration. The overall operation is largely similar to Algorithm~\ref{alg:gore-r} for the rotational registration case, i.e., given a set of input point correspondences $\cH$, iterate over each correspondence and attempt to reject it by comparing the upper bound $\hat{p}_k$ with the lower bound $l$. Both values are computed/updated with the technique described in Sec.~\ref{sec:6dofbound}. The main differences with Algorithm~\ref{alg:gore-r} are due to the two-step procedure to compute $\hat{p}_k$ (see Sec.~\ref{sec:6dofbound}).

\begin{algorithm}[t]\centering
	\begin{algorithmic}[1]
		\REQUIRE Point correspondences $\{(\bx_i,\by_i)\}^{N}_{i=1}$, threshold $\xi$.
		\STATE $\cH \leftarrow \{1,2,\dots,N\}$.
		\STATE $\cH' \leftarrow \cH$, $\mathcal{O} \leftarrow \cH$, $\mathcal{V} \leftarrow \emptyset$, and $l \leftarrow 0$.
		\FORALL{$k \in \mathcal{O}$}\label{step:loop}
		\STATE $\mathcal{V} \leftarrow \mathcal{V} \cup \{k\}$.
		\STATE $(\cH^\prime_k,\tilde{\bR}_k) \leftarrow$ Output of Alg.~\ref{alg:gore-r} to solve \ref{eq:upbndfk} on $\cH^\prime$.\label{step:prune}
		\IF{$|\cH_k^\prime|+1 < l$}
		    \STATE  $\hat{p}_k \leftarrow |\cH_k^\prime|+1$.
		\ELSE
		    \STATE $(\hat{p}_k, \tilde{\bR}_k) \leftarrow$ Maximized value and maximizer of~\ref{eq:upbndfk} on $\cH_k^\prime$ using BnB. \label{step:bnb}
		\ENDIF
		\STATE $\tilde{\bT}_k \leftarrow (\tilde{\bR}_k,~\by_k - \tilde{\bR}_k \bx_k)$.
		\STATE $\mathcal{C}_k \leftarrow \{ i \in \cH^\prime \mid \| f(\bx_i \mid \tilde{\bT}_k ) - \by_i \| \le \xi \}$.		
		\STATE $l_k \leftarrow |\mathcal{C}_k|$.
		\IF{$l_k>l$} 
		\STATE $l \leftarrow l_k$, $\tilde{\bT} \leftarrow \tilde{\bT}_k$.
		\STATE $\mathcal{O} \leftarrow \cH^\prime \setminus \mathcal{C}_k$.
		\ENDIF
		\IF{$\hat{f}_k<l$}
		\STATE $\cH' \leftarrow \cH' \setminus \{k\}$.
		\ENDIF
		\STATE $\mathcal{O} \leftarrow \mathcal{O} \setminus \mathcal{V}$.
		\ENDFOR
		\RETURN $\cH^\prime$ and $\tilde{\bT}$.
	\end{algorithmic}
	\caption{GORE for 6 DoF Euclidean registration~\eqref{eq:rigid}.}
	\label{alg:gore}
\end{algorithm}

Algorithm~\ref{alg:gore}, however, has exponential worst case time complexity, due to the possible usage of BnB to solve \ref{eq:upbndfk} (Line~\ref{step:bnb}). However, in practice, BnB is efficient since it is run on a subset $\cH^\prime_k$ of the original input $\cH_k$, where usually $|\cH^\prime_k| \ll |\cH^\prime|$ due to the usage of Algorithm~\ref{alg:gore-r} to first prune the correspondences (Line~\ref{step:prune}). In Sec.~\ref{sec:results}, we will demonstrate the high efficiency of Algorithm~\ref{alg:gore} in processing large inputs.

\section{Results for rigid registration}\label{sec:results}

\subsection{Point cloud registration} \label{sec:resultpcr}
\subsubsection{Experimental setup} \label{sec:expsetuppcr}

We tested our GORE algorithms on real data. This experiment was designed as follows. We used scans of objects from four different sources, namely, our own dataset of laser scans of underground mines (specifically \emph{mine-a}, \emph{mine-b} and \emph{mine-c}), the Stanford 3D Scanning Repository~\cite{curless96}, (specifically \emph{bunny}, \emph{armadillo}, \emph{dragon} and \emph{buddha}), Mian's dataset\footnote{\url{http://staffhome.ecm.uwa.edu.au/~00053650/3Dmodeling.html}} (specifically \emph{t-rex}, \emph{parasauro}, \emph{chef} and \emph{chicken}), and remote sensing data from the ISPRS\footnote{\url{http://www2.isprs.org/commissions/comm3/wg4/tests.html}} (specifically \emph{vaihingen-a} and \emph{vaihingen-b}). Scans of two objects per dataset are shown in Fig.~\ref{fig:6dofresults} (scans of remaining objects can be found in the Sec.~B in the supp.~material).


Two partially overlapping scans $\cX$ and $\cY$ were selected for each object (for the ISPRS set, we manually divided the full scan of a scene into two overlapping scans). The average sizes of $\cX$ and $\cY$ per dataset are listed on the first row in Table~\ref{tab:comparison6dof}. $\cX$ and $\cY$ were centred and scaled such that their centroids coincided with the origin and both point sets were contained in the cube $[-50,50]^3$. Point matches between $\cX$ and $\cY$ were obtained using the state-of-the-art ISS3D~\cite{zhong09} detector and the PFH~\cite{rusu08} descriptor, as implemented on Point Cloud Library\footnote{\url{http://pointclouds.org/}}.

The $N$ best keypoint correspondences (sorted based on the $\ell_2$-norm of the PFH descriptors), where $N \in \{500,1000,2000\}$, were retained to create instances of the robust point cloud registration problem~\eqref{eq:rigid}. The ratio of true outliers $\eta$ are listed in Table~\ref{tab:comparison6dof} per each dataset, based on the inlier threshold $\xi$ taken as the average nearest neighbor distance in $\cX$ and $\cY$. Observe that the outlier rates in this problem are extremely high, even reaching $99\%$ in some ($N$, dataset) configurations. For each correspondence set, $10$ different randomized transformations were applied on $\cX$ to produce $10$ data instances for Euclidean registration.

\subsubsection{Qualitative evaluation}\label{sec:pcrqualitative}

Fig.~\ref{fig:6dofresults} shows qualitative results from applying GORE (Algorithm~\ref{alg:gore}) on the data instances with $N = 500$. All of these instances contain  $>90\%$ outliers. GORE terminated within $5$ seconds, except for \emph{vaihingen-b}, for which GORE terminated within $11$ seconds (see Table~1 in Sec.~B of the supp.~material for results on individual instances). Importantly, GORE reduced the input data $\cH$ to be much smaller set $\cH^\prime$ ($< 15\%$ of original size). The approximate solutions $\tilde{\bT}$ output by GORE also gave satisfactory registrations.

GORE even produced satisfactory alignments for the mining dataset. In this dataset, registration instances contain \emph{structured outliers} resulting from the “self-similar” structures of channels.  Moreover, point density vary w.r.t. the distance to the LIDAR scan. To efficiently capturing a mining site, scanning devices are distant located each other. Consequently, overlapping between scans is low.

\begin{figure*}
	\begin{tabularx}{\textwidth}{c|c|C{.9}|C{.8}|C{.7}}
\footnotesize
		& & {\footnotesize Input correspondences $\cH$} & {\footnotesize Remaining data $\cH^\prime$ after GORE} & {\footnotesize Registration using approximate solution $\tilde{\bT}$ from GORE}\\ 
		\hline
		&&&&\\[-.35cm]
		\multirow{2}{*}{\rotatebox{90}{ Stanford}} 
		&\rotatebox[origin=l]{90}{\emph{bunny}}&
		\includegraphics[height=1.7cm]{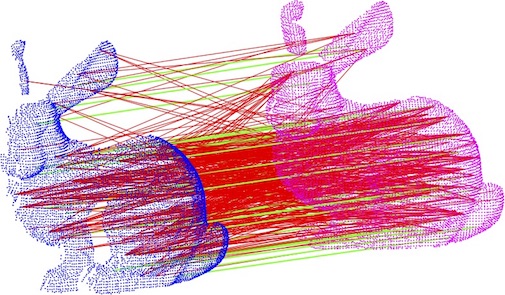} &
		\includegraphics[height=1.7cm]{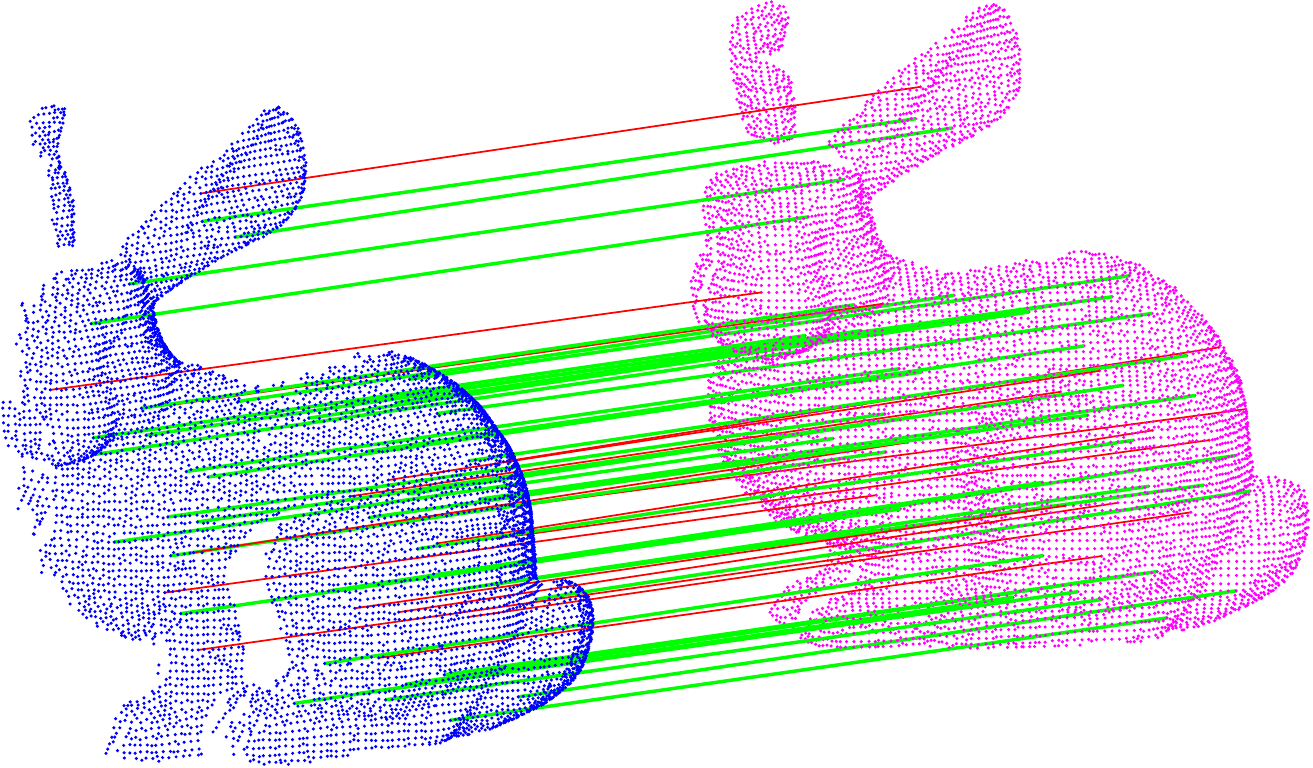} &
		\includegraphics[height=1.7cm]{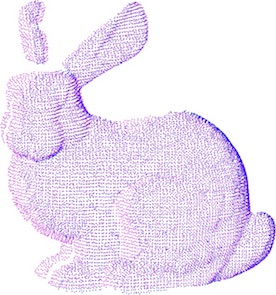} 
		\\[-.1cm]
		&\rotatebox[origin=l]{90}{\emph{dragon}}&
		\includegraphics[height=1.6cm]{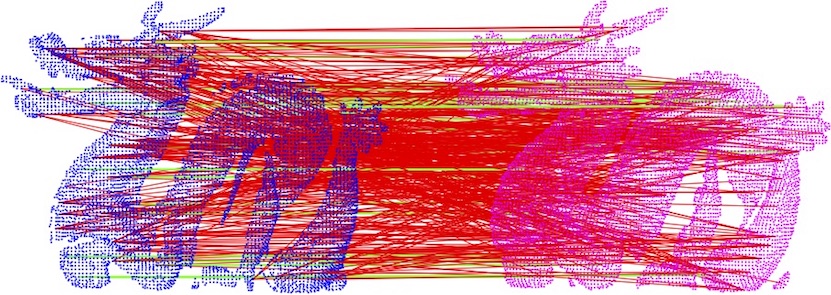}&
		\includegraphics[height=1.6cm]{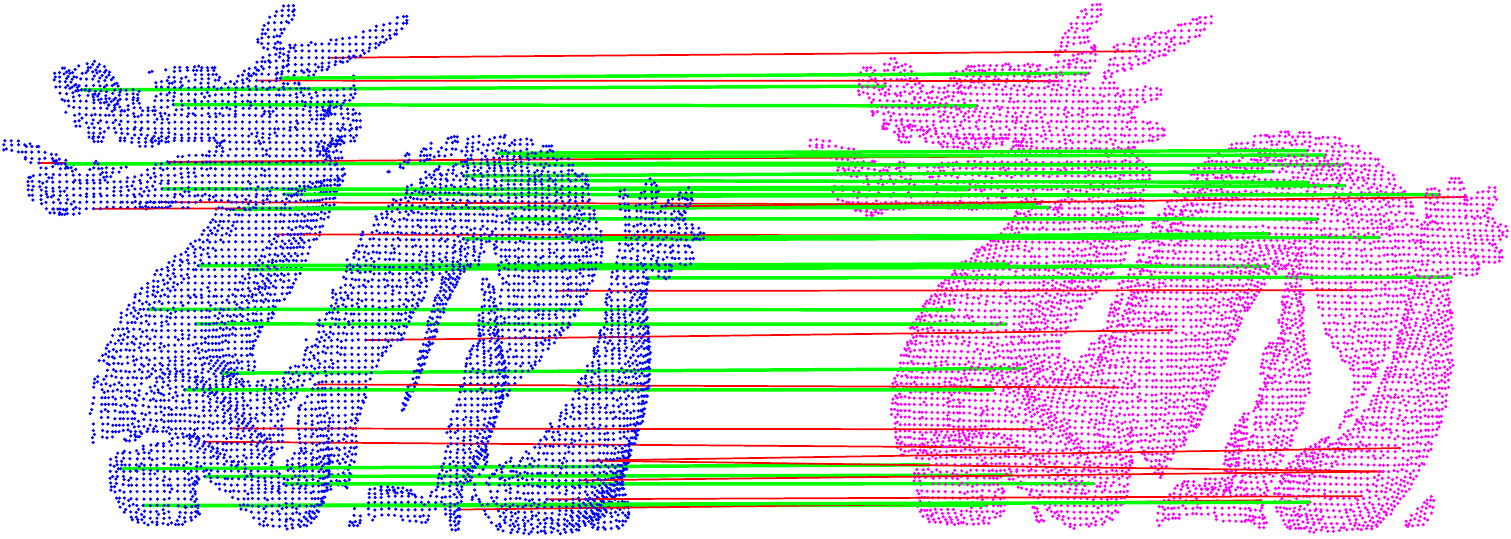}&
		\includegraphics[height=1.6cm]{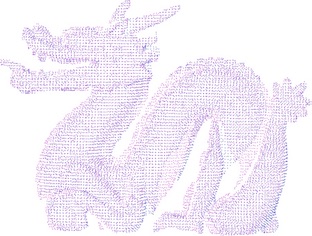}
		\\
		\hline
		&&&&\\[-.35cm]
		\multirow{2}{*}{\rotatebox{90}{ Mian}} 
		&\rotatebox[origin=l]{90}{ \emph{t-rex}}&
		\includegraphics[height=1.5cm]{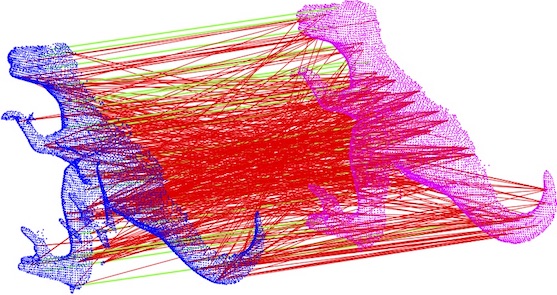}&
		\includegraphics[height=1.5cm]{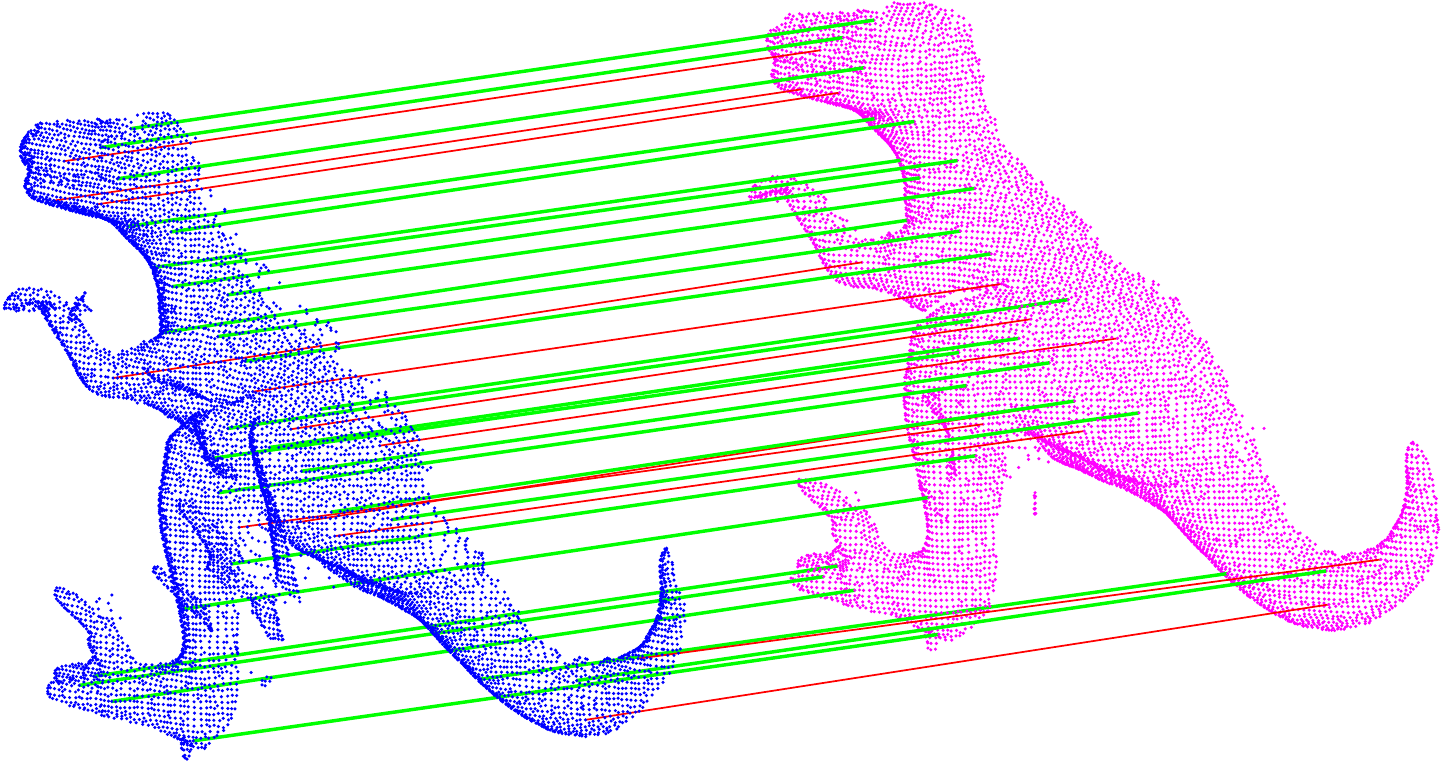}&
		\includegraphics[height=1.5cm]{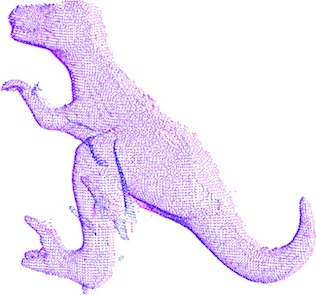}
		\\
		&\rotatebox[origin=l]{90}{\emph{parasauro}}&
		\includegraphics[height=1.7cm]{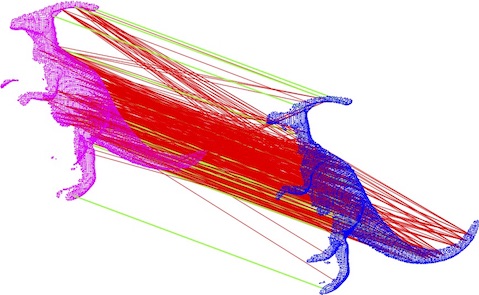}&
		\includegraphics[height=1.7cm]{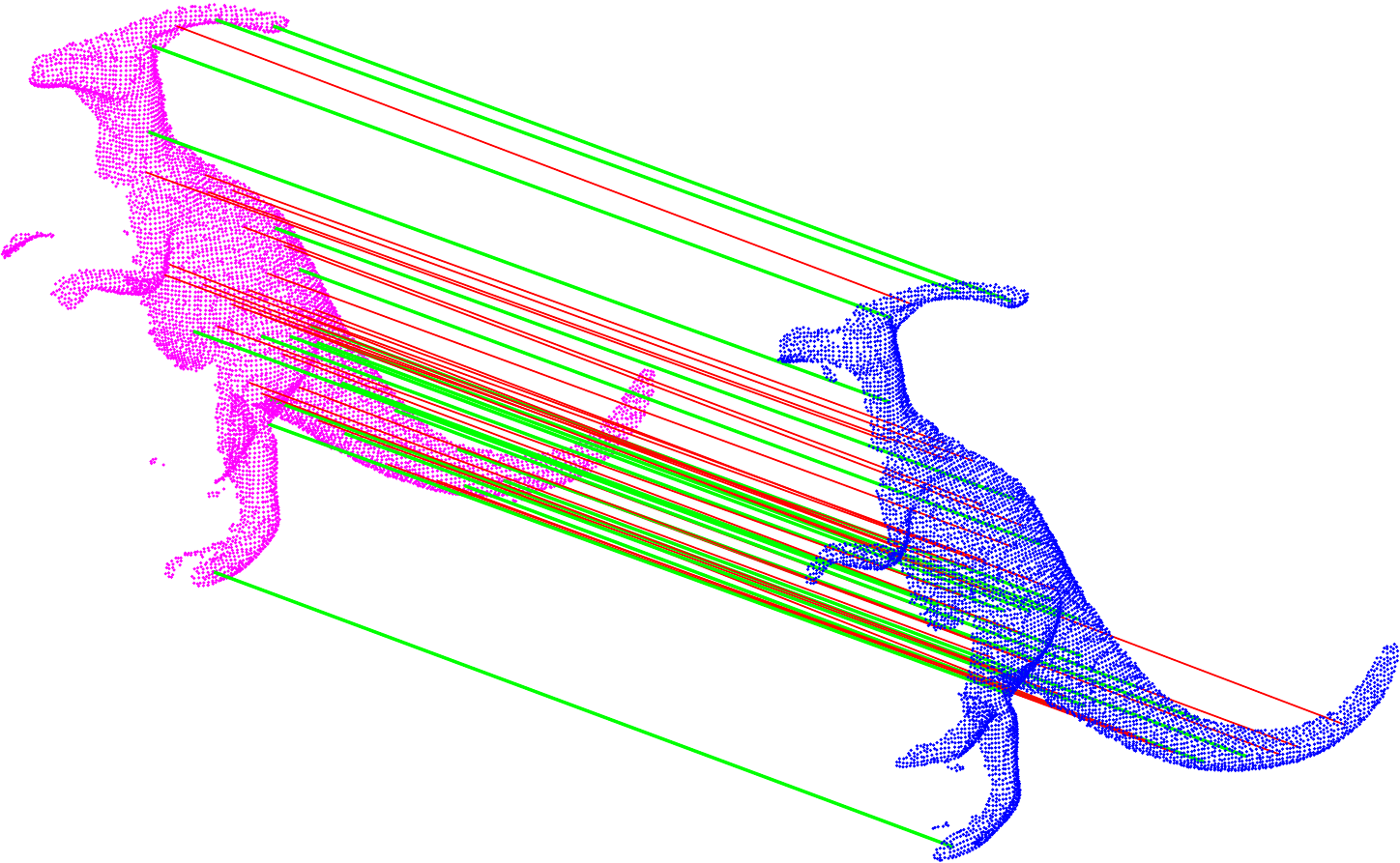}&
		\includegraphics[height=1.7cm]{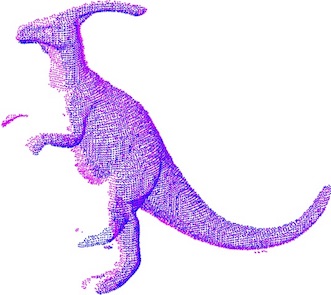}
		\\
		\hline
&&&&\\[-.35cm]
		\multirow{3}{*}{\rotatebox{90}{ Mining~~}} 
		&\rotatebox[origin=l]{90}{\footnotesize\emph{mining-a}}&
		\includegraphics[height=1.7cm]{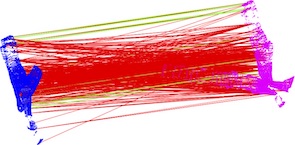}&
		\includegraphics[height=1.7cm]{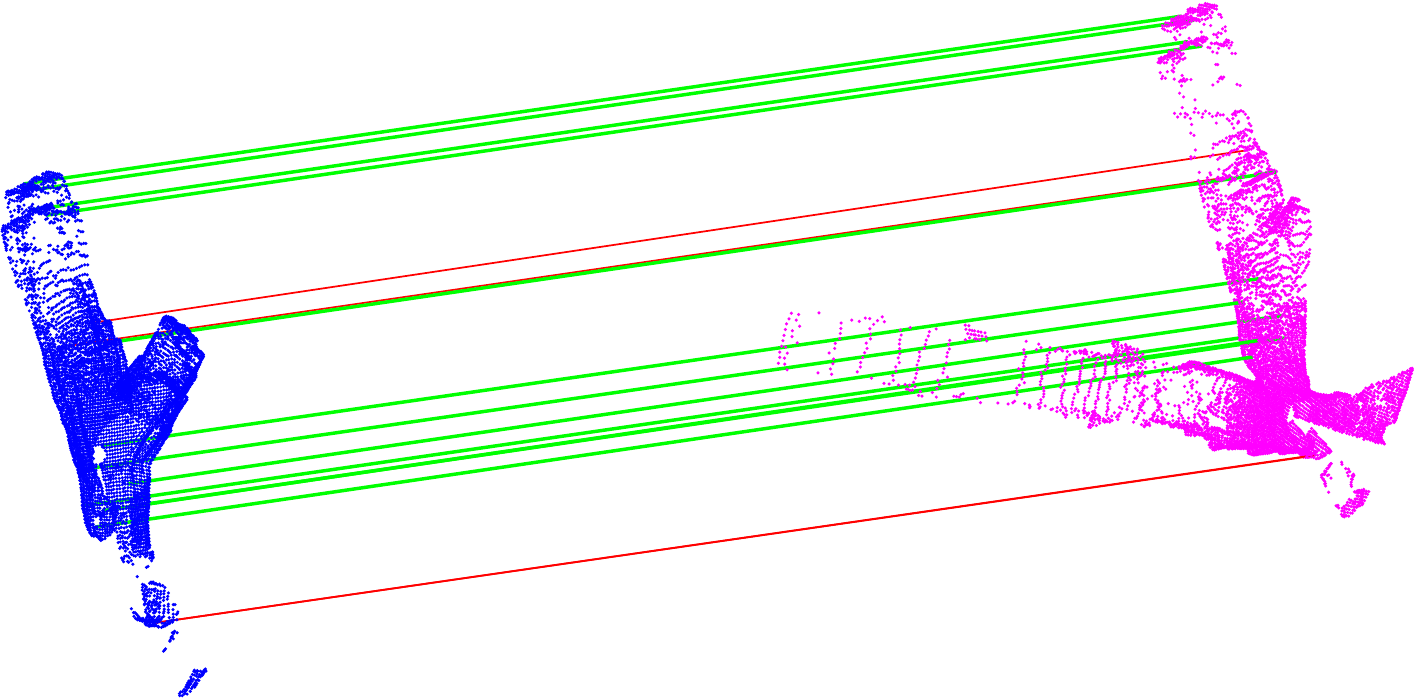}&
		\includegraphics[height=1.7cm]{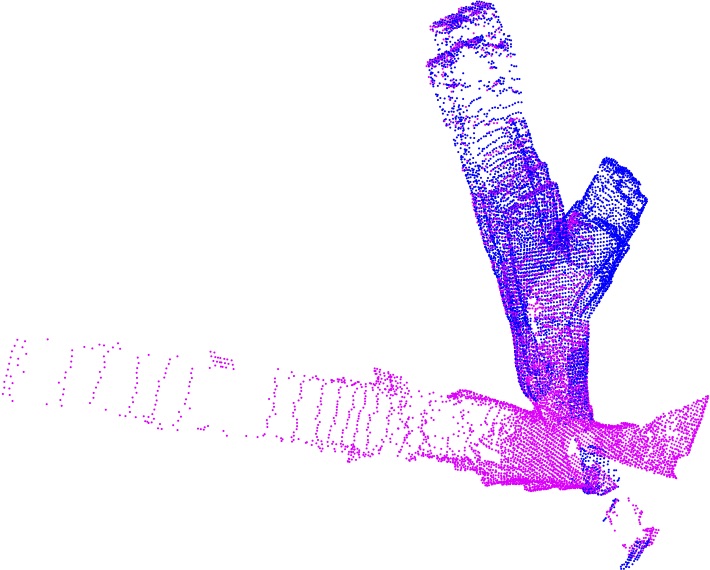}
		\\
		&\rotatebox[origin=l]{90}{\emph{mining-b}}&
		\includegraphics[height=1.8cm]{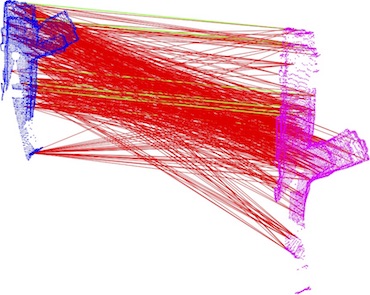}&
		\includegraphics[height=1.8cm]{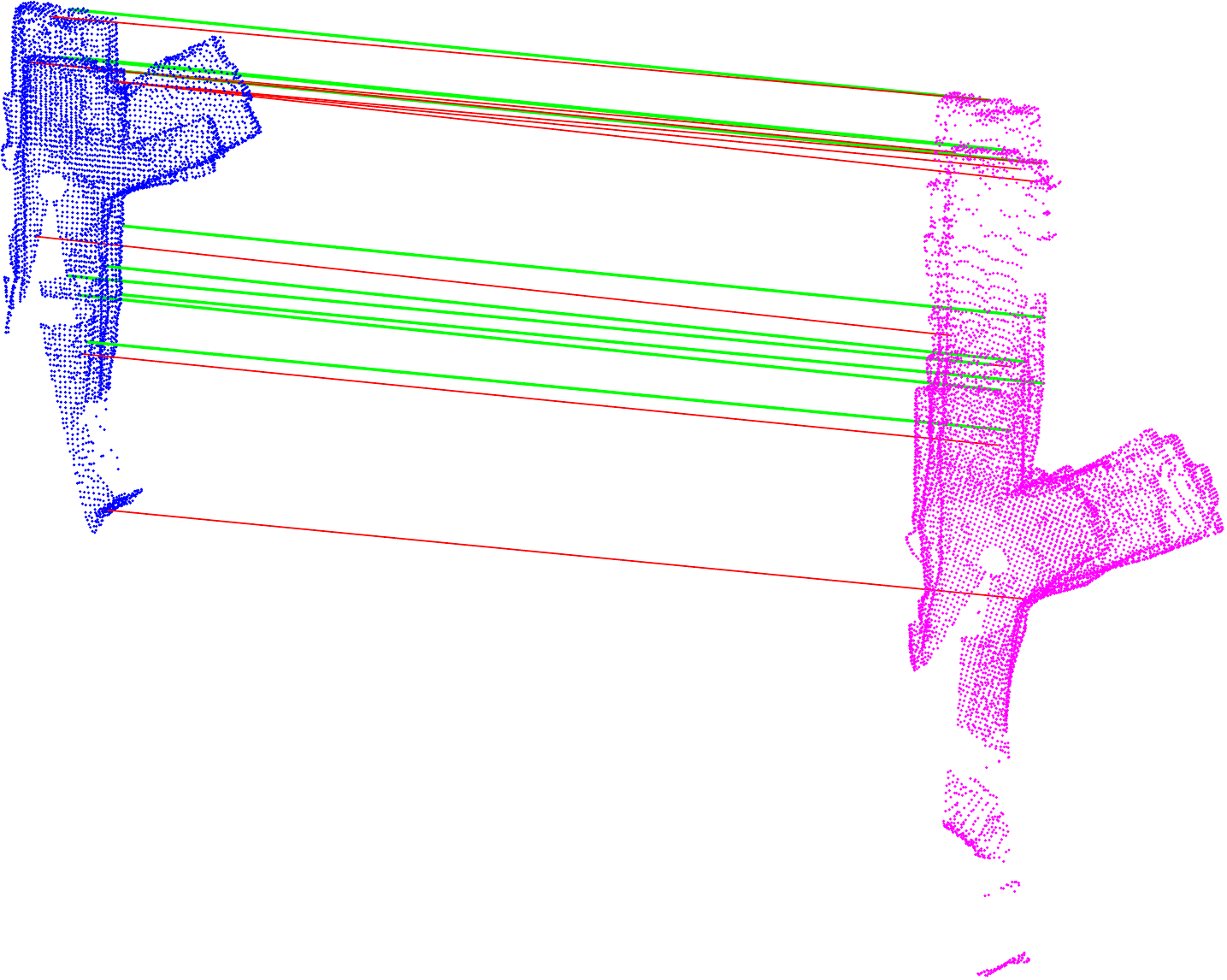}&
		\includegraphics[height=1.8cm]{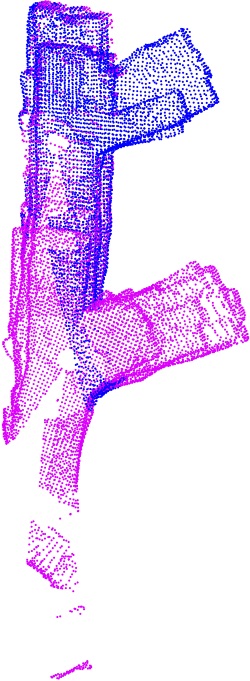}
		\\
\hline
&&&&\\[-1em]
\multirow{2}{*}{\rotatebox{90}{ Remote Sensing \hspace{-2em}}} 
&\rotatebox[origin=l]{90}{\emph{vaihingen-a}}&
\includegraphics[height=2cm]{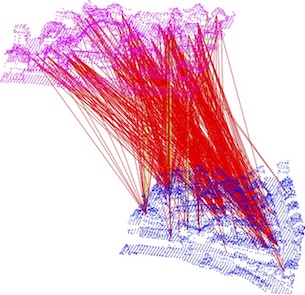}&
\includegraphics[height=2cm]{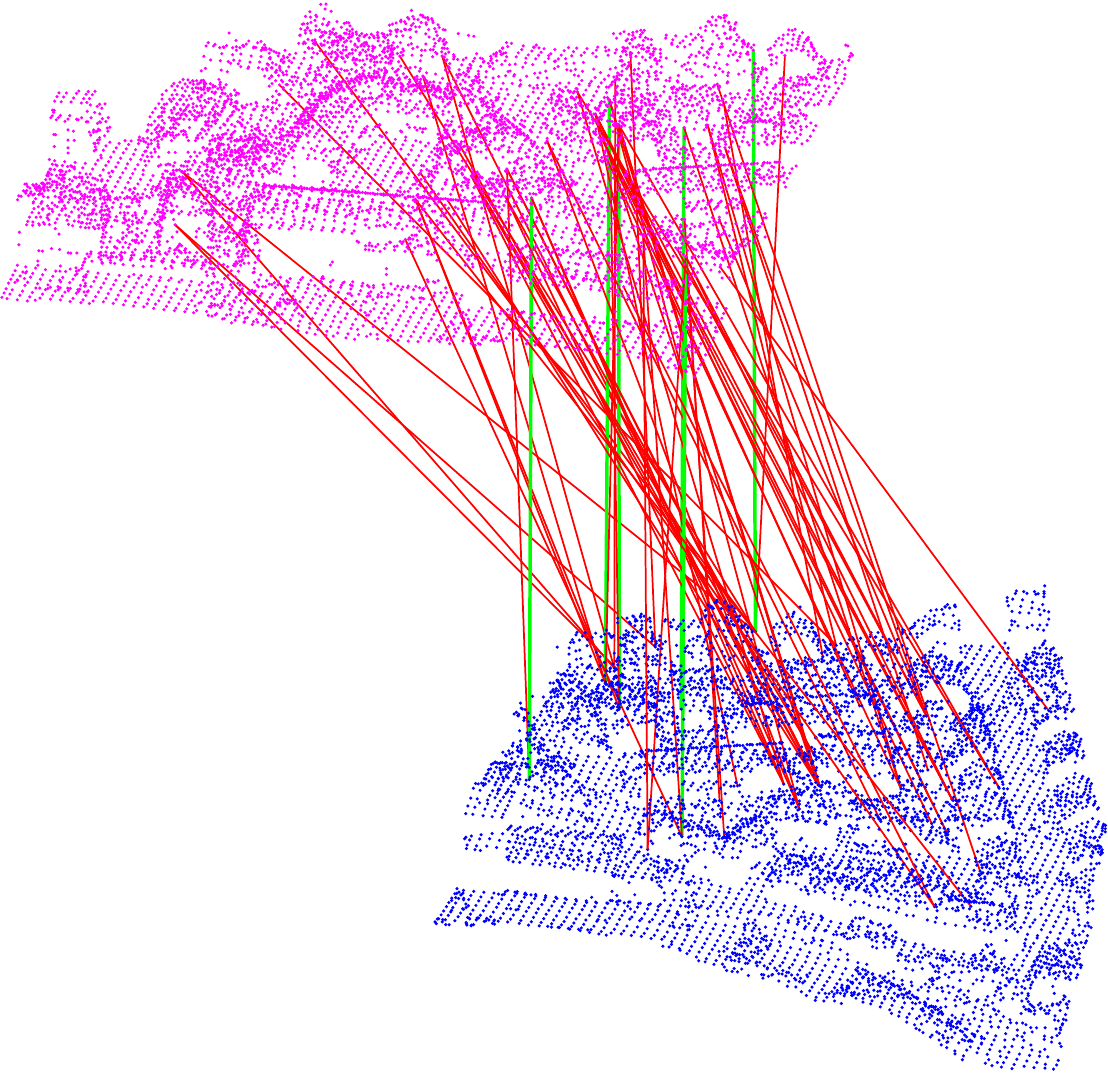}&
\includegraphics[height=1.8cm]{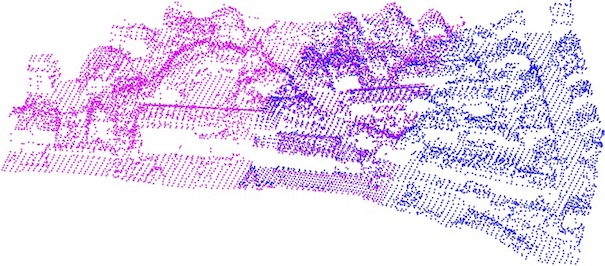}
\\
&\rotatebox[origin=l]{90}{\emph{vaihingen-b}}&
\includegraphics[height=2.2cm]{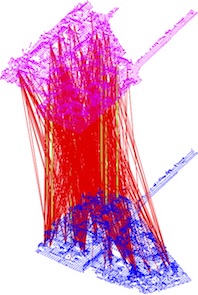}&
\includegraphics[height=2.2cm]{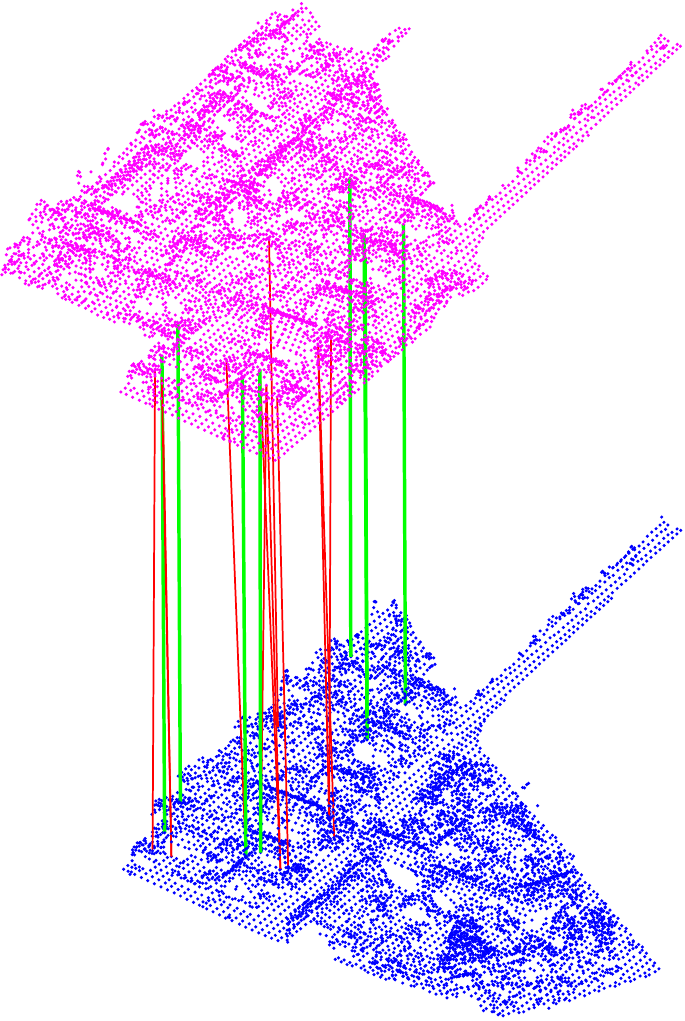}&
\includegraphics[height=2cm]{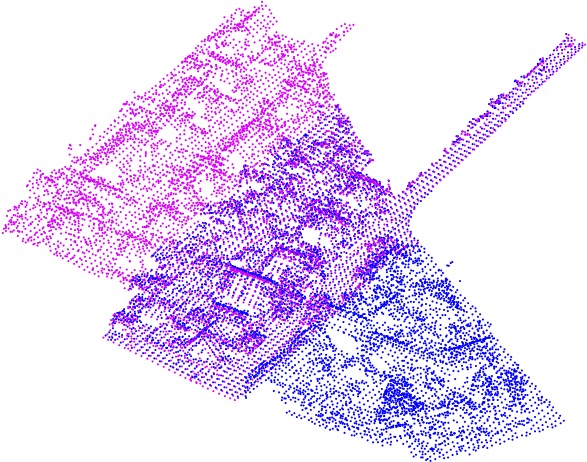} 
\\
\end{tabularx}
\caption{Qualitative results of GORE for 6 DoF Euclidean registration~\eqref{eq:rigid} with $N=500$. Column 1: Input correspondences (true inliers are represented by green lines, and true outliers by red lines). Column 2: Data remaining after GORE. Column 3: Registration using approximate solution $\tilde{\bT}$ produced by GORE.}
\label{fig:6dofresults}	
\end{figure*}

\subsubsection{Quantitative benchmarking}\label{sec:benchpcr}

We tested the following approaches or pipelines: GORE, GORE+RANSAC, RANSAC, GORE+BnB, and BnB (for BnB, the 6 DoF variant in~\cite{parrabustos16_pami} was used). We recorded the following measures for each method/run.
\begin{itemize}
\item $|\cI|$: Consensus size of best solution found.
\item RMSE: Root mean square error between the transformed points in the consensus (with the best solution found) and the true locations.
\item $|\cH^\prime|$: Number of remaining data after GORE.
\item time (s): Total runtime.  A timeout of $5$ hours ($18,000$ seconds) was imposed on all methods.
\item $|\cI^\ast|$: Consensus size of global solution.
\end{itemize}
The median values over all $10$ data instances per every ($N$, dataset) combination are summarized in Table~\ref{tab:comparison6dof}  ($|\cI|$,$|\cH'|$ and $|\cI^*|$) and in Fig.~\ref{fig:pcr_timerms}  (time and RMS). RMSE was lower than $0.5$ for all methods showing a correct  alignment. Not all BnB runs finalized within the time limit; see Table~1 in Sec.~B (supp.~material) for results on individual objects.



\begin{figure*}
\centering
\scriptsize
\begin{tabularx}{\textwidth}{C{1} C{1} C{1} C{1}}
	Stanford & Mian & Mining & Remote Sensing\\
	\hline
\includegraphics[width=.24\textwidth]{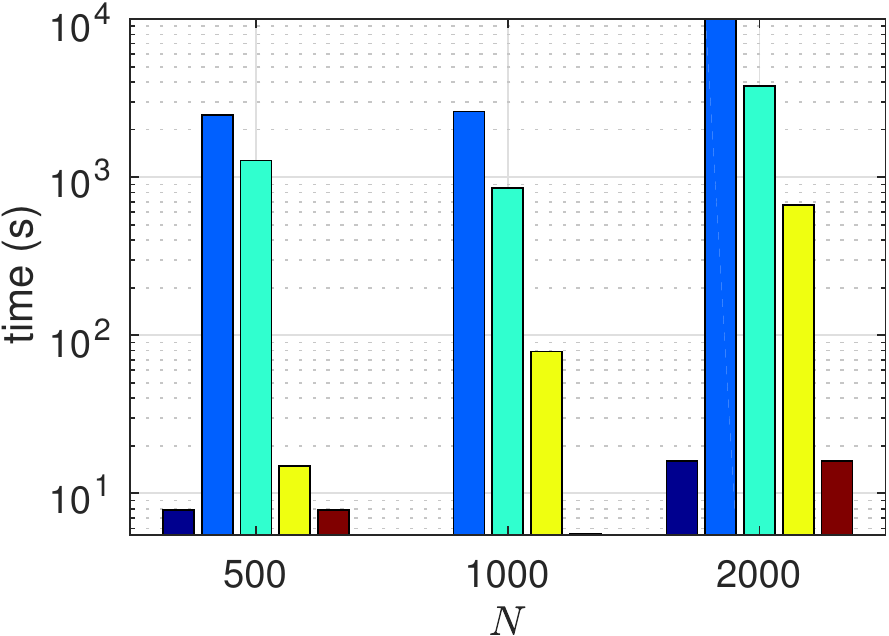} &
\includegraphics[width=.24\textwidth]{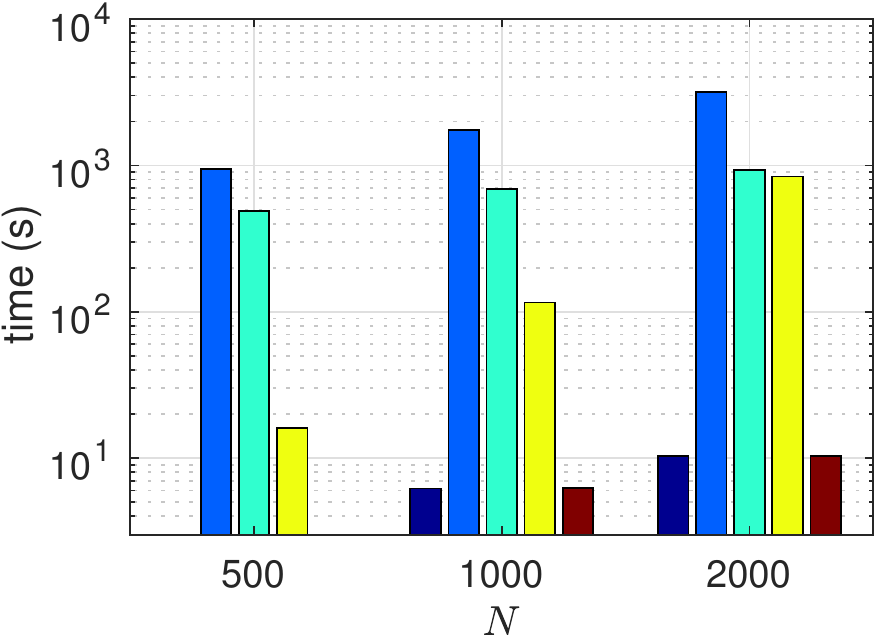}     &
\includegraphics[width=.24\textwidth]{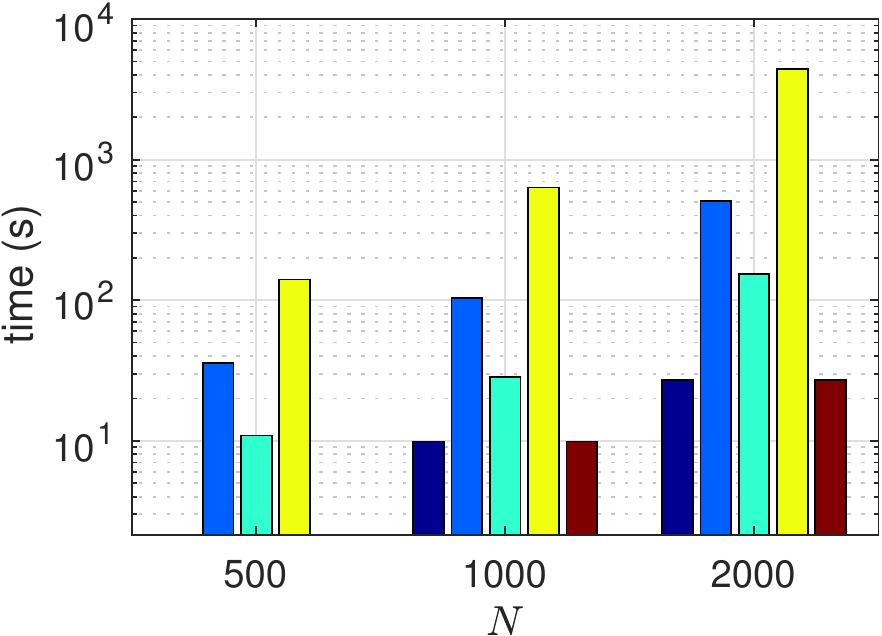}   &
\includegraphics[width=.24\textwidth]{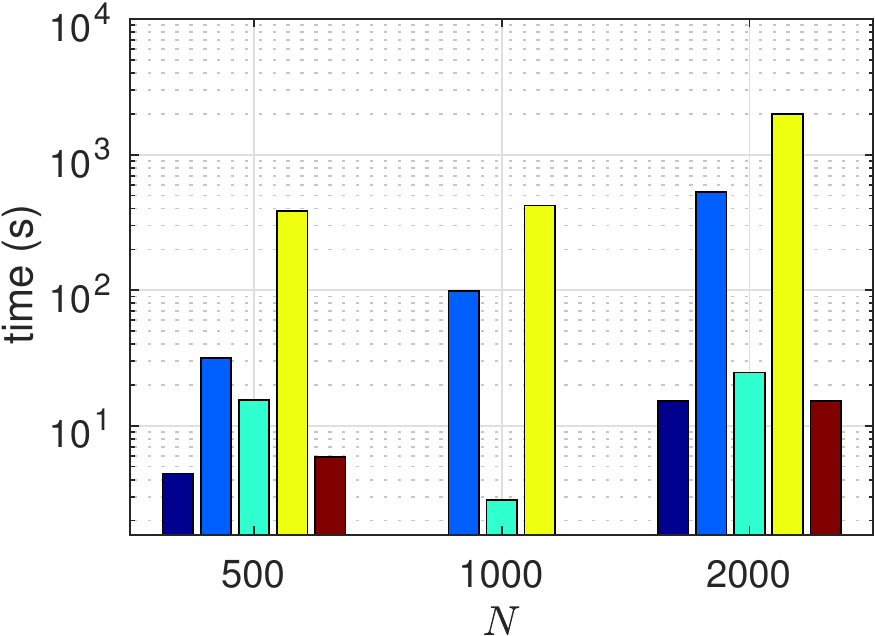}\\[-.3cm]
\includegraphics[width=.24\textwidth]{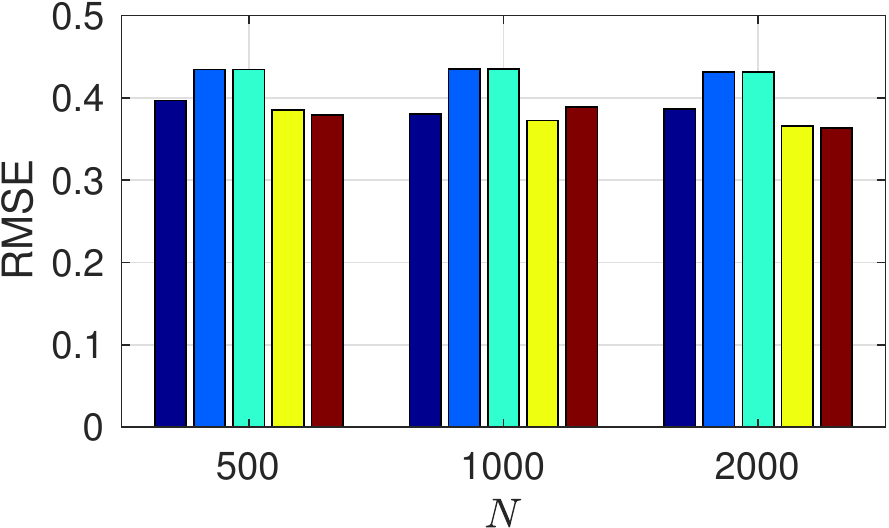}  &
\includegraphics[width=.24\textwidth]{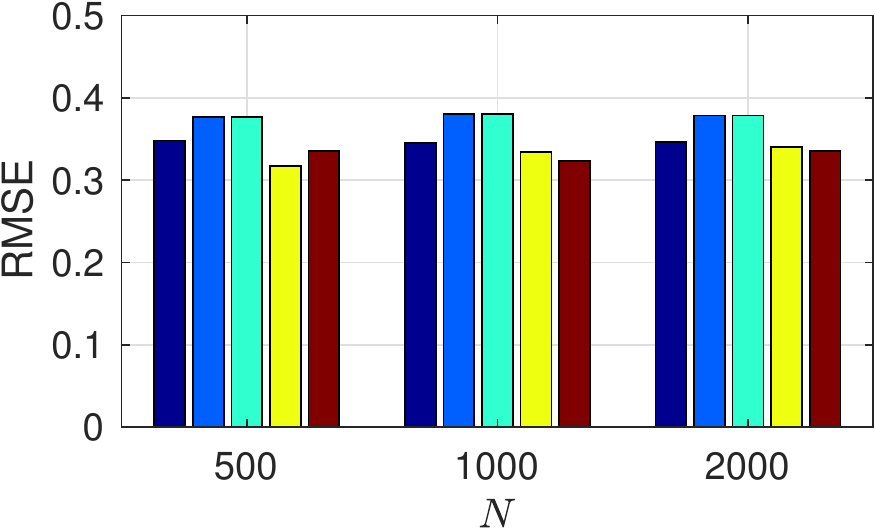}      &
\includegraphics[width=.24\textwidth]{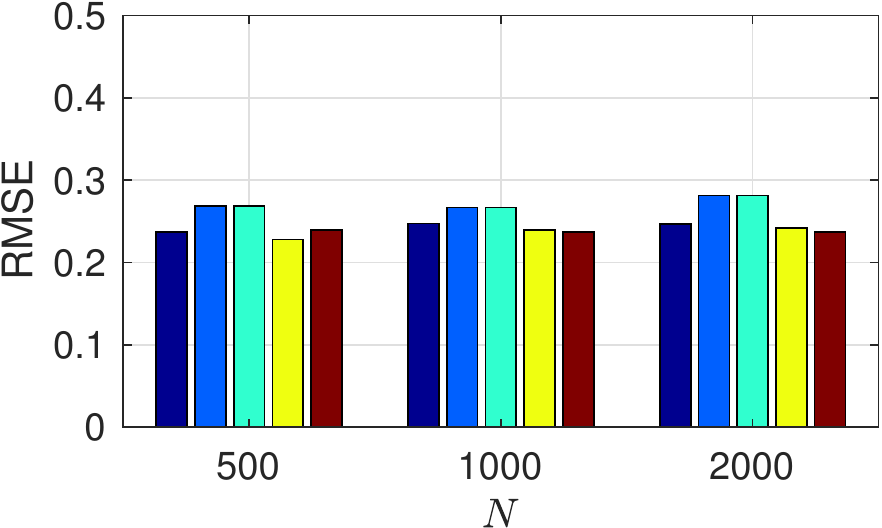}    &
\includegraphics[width=.24\textwidth]{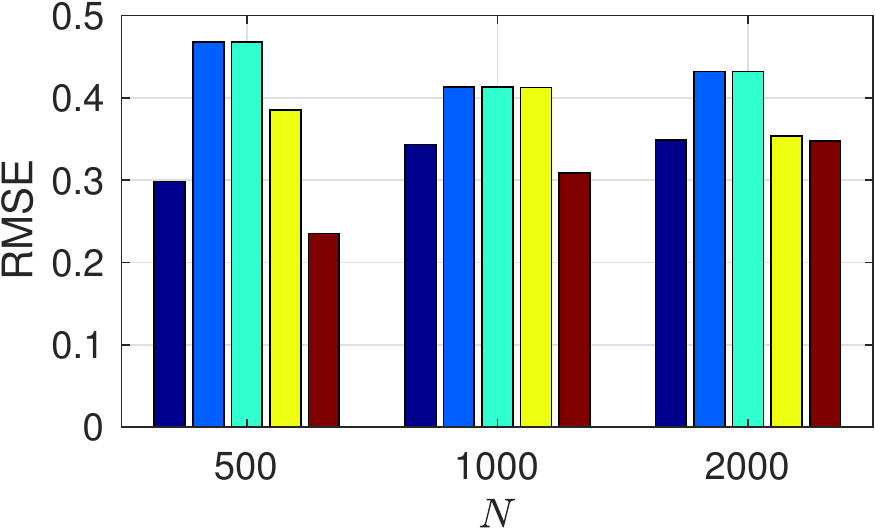} \\[-.25cm]
\multicolumn{4}{c}{
	\setlength{\fboxsep}{1pt}%
	\fbox{
		\includegraphics[height=.24cm]{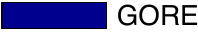} \;\;\; 
		\includegraphics[height=.24cm]{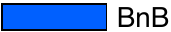}\;\;\; 
		\includegraphics[height=.24cm]{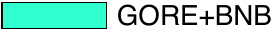}\;\;\; 
		\includegraphics[height=.24cm]{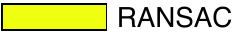}\;\;\;
		\includegraphics[height=.24cm]{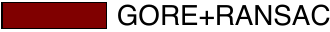}
	}	
}\\
\end{tabularx}
\caption{Time (top row) and RMSE (bottom row) per every ($N$, dataset) combination for 6 DoF Euclidean registration. Time is plotted in log scale. }
\label{fig:pcr_timerms}
\end{figure*}

\begin{table*} 
	\begin{center}
		
			\begin{tabular}{|c|c|lr|r|r|r|r|}
				\cline{5-8}
				\multicolumn{4}{c|}{} & 
				\multicolumn{1}{c|}  {\multirow{1}{*}{Stanford}} & 
				\multicolumn{1}{c|}  {\multirow{1}{*}{Mian}} & 
				\multicolumn{1}{c|}  {\multirow{1}{*}{Mining}} &
				\multicolumn{1}{c|}  {\multirow{1}{*}{Remote Sensing}} \\
				\multicolumn{4}{c|}{}  & 
				$|\cX|=7767.5$ & $|\cY|=7807.0$ & $|\cX|=7667.0$ & $|\cX|=7780.5$ \\
				\multicolumn{4}{c|}{}  & 
				$|\cY|=7172.0$ & $|\cX|=8000.5$ & $|\cY|=6452.0$ & $|\cY|=7448.0$\\
				\hline
				
				\multirow{5}{*}{$N = 500$} & 
				
				\multicolumn{3}{c|}{outlier ratio} & 0.96 & 0.96 & 0.98 & 0.99\\
				\cline{2-8}
				
				& \multirow{6}{*}{\rot{Pipeline}} 
				
				&\multirow{2}{*}{GORE} & $|\cI|$ & 18 & 17 & 9 & 5\\
				& & & $|\cH^\prime|$ & 61 & 46 & 18 & 46\\
				& & GORE+BnB & $|\cI^*|$ & 25 & 21 & 12 & 7\\
				& & GORE+RANSAC & $|\cI|$ &	21 & 18 & 10 & 6\\
				& & RANSAC & $|\cI|$ & 21 & 17 & 8 & 6\\
				\hline
				
				\multirow{5}{*}{$N=1000$} & 
				
				\multicolumn{3}{c|}{outlier ratio} &
				0.98 & 0.98 & 0.99 & 0.99\\
				\cline{2-8}
				
				& \multirow{6}{*}{\rot{Pipeline}} 
				&\multirow{2}{*}{GORE} & $|\cI|$ & 21 & 18 & 10 & 11\\
				& & & $|\cH^\prime|$ & 60 & 57 & 25 & 14\\
				& & GORE+BnB & $|\cI^*|$ & 27 & 22 & 12 & 12\\
				& & GORE+RANSAC & $|\cI|$ & 20 & 19 & 10 & 11\\
				& & RANSAC & $|\cI|$ & 22 & 19 & 10 & 11\\
			\hline
				
				\multirow{5}{*}{ $N=2000$ } & 
				
				\multicolumn{3}{c|}{outlier ratio} & 
				0.98 & 0.99 & 0.99 & 0.99\\
				\cline{2-8}
				
				& \multirow{6}{*}{\rot{Pipeline}} 
				
				&\multirow{2}{*}{GORE} & $|\cI|$ & 	22 & 19 & 12 & 15\\
				& & & $|\cH^\prime|$ & 69 & 64 & 35 & 21\\
				& & GORE+BnB & $|\cI^*|$ & 28 & 24 & 14 & 17\\
				& & GORE+RANSAC & $|\cI|$ & 22 & 19 & 12 & 15\\
				& & RANSAC & $|\cI|$ &	24 & 19 & 12 & 15\\
			\hline
				
			\end{tabular}
	\end{center}
	\caption{6 DoF Euclidean registration results.}
	\label{tab:comparison6dof}
\end{table*}

The result of GORE could vary depending on the initial relative pose of the point clouds (usually by  $\leq2$ points in ours experiments), since the suboptimal rotation~\eqref{eq:suboptrot} used by GORE to reject outliers is obtained from minimum geodesic motions. However, initial pose does not have to be "close" to the correct alignment for GORE to be efficient. 

On all the ($N$, dataset) combinations, GORE was able to terminate within $30$ seconds; in fact, for $N\leq 1000$, it was able to finish within $10$ seconds in all the datasets. Also, for most of the instances, GORE reduced the input $\cH$ to a very small subset $\cH^\prime$ of size $< 70$. Due to the efficacy of GORE in reducing the data size and outlier rate, the time required to attain a solution (either global or local) to the robust registration problem~\eqref{eq:rigid} has been greatly reduced.


Specifically, the combination GORE+BnB was able to find the globally optimal result using typically only half of the time required by BnB alone. This was due to the massive reduction of data size before BnB - in this experiment, after GORE the median problem size to BnB was just $50$. Note that GORE+BnB is guaranteed to be globally optimal.


As evidenced from the significant runtimes, RANSAC is not a practical algorithm to solve~\eqref{eq:rigid} due to the very high outlier rates. However, the strategy GORE+RANSAC was able to drastically reduce RANSAC's runtime and yield very satisfactory registration results. In fact, we see that the runtimes of GORE+RANSAC were two orders of magnitude smaller than RANSAC alone.



\subsection{Depth map registration}

\begin{figure*}
	\footnotesize
	\begin{tabularx}{\textwidth}{c|c|C{.7}|C{.7}|C{.5}}
		& & {\footnotesize Input correspondences $\cH$} & {\footnotesize Remaining data $\cH^\prime$ after GORE} & {\footnotesize Registration using approximate solution $\tilde{\bT}$ from GORE}\\ 
		\hline
		&&&&\\[-.3cm]
		\multirow{2}{*}{\rotatebox{90}{ Stitching (Microsoft 7-scenes)\hspace{-1.5cm}} } 
		&\rotatebox[origin=l]{90}{\emph{\hspace{2em}office}}&
		\includegraphics[height=2.1cm]{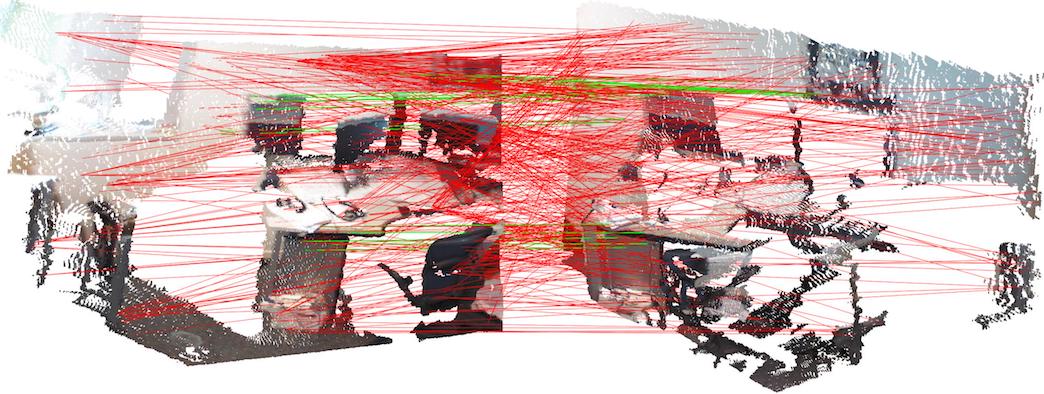}&
		\includegraphics[height=2.1cm]{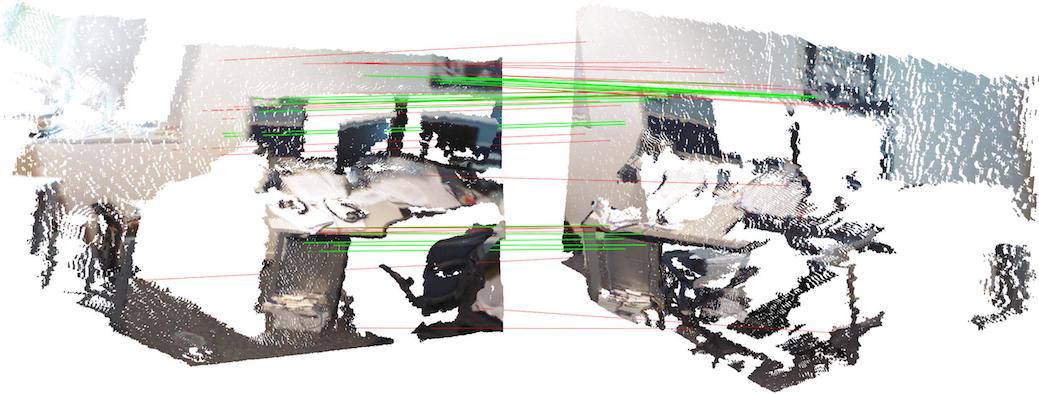}&
		\includegraphics[height=2.3cm]{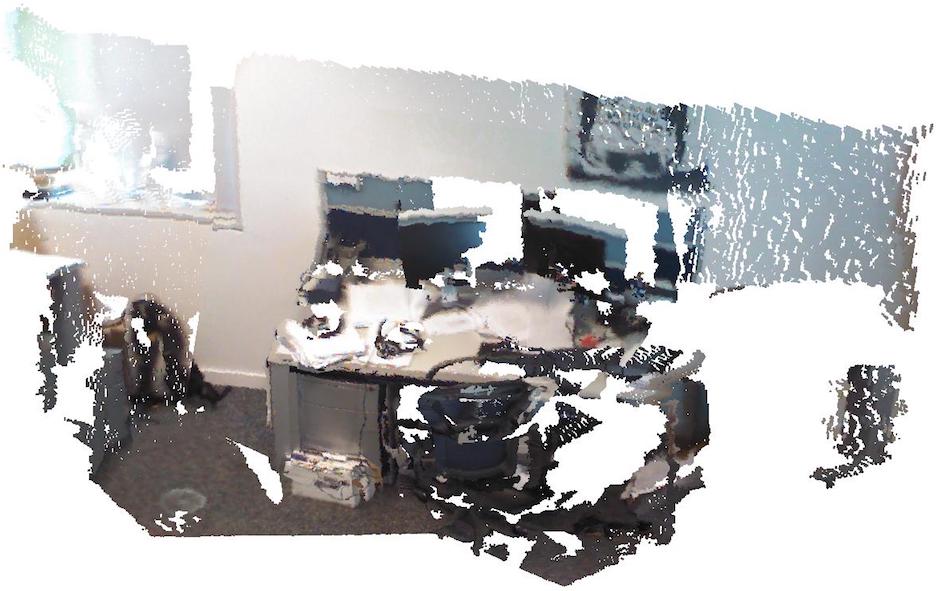}
		\\[-.3cm]
		&\rotatebox[origin=l]{90}{\emph{\hspace{4em}kitchen}}&
		\includegraphics[height=2.5cm]{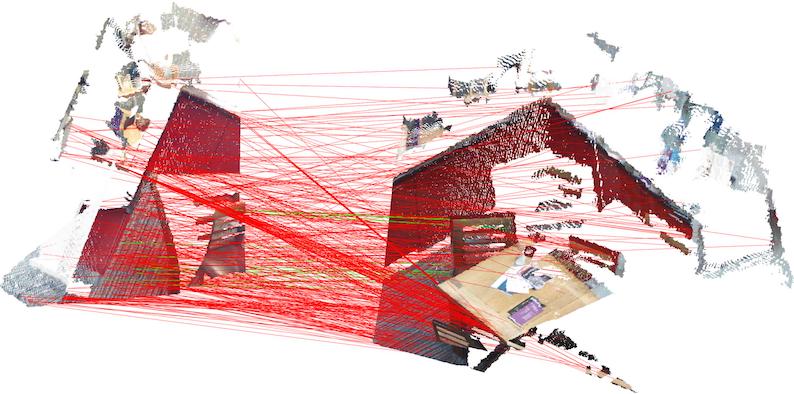}&
		\includegraphics[height=2.5cm]{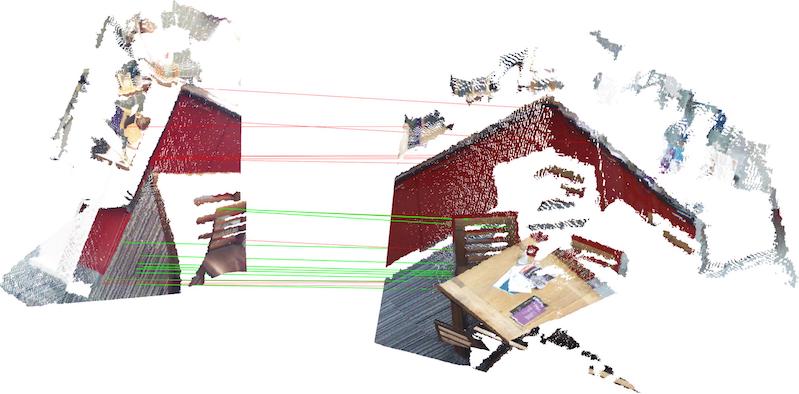}&
		\includegraphics[height=2.5cm]{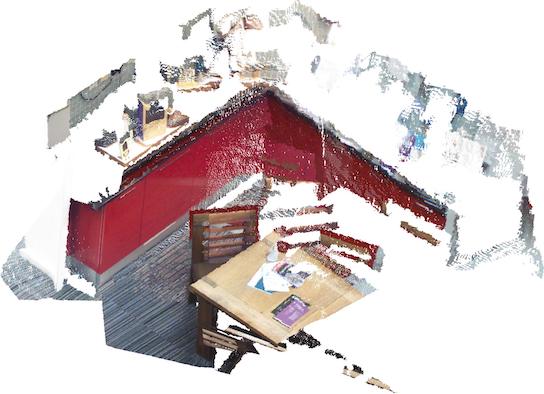}
		\\
		\hline
		&&&&\\[-.2cm]
		\multirow{2}{*}{\rotatebox[origin=c]{90}{ Localization (RGB-D objects and scenes dataset) \hspace{-3.3cm}}} &
		\rotatebox{90}{\hspace{.5cm}(\emph{cereal-box}, \emph{scene13})} & 
		\includegraphics[height=3.7cm]{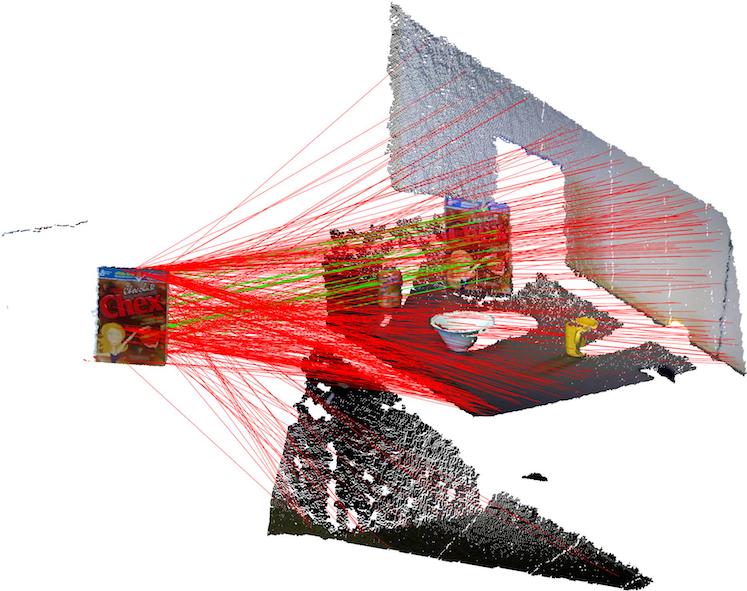} &
		\includegraphics[height=3.7cm]{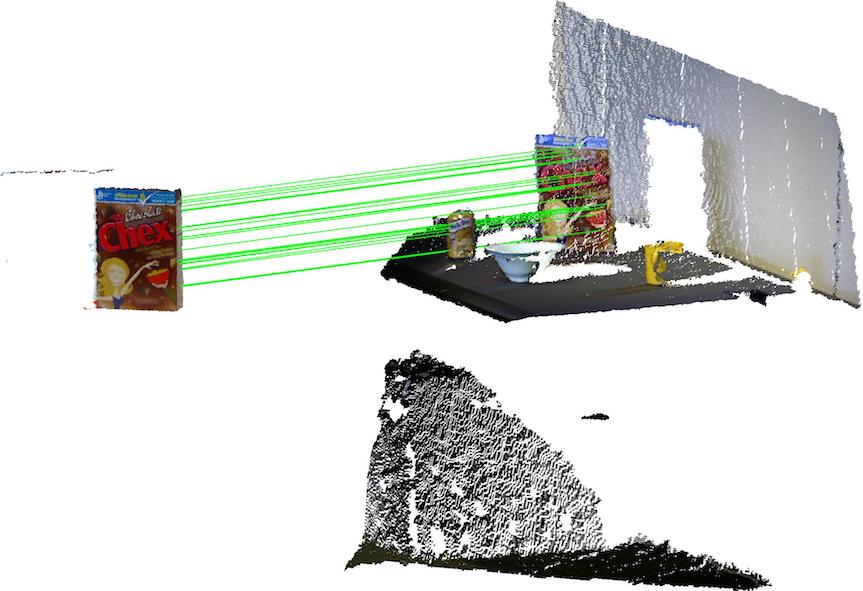} &
		\includegraphics[height=3.7cm]{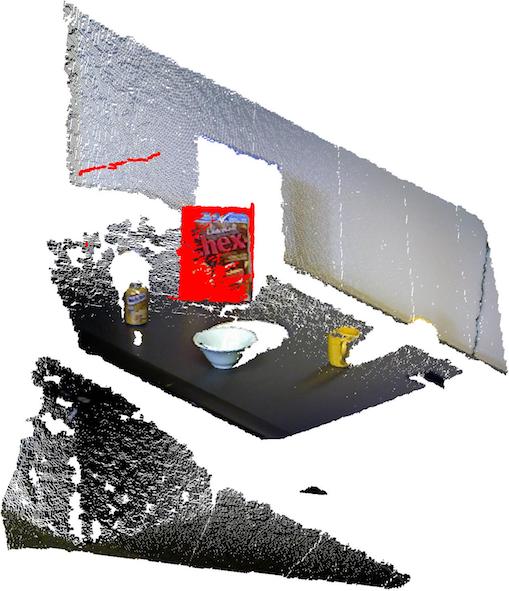} \\[-.2cm]
		&\rotatebox[origin=l]{90}{\hspace{1em} (\emph{cap}, \emph{scene12})}&
		\includegraphics[height=2.7cm]{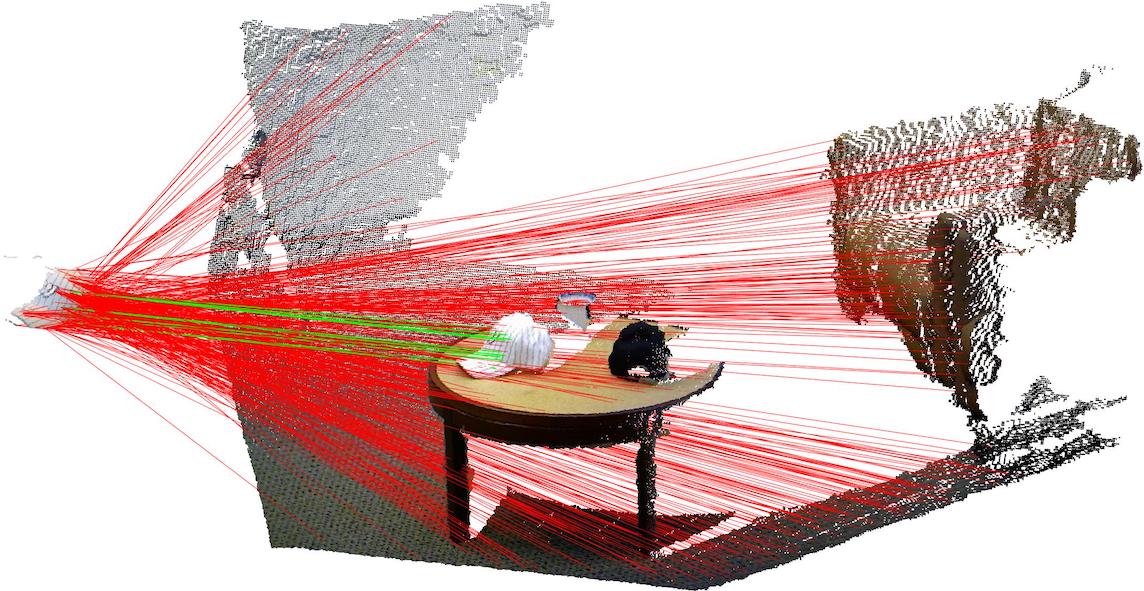}&
		\includegraphics[height=2.7cm]{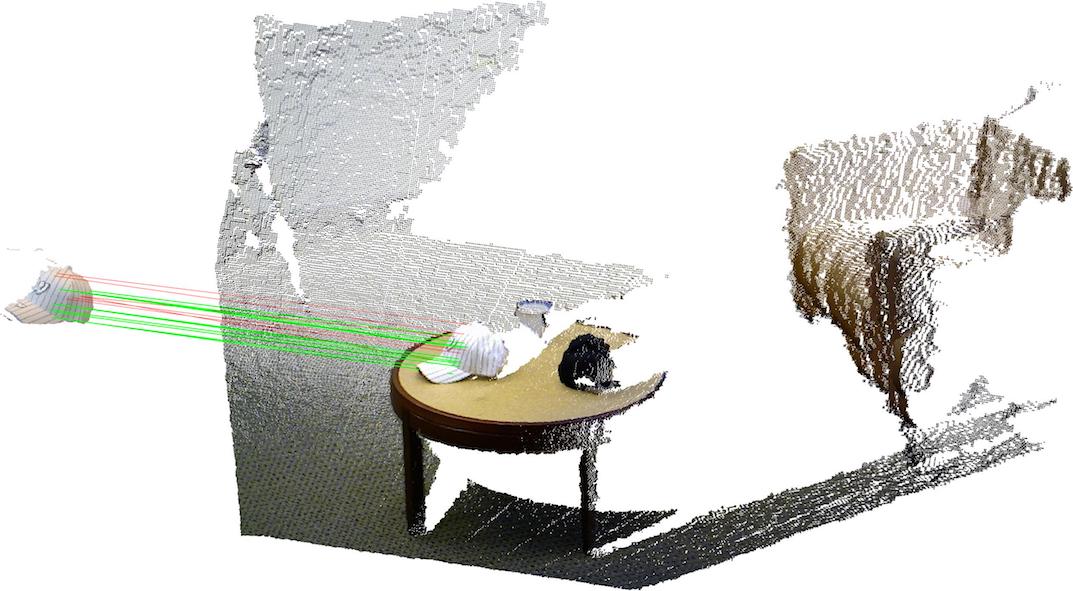}&
		\includegraphics[height=2.7cm]{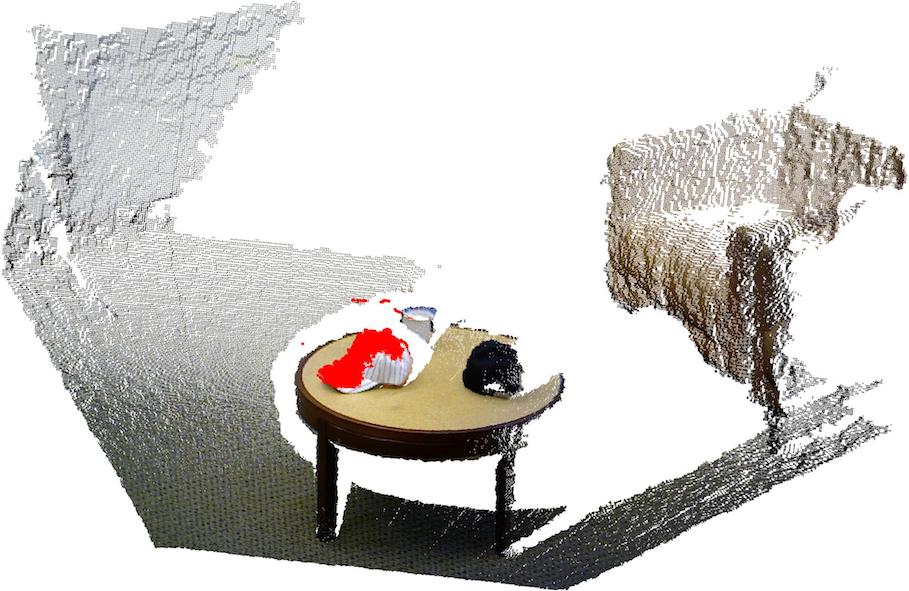}\\[-.3em]
	\end{tabularx} 
	\caption{Qualitative results of GORE for 6 DoF Euclidean registration~\eqref{eq:rigid} with $N=500$ on RGB-D data. Column 1: Input correspondences (true inliers are represented by green lines, and true outliers by red lines). Column 2: Data remaining after GORE. Column 3: Registration using approximate solution $\tilde{\bT}$ produced by GORE. For localization, the localized object is colored in red. }
	\label{fig:rgbdresults}
\end{figure*}

\subsubsection{Experimental setup}

We tested GORE on depth maps (acquired, e.g., using RGB-D cameras) based on two problems:
\begin{itemize}
	\item \textbf{Stitching}: Registering two 3D views of a scene. 
	\item \textbf{Localization}: Registering a 3D object against a cluttered 3D scene.
\end{itemize}

For stitching, we used views from the Microsoft 7-scenes dataset~\cite{shotton13}. Two partially overlapped views of the \emph{office} (Row~1 in Fig.~\ref{fig:rgbdresults}) and \emph{kitchen} (Row~2) scenes were selected. We repeated the experimental setup in Sec.~\ref{sec:expsetuppcr}, except to generate  correspondences. Since ISS3D was inaccurate on RGB-D data, we used SIFT on the associated RGB images. To obtain exact $N$ 3D point correspondences, we select the top-$N$ SIFT correspondences with valid depth values.
 
For localization, we repeated the above setup on (object, scene) pairs from the RGB-D object and scenes datasets~\cite{lai11}. We selected 3D views for the pairs (\emph{cereal-box}, \emph{scene-13}), and (\emph{cup}, \emph{scene-12}); see the last 2 columns in Fig.~\ref{fig:rgbdresults}.

\subsubsection{Qualitative evaluation}
Analogously to Sec.~\ref{sec:pcrqualitative},  Fig.~\ref{fig:rgbdresults} shows qualitative results for both problems with $N=500$. All of these instances contain more than $95\%$ outliers. GORE terminated within $7$ seconds for stitching, and in $< 0.5$ seconds for localization (see Sec.~B in the supp.~material for results on individual instances). GORE produced a reduced set $\cH^\prime$ with  $<10\%$ of the $N$ original points, and a satisfactory approximate transformation $\tilde{\bT}$.

\subsubsection{Quantitative benchmarking}
We tested same pipelines and recorded the same measurements that in Sec.~\ref{sec:benchpcr}. Table~\ref{tab:comparisonrgbd} and Fig.~\ref{fig:rgbd_timerms} summarizes the median values  over all $10$ data instances per problem and $N$. For all methods, RMSE was lower than $0.6$ for stitching and lower than $0.4$  for localization what shows a correct  alignment.

\begin{figure}
\centering
\scriptsize
\begin{tabularx}{\linewidth}{C{1} C{1} }
	Stitching  & 
    Localization \\
(Microsoft 7-scenes) & 
    (RGB-D objects and scenes dataset) \\
	\hline
	\includegraphics[width=\linewidth]{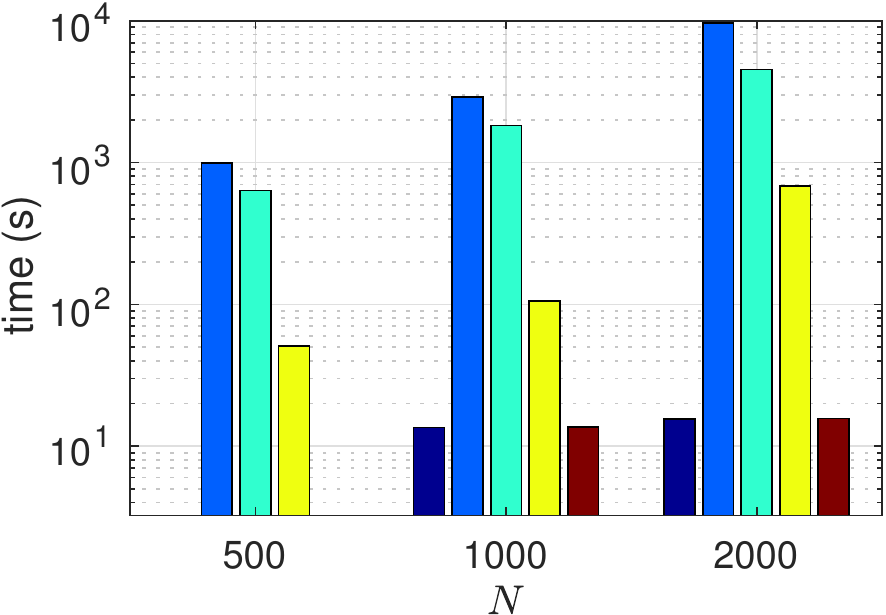} &
	\includegraphics[width=\linewidth]{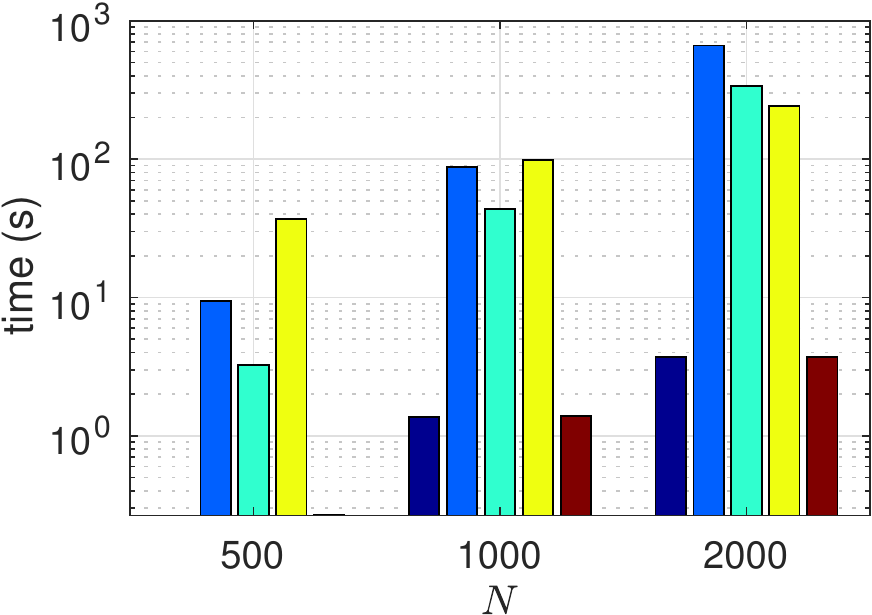}     \\
	\includegraphics[width=\linewidth]{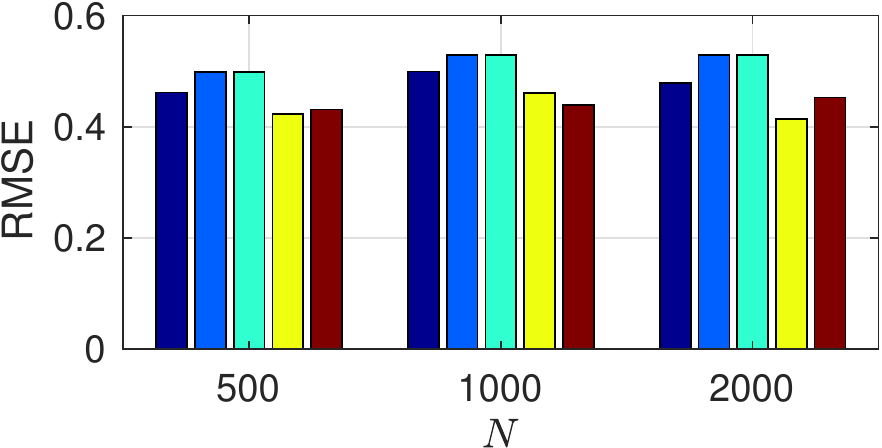}  &
	\includegraphics[width=\linewidth]{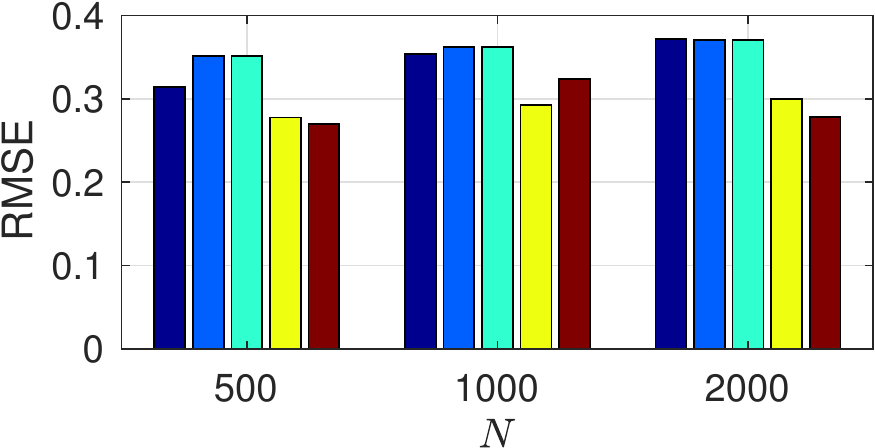}      \\
\end{tabularx}
{
	\setlength{\fboxsep}{1pt}%
	\fbox{
		\begin{minipage}{.9\linewidth}
			\centering
			\begin{tabular}{lll}
		\includegraphics[height=.23cm]{rgbd/legend_gore} &
		\includegraphics[height=.23cm]{rgbd/legend_bnb} & 
		\includegraphics[height=.23cm]{rgbd/legend_gorebnb}\\ 
		\includegraphics[height=.23cm]{rgbd/legend_ransac} &
		\includegraphics[height=.23cm]{rgbd/legend_goreransac} &\\[-.7mm]
			\end{tabular}
		\end{minipage}	
	}	
}
\caption{Time (top row) and RMSE (bottom row) per $N$ for stitching and localization problems on RGB-D data. Time is plotted in log scale.}
\label{fig:rgbd_timerms}
\end{figure}

\begin{table} 
	\begin{center}
		
		\begin{tabular}{|c|c|lr|r|r|}
			\cline{5-6}
			\multicolumn{4}{c|}{} & 
			\multicolumn{1}{c|}  {\multirow{1}{*}{Stitching}} & 
			\multicolumn{1}{c|}  {\multirow{1}{*}{Loc.}} \\
			\hline
			
			\multirow{6}{*}{\rot{$N=500$}} & 
			
			\multicolumn{3}{c|}{outlier ratio} & 0.97 & 0.96 \\
			\cline{2-6}
			
			& \multirow{6}{*}{\rot{Pipeline}} 
			
			&\multirow{2}{*}{GORE} & $|\cI|$ & 
			13 & 16\\ 
			& & & $|\cH^\prime|$ & 36 & 25\\
			
			&  &\multirow{1}{*}{GORE$+$BnB} & $|\cI^*|$ & 16 & 18\\
			
			
			& &\multirow{1}{*}{GORE$+$RANSAC} & $|\cI|$ & 13 & 17\\
			
			&  &\multirow{1}{*}{RANSAC} & $|\cI|$ & 14 & 17\\
			
			\hline
			\multirow{6}{*}{\rot{$N=1000$}} & 
			
			\multicolumn{3}{c|}{outlier ratio} & 0.98 & 0.97 \\
			\cline{2-6}
			
			& \multirow{5}{*}{\rot{Pipeline}} 
			
			&\multirow{2}{*}{GORE} & $|\cI|$ & 20 & 24\\ 
			& & & $|\cH^\prime|$ & 77 & 40\\
			
			&  & GORE$+$BnB & $|\cI^*|$ & 24 & 27\\
			
			
			& & GORE$+$RANSAC & $|\cI|$ & 20 & 26\\
			
			&  & RANSAC & $|\cI|$ & 20 & 25\\
			\hline
			
			\multirow{6}{*}{\rot{$N=2000$}} & 
			
			\multicolumn{3}{c|}{outlier ratio} & 
			0.98 & 0.98 \\
			\cline{2-6}
			
			& \multirow{5}{*}{\rot{Pipeline}} 
			
			&\multirow{2}{*}{GORE} & $|\cI|$ & 24 & 42\\ 
			& & & $|\cH^\prime|$ & 111 & 61\\
			
			&  &GORE$+$BnB & $|\cI^*|$ & 31 & 46\\
			
			
			&  & GORE$+$RANSAC & $|\cI|$ & 24 & 44\\
			
			& & RANSAC & $|\cI|$ & 24 & 43\\
			\hline
		\end{tabular}
	\end{center}
	\caption{RGB-D registration results.}
	\label{tab:comparisonrgbd}
	\vspace{-2.1em}
\end{table}

GORE took $<16$ seconds for stitching and $<4$ seconds for localization to reduce $\cH$ to a very small subset $\cH^\prime$ of size  $<120$. Similar to point cloud registration, GORE+BnB took only half of the runtime of BnB, and GORE+RANSAC gave a 2-orders-of-magnitude reduction in runtime than running RANSAC alone.

\subsection{Investigating the role of BnB in Algorithm~\ref{alg:gore}} \label{sec:usageofbnb}

To investigate the effects on accuracy and runtime due to the usage of BnB in Algorithm~\ref{alg:gore} (Line~\ref{step:bnb}), we ran Algorithm~\ref{alg:gore} with and without BnB.

First we analyse this effect for synthetic data. We use instances of point cloud registration with outliers ratios in the range $[0.5 , 0.98]$ obtained as in~Sec.~\ref{sec:rotresults} and then randomly translated.  Fig.~\ref{fig:bnb_effect_synth_pruning} plots the pruning rates and Fig.~\ref{fig:bnb_effect_synth_time} the runtimes taken as the median values for 100 random instances. For high outlier ratios (which are more relevant in 6 DoF registration), the addition of BnB has an extra cost of one order of magnitude. BnB slightly improved the pruning of GORE, however, as further results will show, the improvement is more significant on real data.

For real data, we used the instances in Table~\ref{tab:comparison6dof}. Results reported in Table~\ref{tab:bnb_effect_real} are median values over all same size instances (Column 2) for every dataset (Column 1).  We recorded the number of removed outliers and runtime for GORE with (Columns 3 and 6) and without BnB (Columns 4 and 7). Table~\ref{tab:bnb_effect_real} also lists the increment ratio of removed outliers of using BnB (Column 5). Whilst the increment was low for the Mian and the Stanford datasets, it was in general more than a $10\%$ for the remote sensing and mining datasets. For example, in remote sensing with $N = 500$,  the usage of BnB allowed reduction to almost 7 times more true outliers. Most likely this effect occurs due to the more challenging characteristics of real data: more clutter, low overlaps, and structured error. 


Although the runtime increases by one order of magnitude by using BnB in Algorithm~\ref{alg:gore}, this increase is insignificant relative to the savings in \emph{overall runtime} (GORE+BNB), due to the much more aggressive pruning by GORE.


\begin{figure}
	\centering
	\subfloat[]{\includegraphics[width=.49\linewidth]{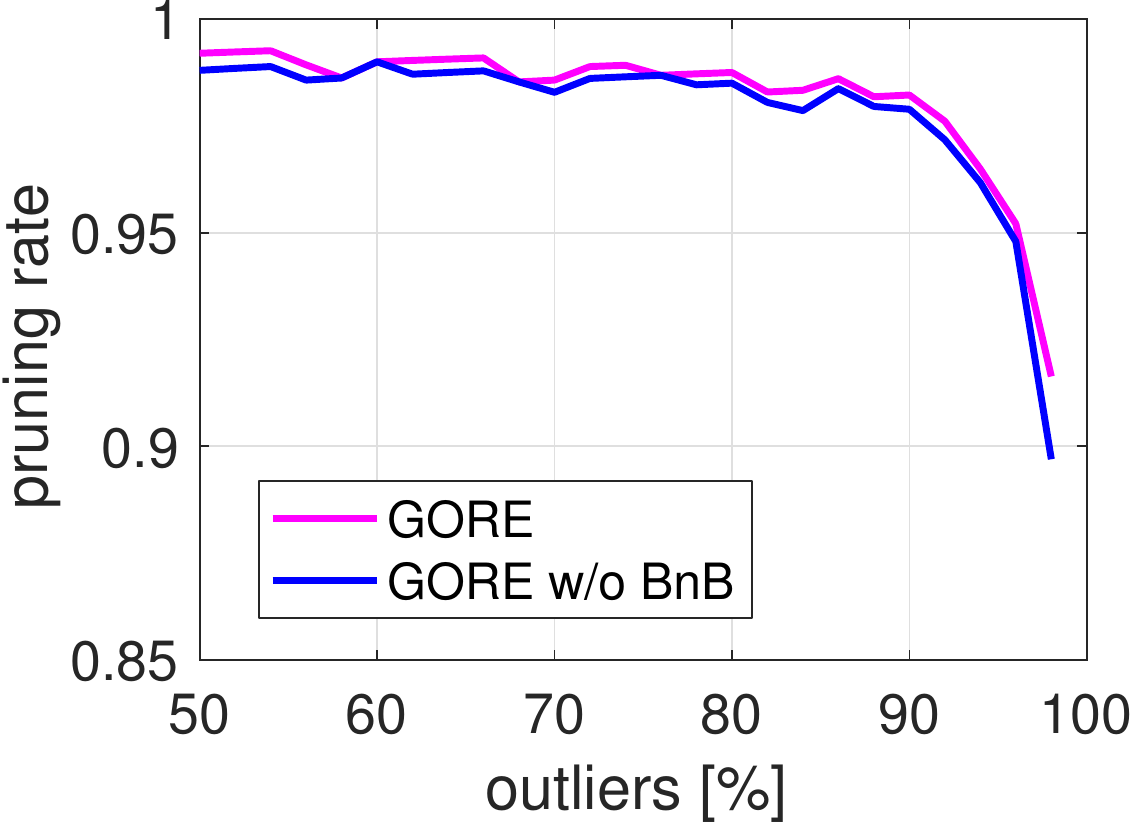} \label{fig:bnb_effect_synth_pruning}} \hfill
	\subfloat[]{\includegraphics[width=.49\linewidth]{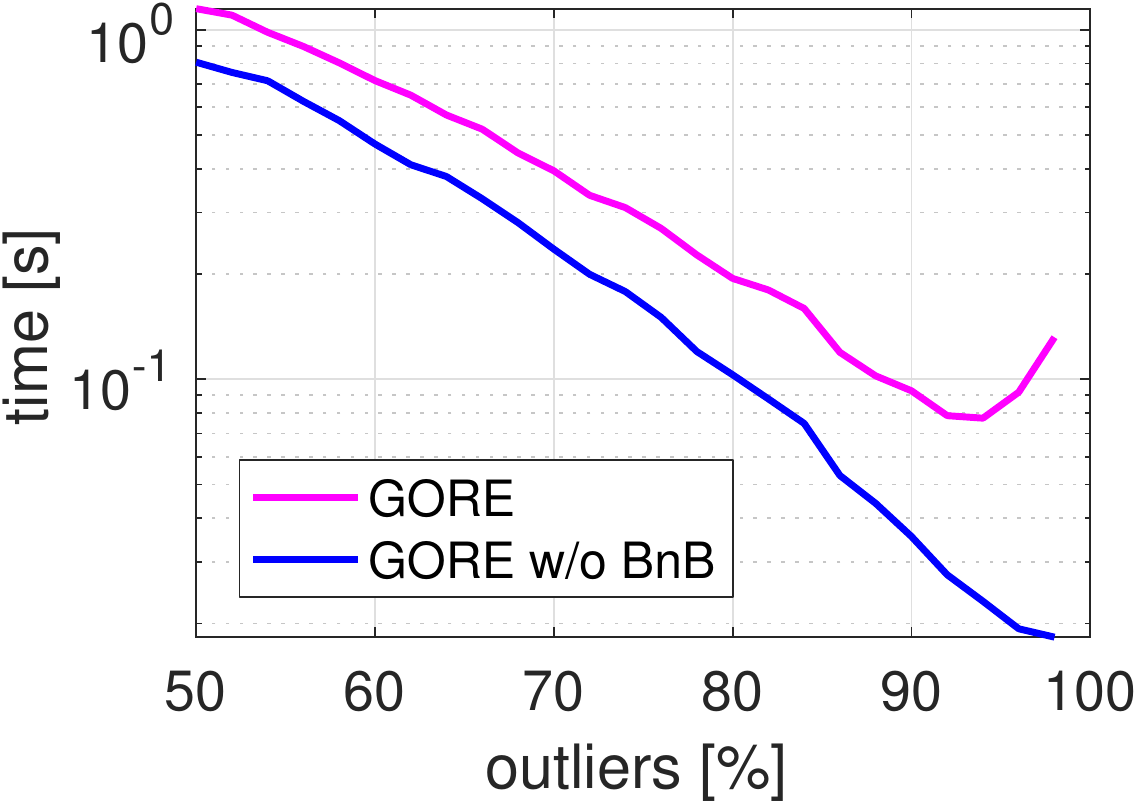}\label{fig:bnb_effect_synth_time}}
	\caption{Effect of BnB on GORE for synthetic data. (a) Pruning rate of GORE with and without BnB is compared for data with $50\%$ to $98\%$ outlier ratios. (b) Time comparison for GORE with and without BnB.} 
	\label{fig:bnb_effect_synth}
\end{figure}

\begin{table} 
	\renewcommand{\arraystretch}{1}
	{
		\begin{center}
			\begin{tabular}{|l|c|rrr|rr|}
				\hline
				\multirow{3}{*}{Dataset} & 
				\multirow{3}{*}{$|\cH|$} &
				\multicolumn{3}{c|}{Removed Outliers} & 
				\multicolumn{2}{c|}{Time (s)} \\ & & 
				\multicolumn{1}{c}{with} & 
				\multicolumn{1}{c}{w/o} & 
				\multicolumn{1}{c|}{inc.} & 
				\multicolumn{1}{c}{with}  & 
				\multicolumn{1}{c|}{w/o} \\
				& & 
				\multicolumn{1}{c}{BnB} & 
				\multicolumn{1}{c}{BnB} & 
				\multicolumn{1}{c|}{ratio} & 
				\multicolumn{1}{c}{BnB}  & 
				\multicolumn{1}{c|}{BnB} \\
				\hline
				
				\multirow{3}{*}{\emph{Mian}}
				&500 &    435 &    422 &  0.025 &  0.424 &  0.110 \\
				&1000 &    929 &    911 &  0.021 &  0.882 &  0.301 \\
				&2000 &   1917 &   1887 &  0.017 &  2.023 &  0.938 \\
				\hline
				\multirow{3}{*}{\emph{Stanford}}
				&500&    424 &    397 &  0.057 &  0.908 &  0.201 \\
				&1000&    919 &    866 &  0.036 &  1.498 &  0.420 \\
				&2000&   1906 &   1831 &  0.035 &  3.371 &  1.368 \\
				\hline
				\multirow{3}{*}{
					\parbox[c][4em][c]{.8cm} {\emph{Remote\\ Sensing}}
				}
				&500 &    353 &     43 &  6.808 &  2.422 &  0.155 \\
				&1000 &    974 &    884 &  0.105 &  0.833 &  0.272 \\
				&2000 &   1957 &   1717 &  0.153 &  5.512 &  2.180 \\
				\hline
				\multirow{3}{*}{\emph{Mining}}
				&500 &    376 &    311 &  0.172 &  0.907 &  0.069 \\
				&1000 &    770 &    744 &  0.063 &  4.386 &  0.305 \\
				&2000 &   1810 &   1698 &  0.105 & 11.034 &  1.284 \\
				\hline
			\end{tabular}
		\end{center}
	}
	\caption{Effect of BnB on GORE for real data. }
	\label{tab:bnb_effect_real}
\end{table}

\subsection{Comparison against an approximate method} \label{sec:albarelli}
Here we compare against the method of Albarelli et al.~\cite{albarelli10,albarelli09,rodala12,rodola13}. We use the following two variants:
\begin{itemize}
	\item Albarelli-Orig: With the proposed payoff matrix of~\cite{albarelli10}
	\begin{align}
	P_{ij} = \dfrac{\min(\|\bx_i - \bx_j\|, \|\by_i - \by_j\|)}{\max(\|\bx_i - \bx_j\|, \|\by_i - \by_j\|)}.
	\end{align}
	\item Albarelli: The payoff matrix is defined as
	\begin{align}
	P_{ij} = \begin{cases}
	1 & \text{if } |\,\| \bx_i - \bx_j\| - \| \by_i -  \by_j\|\, |\leq 2\xi \\
	0 & \text{otherwise.}
	\end{cases}
	\end{align}
This second variant is defined to make~\cite{albarelli10} ``more compatible" to our objective~\eqref{eq:rigid}, in that if a pair of correspondences $(\bx_i,\by_i)$ and $(\bx_j,\by_j)$ are in a consensus set $\cI$ (see~\eqref{eq:rigid}), then the condition
\begin{align}
|\,\| \bx_i - \bx_j\| - \| \by_i -  \by_j\|\, |\leq 2\xi.
\end{align}
must be satisfied.
\end{itemize}
For a fair comparison, we did not impose one-to-one correspondences for the solution of the Albarelli's variants since such a constraint is not considered in the objective function~\eqref{eq:rigid}. We then take the output (surviving correspondences) of Albarelli's variants as the ``reduced set" $\cH^\prime \subseteq \cH$.

To compare GORE against Albarelli's variants, we computed precision (pre) and recall (rec) for $\cH^\prime$ as follows
\begin{align}
\text{pre} &:= |\cH^\prime \cap \cI^* | / |\tilde{\cH}| \\ 
\text{rec} &:= |\cH^\prime \cap \cI^* | / |\cI^*|.           
\end{align}
We also recorded the total runtime (time). Since Albarelli et al.'s method is stochastic, all its reported metrics in Table~\ref{tab:albarelli} were the median value over 100 repetitions.

Both variants converged extremely fast: $< 0.06$ seconds for the larger instances to few milliseconds for small instances ($N = 500$). In general we observed that Albarelli obtained a better precision than Albarelli-Orig whilst the recall was higher for the former. For both variants, precision was higher than GORE in the Stanford and Mian datasets where the method reported very acceptable results. However, in several instances of the remote sensing and mining datasets precision and recall were $0$. The performance of Albarelli's variants is worse for instances of those datasets, presumably because the low overlapping between point clouds at the optimal solution. Contrast to GORE, where recall is always one as it is guaranteed that $\cH^{\prime}$ contains all true inliers.

\begin{table*} 
	\begin{center}
		\begin{tabular}{|l|rr|rrrr|rrrr|rrrr|}
			\hline
			\multicolumn{1}{|c|}{\multirow{2}{*}{Object}} & 
			\multicolumn{1}{c}  {\multirow{2}{*}{$N$}} & 
			\multicolumn{1}{c|} {\multirow{2}{*}{$\eta$}} &
			\multicolumn{4}{c|} {\multirow{2}{*}{GORE}} &
			\multicolumn{4}{c|} {\multirow{2}{*}{Albarelli}} &
			\multicolumn{4}{c|} {\multirow{2}{*}{Albarelli-Orig}} 
			\\[.5em]
			& & 
			& \multicolumn{1}{c}{$|\cH^\prime|$} 
			& \multicolumn{1}{c}{pre}
			& \multicolumn{1}{c}{rec}  
			& \multicolumn{1}{c|}{time (s)}     
			& \multicolumn{1}{c}{$|\cI|$} 
			& \multicolumn{1}{c}{pre} 
			& \multicolumn{1}{c}{rec} 
			& \multicolumn{1}{c|}{time (s)}
			& \multicolumn{1}{c}{$|\cI|$} 
			& \multicolumn{1}{c}{pre} 
			& \multicolumn{1}{c}{rec} 
			& \multicolumn{1}{c|}{time (s)}			
			\\
			\hline
\multirow{3}{*}{{\parbox[c][4em][c]{1.5cm} {\emph{bunny}
		}}}
& 500    & 0.92   & 56     & 0.71   & 1.00   & 0.860  & 47     & \textbf{0.83  } & 1.00   & 0.003  & 57     & 0.70   & 1.00   & 0.004  \\ 
& 1000   & 0.96   & 59     & 0.73   & 1.00   & 1.362  & 50     & \textbf{0.84  } & 0.98   & 0.012  & 58     & 0.74   & 1.00   & 0.013  \\ 
& 2000   & 0.98   & 60     & 0.72   & 1.00   & 4.618  & 50     & \textbf{0.84  } & 1.00   & 0.045  & 60     & 0.72   & 1.00   & 0.048  \\ 
\hline 
\multirow{3}{*}{{\parbox[c][4em][c]{1.5cm} {\emph{dragon}
		}}}
& 500    & 0.94   & 50     & 0.64   & 1.00   & 1.807  & 37     & \textbf{0.83  } & 0.94   & 0.003  & 52     & 0.62   & 1.00   & 0.004  \\ 
& 1000   & 0.97   & 53     & 0.64   & 1.00   & 2.813  & 39     & \textbf{0.82  } & 0.94   & 0.011  & 55     & 0.62   & 1.00   & 0.013  \\ 
& 2000   & 0.98   & 57     & 0.60   & 1.00   & 4.770  & 41     & \textbf{0.79  } & 0.97   & 0.046  & 60     & 0.57   & 1.00   & 0.051  \\ 
\hline 

\multirow{3}{*}{{\parbox[c][4em][c]{1.5cm} {\emph{t-rex}
		}}}
& 500    & 0.94   & 40     & 0.70   & 1.00   & 0.788  & 35     & \textbf{0.80  } & 1.00   & 0.003  & 43     & 0.65   & 1.00   & 0.003  \\ 
& 1000   & 0.97   & 47     & 0.66   & 1.00   & 1.721  & 39     & \textbf{0.79  } & 1.00   & 0.012  & 53     & 0.58   & 1.00   & 0.013  \\ 
& 2000   & 0.98   & 49     & 0.65   & 1.00   & 3.493  & 40     & \textbf{0.80  } & 1.00   & 0.045  & 57     & 0.56   & 1.00   & 0.049  \\ 
\hline 
\multirow{3}{*}{{\parbox[c][4em][c]{1.5cm} {\emph{parasauro}
		}}}
& 500    & 0.96   & 46     & 0.46   & 1.00   & 4.314  & 28     & \textbf{0.69  } & 0.95   & 0.003  & 49     & 0.43   & 1.00   & 0.004  \\ 
& 1000   & 0.98   & 57     & 0.39   & 1.00   & 8.004  & 30     & \textbf{0.69  } & 0.91   & 0.011  & 56     & 0.39   & 1.00   & 0.014  \\ 
& 2000   & 0.99   & 58     & 0.40   & 1.00   & 16.216 & 30     & \textbf{0.69  } & 0.91   & 0.044  & 58     & 0.40   & 1.00   & 0.050  \\ 
\hline 

\multirow{3}{*}{{\parbox[c][4em][c]{1.5cm} {\emph{mining-a}
		}}}
& 500    & 0.98   & 16     & \textbf{0.75  } & 1.00   & 2.669  & 15     & \textbf{0.75  } & 1.00   & 0.003  & 13     & 0.00   & 0.00   & 0.003  \\ 
& 1000   & 0.99   & 18     & \textbf{0.67  } & 1.00   & 9.885  & 6      & 0.33   & 0.17   & 0.011  & 16     & 0.00   & 0.00   & 0.013  \\ 
& 2000   & 0.99   & 21     & \textbf{0.67  } & 1.00   & 27.076 & 6      & 0.00   & 0.00   & 0.044  & 20     & 0.00   & 0.00   & 0.052  \\ 
\hline 
\multirow{3}{*}{{\parbox[c][4em][c]{1.5cm} {\emph{mining-b}
		}}}
& 500    & 0.98   & 20     & 0.50   & 1.00   & 3.546  & 10     & \textbf{0.90  } & 0.90   & 0.003  & 14     & 0.07   & 0.10   & 0.003  \\ 
& 1000   & 0.99   & 33     & \textbf{0.30  } & 1.00   & 20.328 & 6      & 0.00   & 0.00   & 0.011  & 16     & 0.00   & 0.00   & 0.012  \\ 
& 2000   & 0.99   & 38     & \textbf{0.29  } & 1.00   & 98.255 & 7      & 0.00   & 0.00   & 0.044  & 19     & 0.00   & 0.00   & 0.047  \\ 
\hline 

\multirow{3}{*}{{\parbox[c][4em][c]{1.5cm} {\emph{vaihingen-a}
		}}}
& 500    & 0.99   & 75     & \textbf{0.08  } & 1.00   & 10.756 & 7      & 0.00   & 0.00   & 0.003  & 13     & 0.00   & 0.00   & 0.003  \\ 
& 1000   & 0.99   & 14     & \textbf{0.93  } & 1.00   & 1.402  & 7      & 0.00   & 0.00   & 0.011  & 17     & 0.00   & 0.00   & 0.013  \\ 
& 2000   & 0.99   & 20     & \textbf{1.00  } & 1.00   & 10.631 & 8      & 0.00   & 0.00   & 0.046  & 21     & 0.24   & 0.25   & 0.053  \\ 
\hline 
\multirow{3}{*}{{\parbox[c][4em][c]{1.5cm} {\emph{vaihingen-b}
		}}}
& 500    & 0.98   & 20     & \textbf{0.40  } & 1.00   & 3.375  & 7      & 0.17   & 0.12   & 0.003  & 11     & 0.00   & 0.00   & 0.003  \\ 
& 1000   & 0.99   & 22     & \textbf{0.55  } & 1.00   & 2.318  & 7      & 0.00   & 0.00   & 0.011  & 18     & 0.33   & 0.50   & 0.013  \\ 
& 2000   & 0.99   & 22     & \textbf{0.64  } & 1.00   & 20.325 & 7      & 0.00   & 0.00   & 0.046  & 21     & 0.00   & 0.00   & 0.049  \\ 
\hline			
			
		\end{tabular}
	\end{center}
	\caption{Comparison against to Albarelli's variants.}
	\label{tab:albarelli}
	\vspace{-1.2em}
\end{table*}

\subsection{Failure case} \label{se:failure}
Fig.~\ref{fig:failure} depicts a failure case of GORE for 6 DoF point cloud registration for two consecutive views of \emph{bunny}. Correspondences (Fig.~\ref{fig:failin}) with low outlier ratio ($\eta=0.23$) were obtained by selecting ISS3D keypoints with \emph{mutually} lowest distance between their PFH descriptors. Since GORE was only able to remove two outliers, it was ineffective in reducing the overall runtime when finding the optimal solution; BnB took lesser time without the preprocessing of GORE. This result suggest that GORE could be ineffective as a preprocessor for low outlier ratio problems.

\begin{figure}
	\vspace{-1.4cm}
	\subfloat[correspondences]{
		\begin{minipage}{.54\linewidth} 
			\vspace{1cm}
			\includegraphics[width=\linewidth]{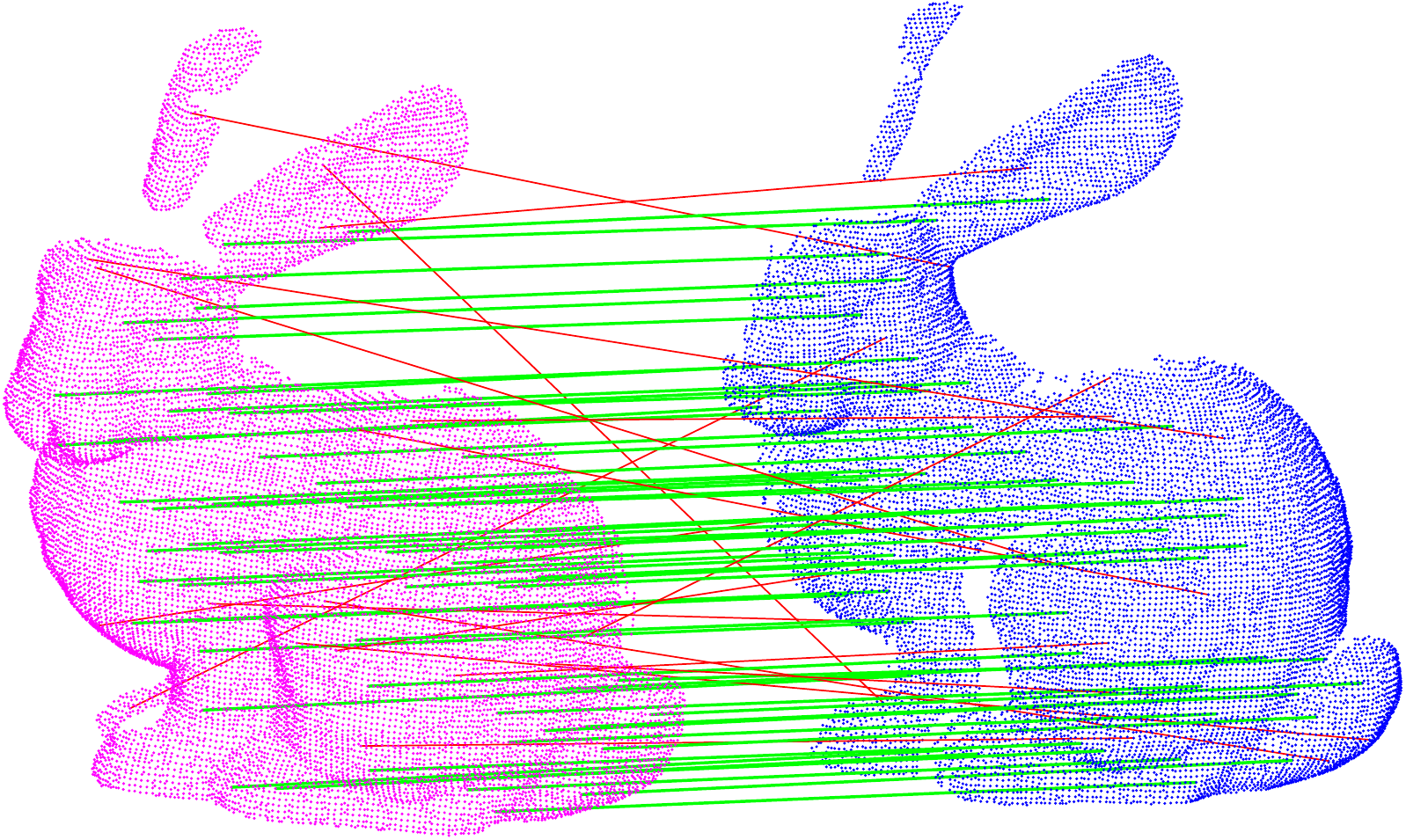}
		\end{minipage}
		\label{fig:failin}
	}
\subfloat[results]{
	\begin{minipage}{.43\linewidth}
\vspace{1.7cm}
	\resizebox*{\linewidth}{!}{
			\begin{tabular}{|c|c|r|}
				\hline
				\multirow{2}{*}{GORE} & $|\cH^\prime|$ & 58\\
				& time (s) & 0.016\\
				\hline
					\begin{minipage}{.9cm}
						\vspace{1mm}
					\centering GORE\\+\\BnB
					\vspace{1mm}
					\end{minipage}
			&
			\begin{minipage}{1.2cm}
				\centering
				$|\cI^*|$ \\ time (s)
			\end{minipage}
		    & 
		    \begin{minipage}{.8cm}
		    	\flushright
		    	56 \\ 1.05
		    \end{minipage}\\
			\hline
			BnB & time (s) & 0.99 \\
			\hline
			\end{tabular}
		}
		\end{minipage}
	\label{fig:failres}
}
\caption{A failure case for \emph{bunny} with $N = 73 $ and $\eta = 0.23$.}
\label{fig:failure}
\end{figure}

\section{Conclusions}

We have presented a guaranteed outlier removal technique for robust Euclidean point cloud registration with correspondences. Any datum removed by our method is guaranteed to not exist in the globally optimal solution. Based on simple geometric operations, our algorithm is deterministic and efficient. Experiments show that, by significantly reducing the amount of data and outliers, our method greatly speeds up the solution (local and global) of the registration.

\section*{Acknowledgements}
This work was supported by the Australian Research Council grant DP160103490.

\bibliographystyle{IEEEtran}

\bibliography{gore.bib}

\vspace{-1.5em}
\begin{IEEEbiography} [{\includegraphics[width=1in,clip,keepaspectratio]{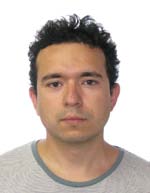}}]
	{\'{A}lvaro Parra} obtained a BSc~Eng. (2006), a B.Eng. Computer Science (2008) and a MSc. Computer Science (2011) from \emph{Universidad de Chile} (Santiago, Chile). He received the PhD degree in computer and mathematical sciences from The University of Adelaide in 2016 (Adelaide, Australia). He is currently a Research Associate at the School of Computer Science in The University of Adelaide. His main research interests include point cloud registration, 3D computer vision and optimization methods.
\end{IEEEbiography}
\vspace{-1.6em}
\begin{IEEEbiography}[{\includegraphics[width=1in,clip,keepaspectratio]{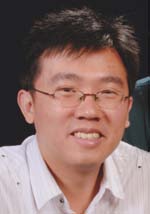}}]
	{Tat-Jun Chin}
	received the BEng degree in	mechatronics engineering from Universiti Teknologi Malaysia (UTM) in 2003 and the PhD degree in computer systems engineering from Monash University, Victoria, Australia, in 2007. He was a research fellow at the Institute for Infocomm Research (I2R) in Singapore from 2007 to 2008. Since 2008, he has been a Senior Research Associate, Lecturer, then Senior Lecturer at The University of Adelaide, South Australia. His research interests include robust estimation and geometric optimization. He is a member of the IEEE.
\end{IEEEbiography}

\end{document}


\title{Guaranteed Outlier Removal for Point Cloud Registration with Correspondences}


\makeatletter 
\renewcommand{\thefigure}{A\@arabic\c@figure}
\makeatother

\markboth{Supplementary Material for: Guaranteed Outlier Removal for Point Cloud Registration with Correspondences}%
{Shell \MakeLowercase{\textit{et al.}}: Bare Demo of IEEEtran.cls for Computer Society Journals}

\renewcommand{\thesection}{A} 
\section{Mathematical derivations}

\subsection{Derivation of (18) in the main text}


Without lose of generality define 
\begin{align}
\bB = \bC \hat{\bB},
\end{align}
where $\hat{\bB}$ is the rotation that minimizes the motion between $\bx_k$ and $\by_k$, and $\bC$ the rotation that minimizes the motion between $\by_k$ and $\bu$ in $S_{\epsilon_k}(\by_k)$.

The distance between two rotations $d_{\angle} (\bP, \bQ)$ is defined as the rotation angle of $\bP \bQ^T$, and the distance of a rotation matrix $d_{\angle} (\bP)$ is defined as $d_{\angle} (\bP, \bI)$. By using that definition 
\begin{align}\label{eq:C}
d_{\angle}(\bC) = \epsilon_k
\end{align}
since $\bC$ is a rotation matrix with rotation axis orthogonal to $\by_k$ and rotation angle $\epsilon_k$.

Equation (18) in the main text is equivalently expressed as
\begin{align}\label{eq:18equiv}
\angle(\bB\bx_i, \hat{\bB}\bx_i) \leq \epsilon_k.
\end{align} 

From Lemma~1 in~\cite{hartley09},  \eqref{eq:18equiv} is true if $d_{\angle}(\bB, \hat{\bB}) = \epsilon_k$.

From the definition of $\bB$
\begin{align}\label{eq:BB}
d_{\angle}(\bB, \hat{\bB}) = d_{\angle}(\bC \hat{\bB}, \hat{\bB}) = d_{\angle}(\bC \hat{\bB}\hat{\bB}^T, \bI) =   d_{\angle}(\bC).
\end{align}
Finally, from~\eqref{eq:C} and~\eqref{eq:BB}, 
\begin{align}
d_{\angle}(\bB, \hat{\bB}) =  \epsilon_k.
\end{align}

%
%

\subsection{Mathematical representations of vectors in (22) in the main text}
\begin{align}
\bp_n &= \exp\left(\epsilon_k  \left[ 
\hat{\bB}\bx_i\times\by_k  /\|\hat{\bB}\bx_i\times\by_k\| \right]_{\times} 
\right)\hat{\bB}\bx_i; \\
\ba_n &= \exp\left(\epsilon_k  \left[
\by_k \times \hat{\bB}\bx_i /\|\by_k \times \hat{\bB}\bx_i\| \right]_{\times} 
\right)\by_k;\\
\bp_f &= \exp\left(\epsilon_k\left[
\by_k \times \hat{\bB}\bx_i /\|\by_k \times \hat{\bB}\bx_i\| \right]_{\times} 
\right)\hat{\bB}\bx_i; \;  \text{and}\\
\ba_f &= \exp\left(\epsilon_k \left[
\hat{\bB}\bx_i\times\by_k/\|\hat{\bB}\bx_i\times\by_k\| \right]_{\times} 
\right)\by_k.
\end{align}

\subsection{Derivation of (23) in the main text}

Consider rotating an arbitrary unit-norm point $\bp$ with $\bA_{\theta,\bu}$ for a fixed angle $\theta$ and an axis $\bu \in S_{\epsilon_k}(\by_k)$. We wish to bound the possible locations of $\bA_{\theta,\bu}\bp$ given the uncertainty in $\bu$. To this end, we establish
\begin{align}
\nonumber \max_{\bu \in S_{\epsilon_k}(\by_k)} \angle(\bA_{\theta,\bu}\bp, \bA_{\theta,\by_k}\bp) &\le \max_{\bu \in S_{\epsilon_k}(\by_k)} \|  \theta \bu   - \theta \by_k \|_2 \\
&= 2|\theta| \sin(\epsilon_k/2),\label{eq:bound1}
\end{align}
where the first line is due to a well-known result of the axis-angle representation (see~\cite[Lemma 2]{hartley09}), and the second line occurs since $S_{\epsilon_k}(\by_k)$ has an angular radius of $\epsilon_k$.

Now we extend~\eqref{eq:bound1} to accommodate the uncertainty of $\bp$ itself as a point from $S_{\epsilon_k}(\hat{\bB}\bx_i)$. We thus establish
\begin{align}
&\hspace{4em}\nonumber \max_{       \substack{  \bp \in S_{\epsilon_k}(\hat{\bB}\bx_i)     \\ \bu \in S_{\epsilon_k}(\by_k)         } } \angle(\bA_{\theta,\bu}\bp, \bA_{\theta,\by_k}\hat{\bB}\bx_i) \\
&\le \nonumber \max_{\substack{\bp \in S_{\epsilon_k}(\hat{\bB}\bx_i)     \\ \bu \in S_{\epsilon_k}(\by_k)         } } \angle(\bA_{\theta,\bu}\bp, \bA_{\theta,\by_k}\bp) + \angle(\bA_{\theta,\by_k}\bp, \bA_{\theta,\by_k}\hat{\bB}\bx_i)\\
&\le 2|\theta|\sin(\epsilon_k/2) + \epsilon_k.\label{eq:bound2}
\end{align}
The second line is due to the triangle inequality, while the third line applies~\eqref{eq:bound1} on the first term of the second line.

\subsection{Approximate solution for $\Theta_i$}

Here, concepts from the spherical coordinate system are used with reference to $\by_k$ as the North Pole.


\subsubsection{Degenerate cases}
If $\hat{\bB}\bx_i$ is close to $\by_k$, the North Pole may lie in $L_k(\bx_i)$. If this occurs, we take $\Theta_i = [-\pi,\pi]$.

\subsubsection{Non-degenerate cases}
Define $\phi(\by_i)$ and $\psi(\by_i)$ respectively as the azimuth and inclination of $\by_i$. The spherical region $S_{\epsilon_i}(\by_i)$ is contained between the meridians $\phi(\by_i)-\gamma_i$ and $\phi(\by_i)+\gamma_i$, where
\begin{align}\label{eq:gamma}
\gamma_i=\arcsin\left(\frac{\sin(\epsilon_i)}{\sin(\psi(\by_i))}\right)
\end{align}
following the geometric considerations in Fig.~\ref{fig:lambda}. Let $\theta_i \in [-\pi,\pi]$ be the rotation angle such that the point $\bA_{\theta_i,\by_k}\hat{\bB}\bx_i$ is on the meridian $\phi(\by_i)$. Refer to Fig.~\ref{fig:bound_d}.

\begin{figure}\centering
	\includegraphics[width=.75\linewidth]{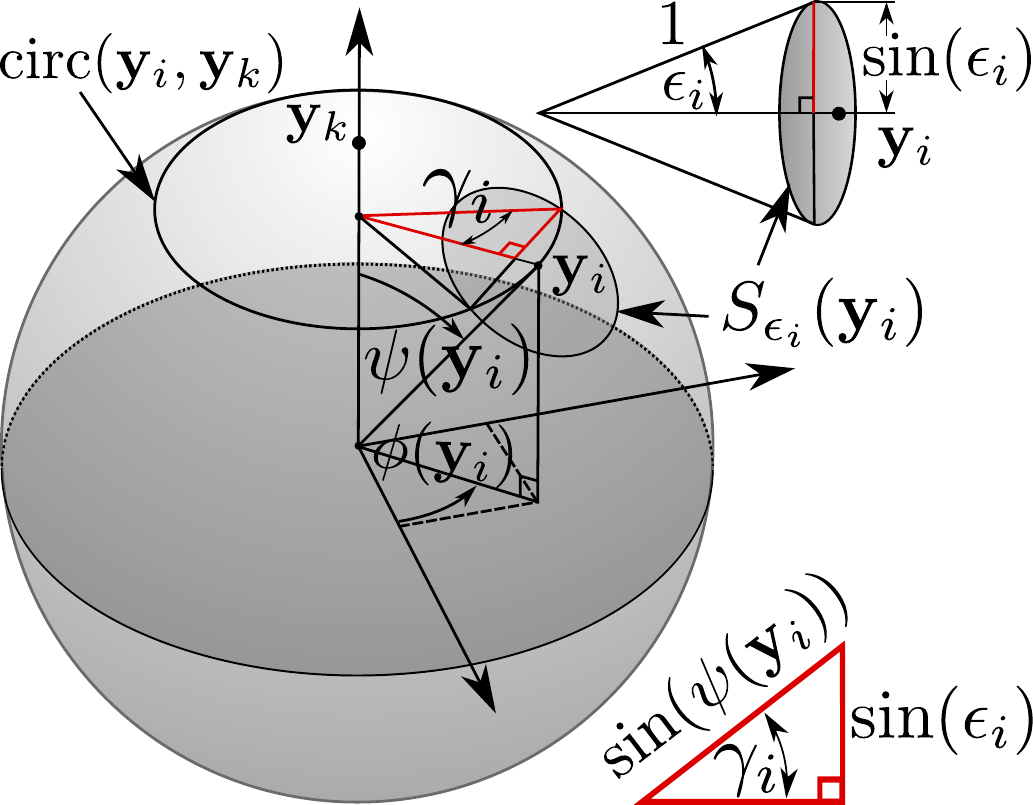}
	\caption{Solving for $\gamma_i$ in the red triangle. Its cathetus is half of the longest segment connecting points in $S_{\epsilon_i}(\by_i)$ and its hypotenuse is the radius of $\text{circ}(\by_i,\by_k)$.} 
	\label{fig:lambda}
\end{figure}
 
\begin{figure}\centering
	\includegraphics[width=.9\linewidth]{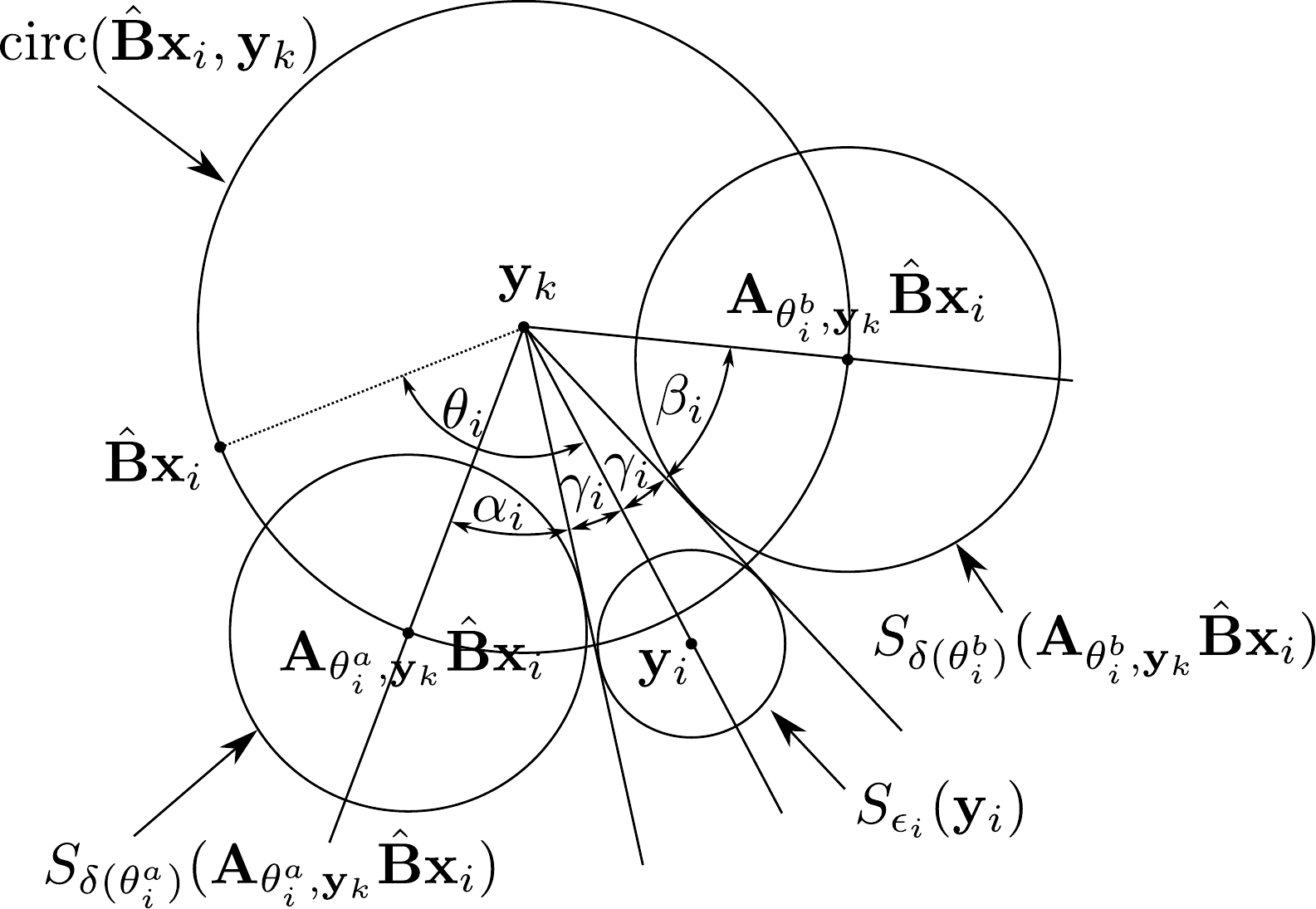}
	\caption{To simplify the diagram and to aid intuition, the sphere in Fig.~3d in the main text is stereographically projected to the 2D plane using the North Pole ($\by_k$) as the projection pole. Recall that the stereographic projection preserves circles~\cite{needham97}, thus the shapes of all the circles and spherical regions on the sphere are preserved. Note that stereographic projection is only for presentation and is not required in practice.}
	\label{fig:bound_d}
\end{figure}

\vspace{1em}
\noindent \textbf{Case 1: $\theta_i \in [0,\pi]$}

This case is shown in Fig.~\ref{fig:bound_d}. Define $\Theta_i = [\theta^{a}_i,\theta^{b}_i]$. The desired bounding interval $\Theta_i$ can 
be obtained by taking
\begin{align}
\theta^{a}_i = \theta_i - \gamma_i - \alpha_i \;\;\; \text{and} \;\;\; \theta^{b}_i = \theta_i + \gamma_i + \beta_i,
\end{align} 
where $\alpha_i$ is the largest value such that the spherical region
\begin{align} \label{eq:Sl}
S_{\delta(\theta^{a}_i)}(\bA_{\theta^{a}_i,\by_k}\hat{\bB}\bx_i)
\end{align} 
still touches the meridian $\phi(\by_i) - \gamma_i$, and $\beta_i$ is the largest value such that the spherical region
\begin{align} \label{eq:Sr}
S_{\delta(\theta^{b}_i)}(\bA_{\theta^{b}_i,\by_k}\hat{\bB}\bx_i)
\end{align}
still touches the meridian $\phi(\by_i) + \gamma_i$. Refer to Fig.~\ref{fig:bound_d}. To determine $\Theta_i$,  we must find $\alpha_i$ and $\beta_i$. From definition~(24) in the main text,
\begin{align}
&\delta(\theta^{a}_i)= 2|\theta_i-\gamma_i-\alpha_i|\sin(\epsilon_k/2)+\epsilon_k \;\;\; \text{and} \\
&\delta(\theta^{b}_i)= 2|\theta_i+\gamma_i+\beta_i|\sin(\epsilon_k/2)+\epsilon_k.
\end{align}
Applying the same geometric considerations in Fig.~\ref{fig:lambda} on the spherical regions~\eqref{eq:Sl} and~\eqref{eq:Sr}, we have
\begin{align}\label{eq:inter}
\sin(\alpha_i)=\dfrac{\sin(\delta(\theta^{a}_i))}{\sin(\psi(\bx_i))} \;\;\;\text{and}
\;\;\;
\sin(\beta_i )=\dfrac{\sin(\delta(\theta^{b}_i))}{\sin(\psi(\bx_i))}.
\end{align}
Note that the functions on both sides of each equation have the unknowns $\alpha_i$ and $\beta_i$ respectively.

Fig.~\ref{fig:beta} plots the two sine functions $\sin(\beta_i)$ and $\sin(\delta(\theta^{b}_i))/\sin(\psi(\bx_i))$. We consider only $\beta_i \in [0,\pi/2]$, since the condition where the two functions do not intersect before $\beta_i \le \pi/2$ corresponds to the degeneracies. The following theorem defines the domain for $\alpha_i$ and $\beta_i$.

\begin{theorem}\label{thm:domain}$\alpha_i$ and $\beta_i$ are well defined in $[0,\pi/2]$.
\end{theorem}
\begin{proof}
	Let the outline of the  spherical region $S_{\delta(\theta_i^b)}(\bA_{\theta_i^b, \by_k} \hat{\bB}\bx_i)$ (see Fig.~\ref{fig:bound_d}) be infinitesimally close to either the North or the South Pole. In such a limiting case, the meridians containing $S_{\delta(\theta_i^b)}(\bA_{\theta_i^b, \by_k} \hat{\bB}\bx_i)$ will be at azimuthal angle $\pi$ apart, thus $\beta_i\leq\pi/2$. If either the North or the South Pole is contained in  $S_{\delta(\theta_i^b)}(\bA_{\theta_i^b, \by_k} \hat{\bB}\bx_i)$, then this spherical region cannot be bounded between two meridians since the spherical region can intersect points with azimuth in all the azimuthal domain $[0,2\pi]$. The same analysis can be done for $\alpha_i$. 
\end{proof}

Further, since usually $\epsilon_k \ll \pi$, the period of the second sine function
\begin{align}
\frac{2\pi}{2\sin(\epsilon_k/2)} \gg 2\pi
\end{align}
is much greater than $2\pi$, thus explaining the almost linear trend of the second sine function for $\beta_i \in [0,\pi/2]$. A largely identical plot occurs for the functions involving $\alpha_i$.

Analytically solving the equations in~\eqref{eq:inter} is non-trivial. However, since all that we require is a bounding interval $\Theta_i$, we can replace the sine functions with more amenable approximations that yield a valid bounding interval. An identical technique is used to solve for $\alpha_i$ and $\beta_i$ respectively, thus we describe our solution only for $\beta_i$.

\begin{figure}[t]\centering
	\begin{overpic}[width=.85\linewidth,tics=10,trim=22 4 30 20]{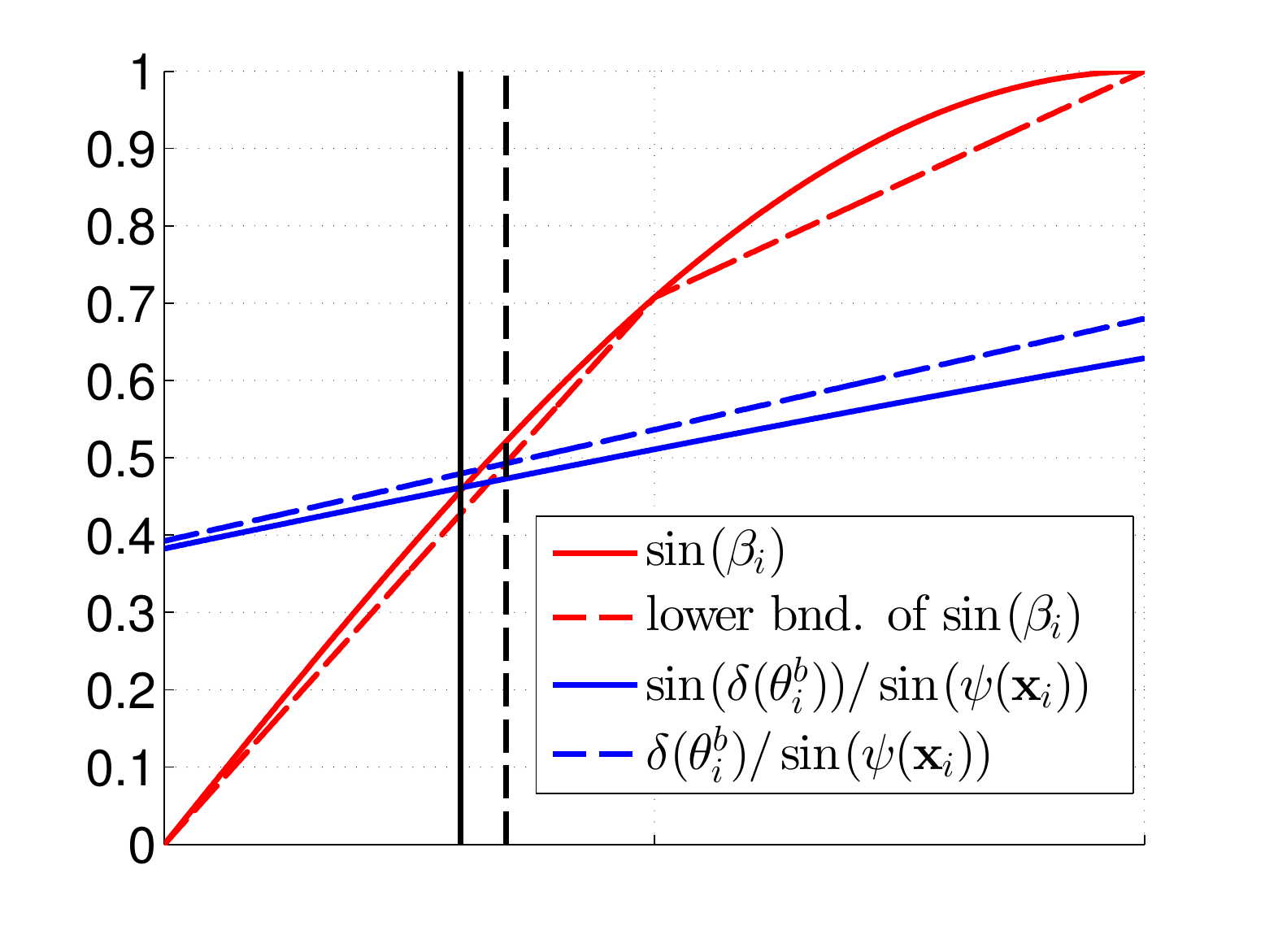}
		\put (29,2) {$\beta_i^*$}
		\put (36,2) {$\beta'_i$}
		\put (8,2) {$0$}
		\put (48,2) {$\pi/4$}
		\put (88,2) {$\pi/2$}
	\end{overpic}
	\caption{  Solving for $\beta_i$ in~\eqref{eq:inter} using the proposed linear approximations. Note that the obtained solution $\beta^\prime_i$ is always greater than the exact solution $\beta^*_i$, thus guaranteeing that the upper limit $\theta^{b}_i$ of $\Theta_i$ is a valid bound.}
	\label{fig:beta}
\end{figure}

We replace $\sin(\beta_i)$ with a lower-bounding two-piece linear function; see Fig.~\ref{fig:beta}. To obtain an upper-bounding line to $\sin(\delta(\theta^{b}_i))/\sin(\psi(\bx_i))$, we use Jordan's inequality
\begin{align}
\sin(t) \le t \;\; \text{for} \;\; t \le \pi/2,
\end{align}
which enables us to replace the second sine function with
\begin{align}
\dfrac{\delta(\theta^{b}_i)}{\sin(\psi(\bx_i))} = \frac{2|\theta_i+\gamma_i+\beta_i|\sin(\epsilon_k/2)+\epsilon_k}{\sin(\psi(\bx_i))}.
\end{align}
This upper-bounding line is legitimate for
\begin{align}
2|\theta_i+\gamma_i+\beta_i|\sin\left(\epsilon_k/2\right)+\epsilon_k \le \pi/2,
\end{align}
where in the worst case requires
\begin{align}
2\pi\sin\left(\epsilon_k/2\right)+\epsilon_k \le \pi/2
\end{align}
or $\epsilon_k \leq \pi / (2\pi+2) \equiv 21.7\degree$, which is more than adequate in most cases. In case that $\epsilon_k > 21.7\degree $ or approximate functions to solve~\eqref{eq:inter} do not intersect in $[0,\pi/2]$, we set $\Theta_i = [-\pi,\pi]$. Solving for $\beta_i$ in the manner above allows us to compute the upper limit $\theta^{b}_i$ in constant time.

Note that the resulting upper limit $\theta^{b}_i$ may extend beyond $\pi$; to ``wrap around" the interval, we break $\Theta_i = [\theta^{a}_i,\theta^{b}_i]$ into two connected intervals $[\theta^{a}_i,\pi]$ and $[\theta^{b}_i - 2\pi,  -\pi]$.

\vspace{1em}
\noindent \textbf{Case 2: $\theta_i \in [-\pi,0]$}

\textbf{Case 2} is simply a mirror of \textbf{Case 1} and the same steps apply with the ``directions" reversed.

\subsection{Proof of Theorem 3 in the main text}
Here, we prove that $q_k \ge p_k$, where $p_k$ is the value of $P_{r}^{rig}$.

\begin{proof}
	At an abstract level, we need to establish that the feasible domain of~$P_{k}^{rig}$ is a subset of the feasible domain of~$Q_{k}^{rig}$. To do this practically, we need to show that the constraints of~$P_{k}^{rig}$ are also satisfied under~$Q_{k}^{rig}$ for all $\bT_k$.
	
	Let $\bT_k=(\bR_k, \bt_k)$ be an arbitrary rigid transformation with consensus set $\cI_k$ which satisfies the constraints in~$P_{k}^{rig}$. By triangle inequality
	\begin{align}
	\| \bx - \by \| \le \| \bx \| + \| \by \|,
	\end{align}
	we have that
	\begin{align}
	\begin{aligned}
	&\| (\bR_k \bx_i + \bt_k - \by_i) - (\bR_k \bx_k + \bt_k  - \by_k) \| \\ 
	&= \| \bR_k \bx^{(k)}_i - \by^{(k)}_i \| \\ 
	&\leq \| \bR_k \bx_i + \bt_k -\by_i \| + \| \bR_k \bx_k + \bt_k -\by_k \| \leq 2\xi,
	\end{aligned}
	\end{align}
	thus the constraints in~$Q_{k}^{rig}$ are also satisfied.
\end{proof}

\renewcommand{\thesection}{B} 
\section{Additional experimental results}
For a better visualization and to show all used scans in Sec.~5.1,  qualitative results are plotted in Figs.~\ref{fig:6dofresultsstanford}, \ref{fig:6dofresultsmian}, \ref{fig:6dofresultsmining}, \ref{fig:6dofresultsrs}, for the Stanford, Mian, mining and remote sensing datasets.

Table~\ref{tab:comparison6dof} presents individual results per every used object for 6 DoF Euclidean registration (results in the main paper are aggregated by dataset). If a method could not terminate successfully within the time limit, the result was marked with a `-' in the table. Analogously, Table~\ref{tab:comparison6dofrgbd} lists the results for each RGB-D instance. 

\begin{figure*}
	\begin{tabularx}{\textwidth}{c|C{.7}|C{.7}|C{.3}}
		\footnotesize
		& {\footnotesize Input correspondences $\cH$} & {\footnotesize Remaining data $\cH^\prime$ after GORE} & {\footnotesize Registration using approximate solution $\tilde{\bT}$ from GORE}\\ 
		\hline
		&&&\\
		\rotatebox[origin=l]{90}{\emph{bunny}}&
		\includegraphics[height=3cm]{6dof/bunny_500_input} &
		\includegraphics[height=3cm]{6dof/bunny_500_gore_out} &
		\includegraphics[height=3cm]{6dof/bunny_gore_align} 
		\\
		\rotatebox[origin=l]{90}{\emph{dragon}}&
		\includegraphics[height=2cm]{6dof/dragon_500_input}&
		\includegraphics[height=2cm]{6dof/dragon_500_gore_out}&
		\includegraphics[height=2.5cm]{6dof/dragon_gore_align}
		\\
\rotatebox[origin=l]{90}{\footnotesize\emph{armadillo}}&
				\includegraphics[trim={2cm 0 0 0},clip,height=2.5cm]{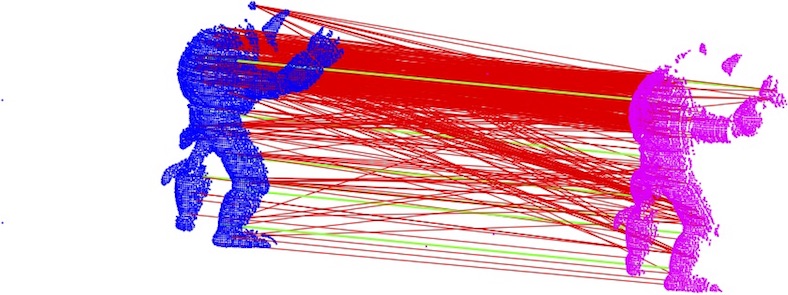}&
				\includegraphics[trim={2cm 0 0 0},clip,height=2.5cm]{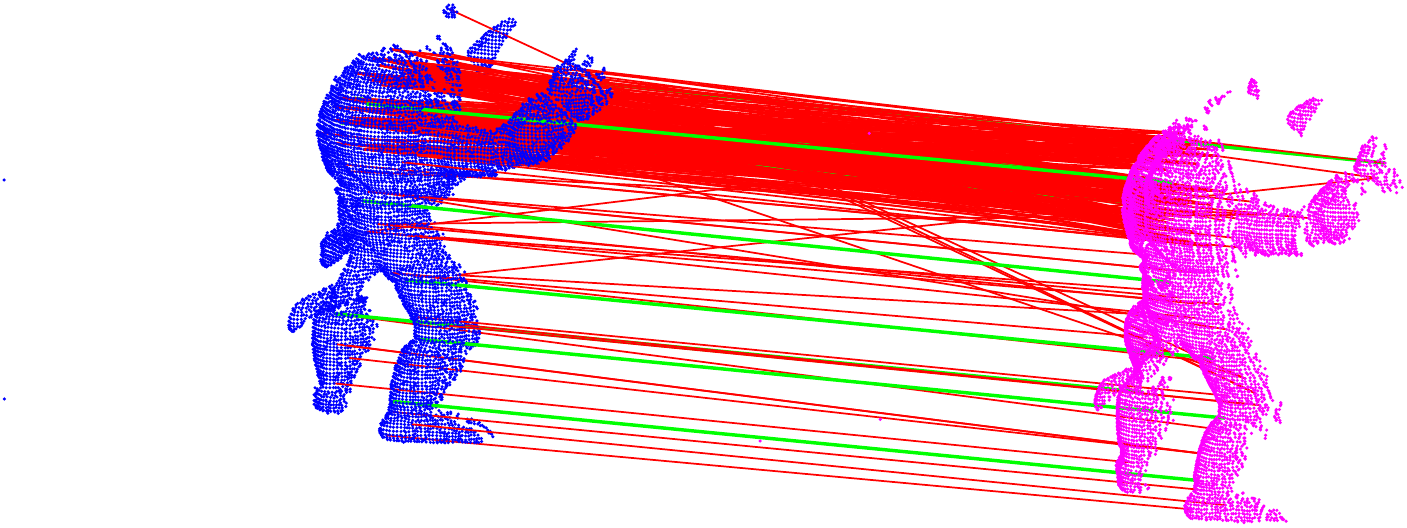}&
				\includegraphics[ trim={10cm 0 0 0},clip, height=3cm]{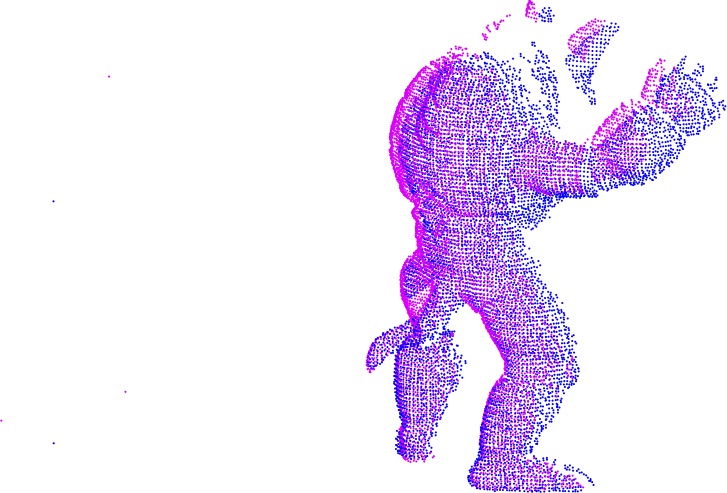}
				\\
\rotatebox[origin=l]{90}{\footnotesize\emph{buddha}}&
				\includegraphics[height=4cm]{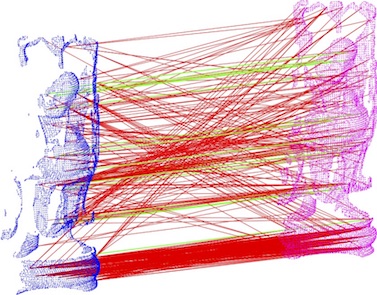}&
				\includegraphics[height=4cm]{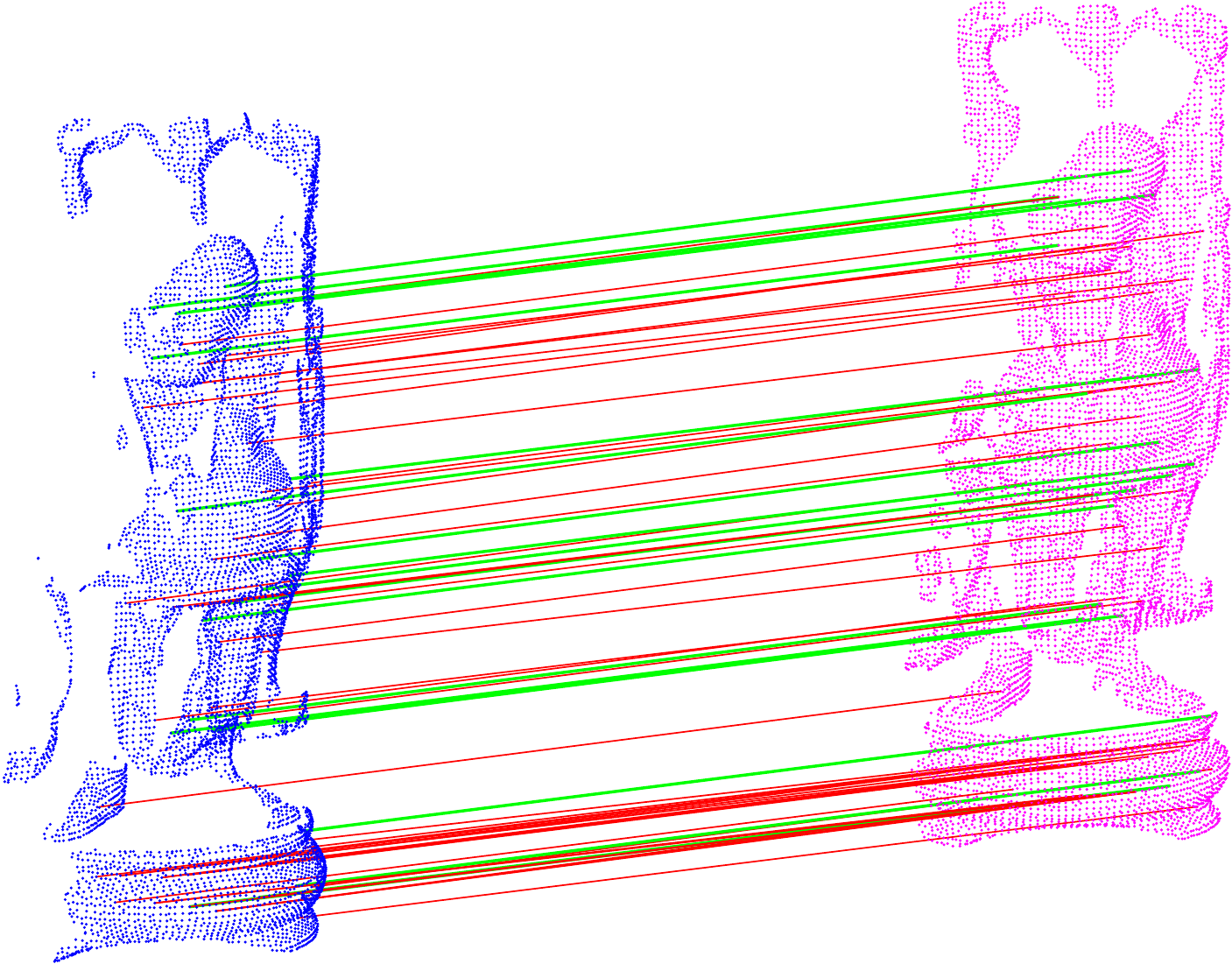}&
				\includegraphics[height=4cm]{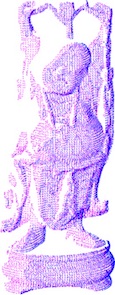}
				\\
\end{tabularx}	
	
	\caption{Qualitative results of GORE for 6 DoF Euclidean registration with $N=500$ on the Stanford dataset. Column 1: Input correspondences (true inliers are represented by green lines, and true outliers by red lines). Column 2: Data remaining after GORE. Column 3: Registration using approximate solution $\tilde{\bT}$ produced by GORE. }
	\label{fig:6dofresultsstanford}	
\end{figure*}

\begin{figure*}
	\begin{tabularx}{\textwidth}{c|C{.9}|C{.8}|C{.5}}
		\footnotesize
		& {\footnotesize Input correspondences $\cH$} & {\footnotesize Remaining data $\cH^\prime$ after GORE} & {\footnotesize Registration using approximate solution $\tilde{\bT}$ from GORE}\\ 
		\hline
		&&&\\
\rotatebox[origin=l]{90}{ \emph{t-rex}}&
		\includegraphics[height=3.3cm]{6dof/trex_500_input}&
		\includegraphics[height=3.3cm]{6dof/trex_500_gore_out}&
		\includegraphics[height=3cm]{6dof/trex_gore_align}
		\\
\rotatebox[origin=l]{90}{\emph{parasauro}}&
		\includegraphics[height=3.5cm]{6dof/parasauro_500_input}&
		\includegraphics[height=3.5cm]{6dof/parasauro_500_gore_out}&
		\includegraphics[height=3cm]{6dof/parasauro_gore_align}
		\\
\rotatebox[origin=l]{90}{\footnotesize\emph{chef}}&
				\includegraphics[height=3cm]{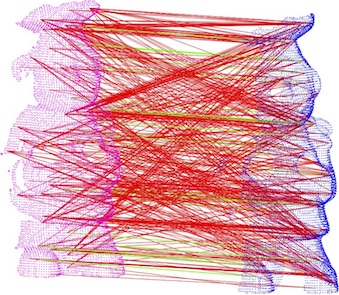}&
				\includegraphics[height=3cm]{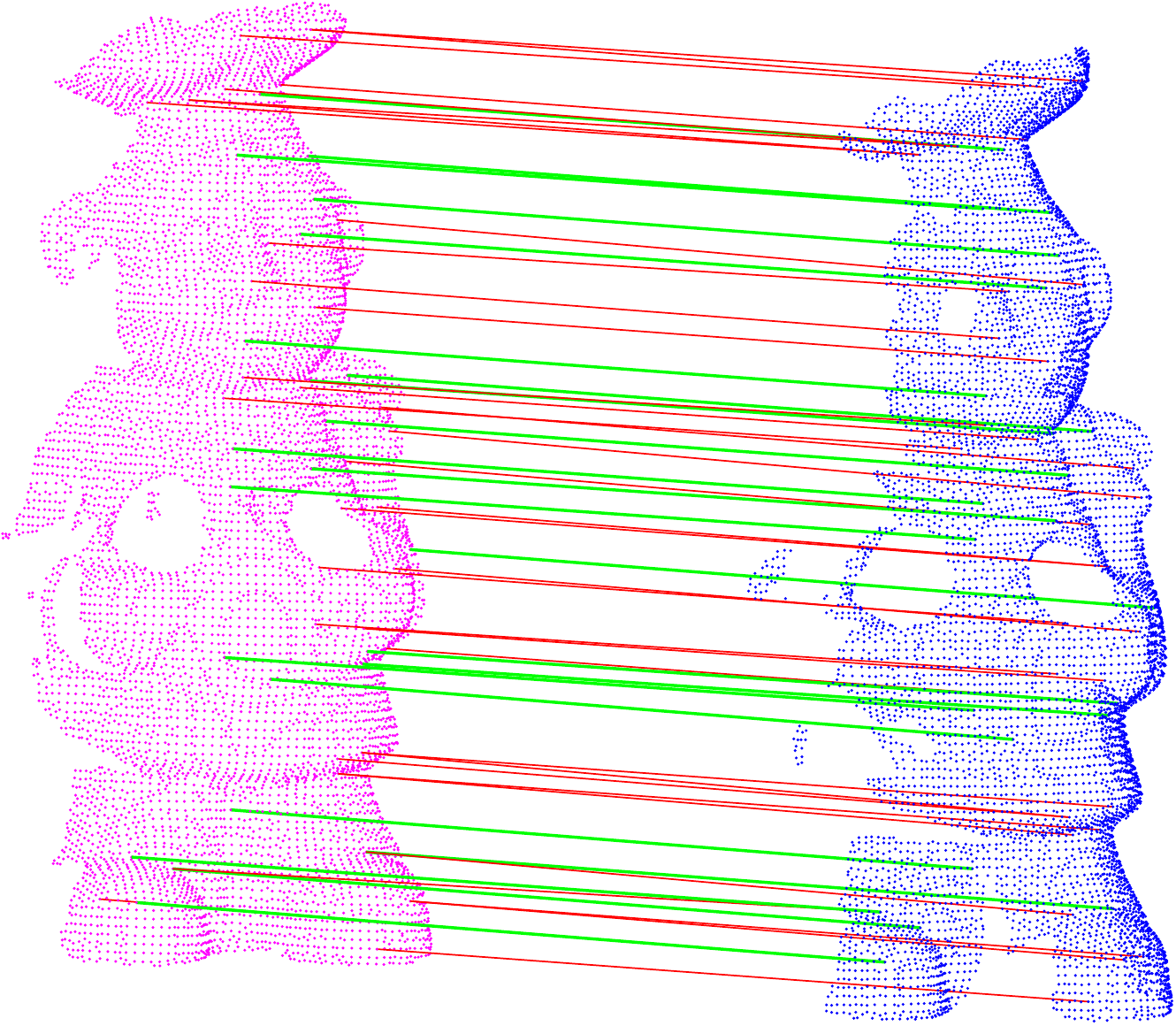}&
				\includegraphics[height=3cm]{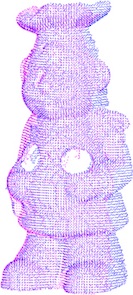}
				\\
\rotatebox[origin=l]{90}{\footnotesize\emph{chicken}}&
				\includegraphics[height=3.5cm]{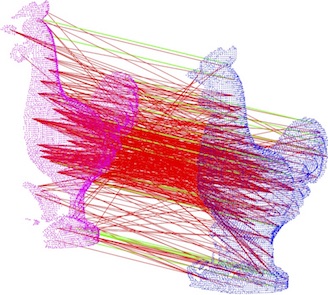}&
				\includegraphics[height=3.5cm]{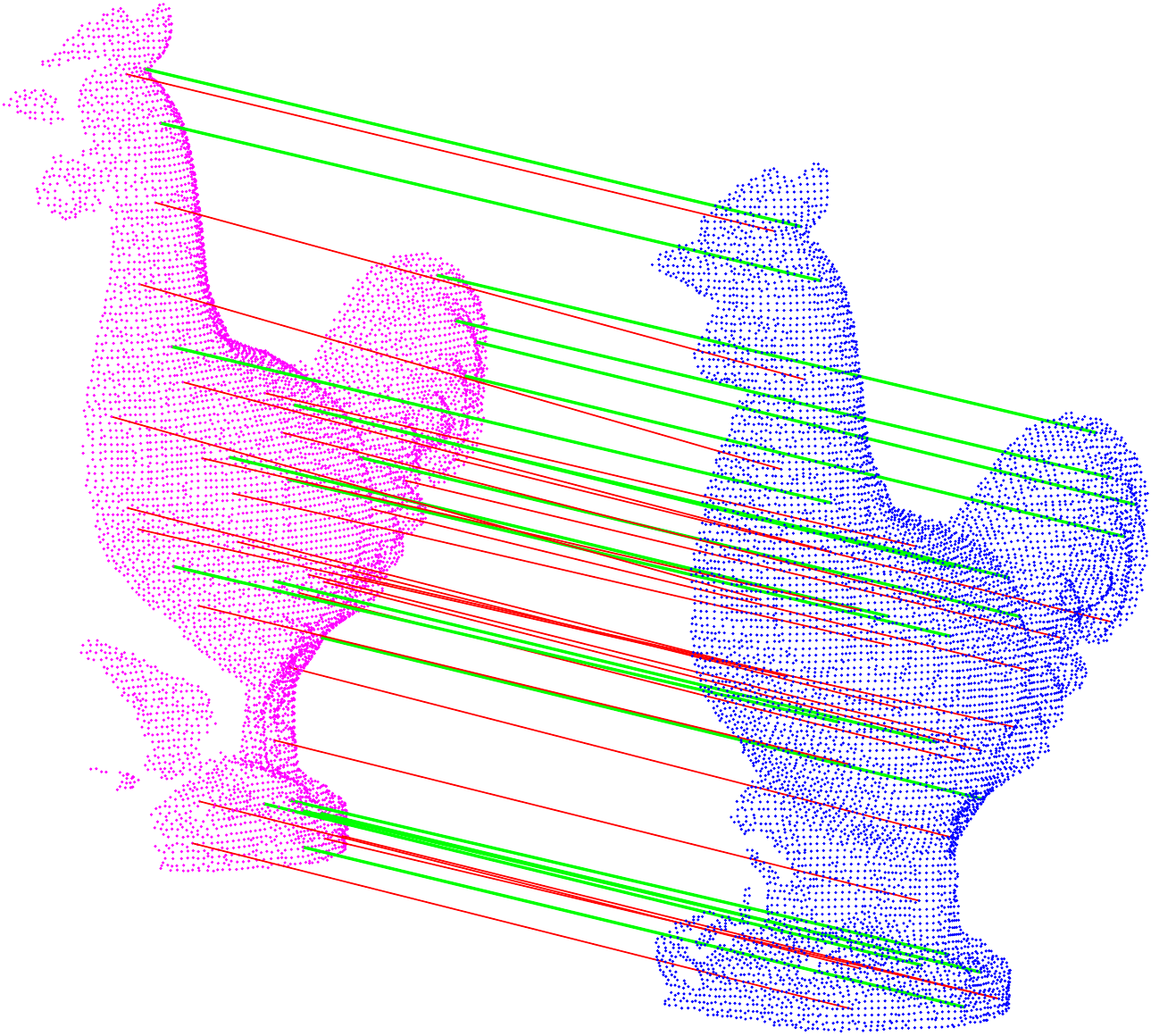}&
				\includegraphics[height=3cm]{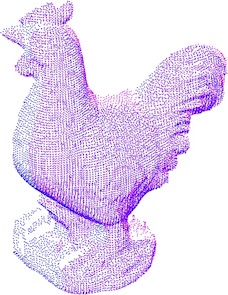}
				\\
		
	\end{tabularx}	
	
	\caption{Qualitative results of GORE for 6 DoF Euclidean registration with $N=500$ on the Mian dataset. Column 1: Input correspondences (true inliers are represented by green lines, and true outliers by red lines). Column 2: Data remaining after GORE. Column 3: Registration using approximate solution $\tilde{\bT}$ produced by GORE. }
	\label{fig:6dofresultsmian}	
\end{figure*}

\begin{figure*}
\begin{tabularx}{\textwidth}{c|C{.8}|C{.8}|C{.5}}
		\footnotesize
		& {\footnotesize Input correspondences $\cH$} & {\footnotesize Remaining data $\cH^\prime$ after GORE} & {\footnotesize Registration using approximate solution $\tilde{\bT}$ from GORE}\\ 
		\hline
&&&\\
\rotatebox[origin=l]{90}{\footnotesize\emph{mining-a}}&
		\includegraphics[height=3cm]{6dof/mining1-2_500_input}&
		\includegraphics[height=3cm]{6dof/mining1-2_500_gore_out}&
		\includegraphics[height=3cm]{6dof/mining1-2_gore_align}
		\\
&&&\\
\rotatebox[origin=l]{90}{\emph{mining-b}}&
		\includegraphics[height=3.7cm]{6dof/mining2-3_500_input}&
		\includegraphics[height=3.7cm]{6dof/mining2-3_500_gore_out}&
		\includegraphics[height=3.7cm]{6dof/mining2-3_gore_align}
		\\
&&&\\
\rotatebox[origin=l]{90}{\footnotesize\emph{mining-c}}&
				\includegraphics[height=2.3cm]{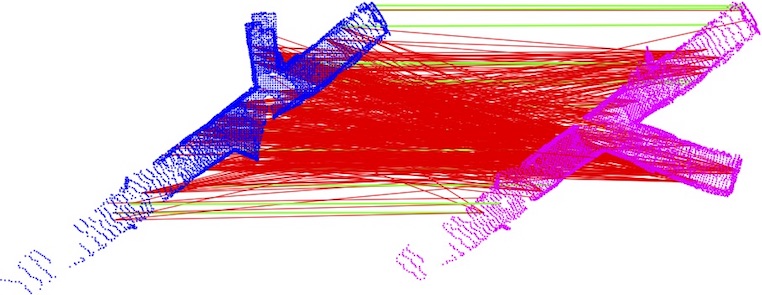}&
				\includegraphics[height=2.3cm]{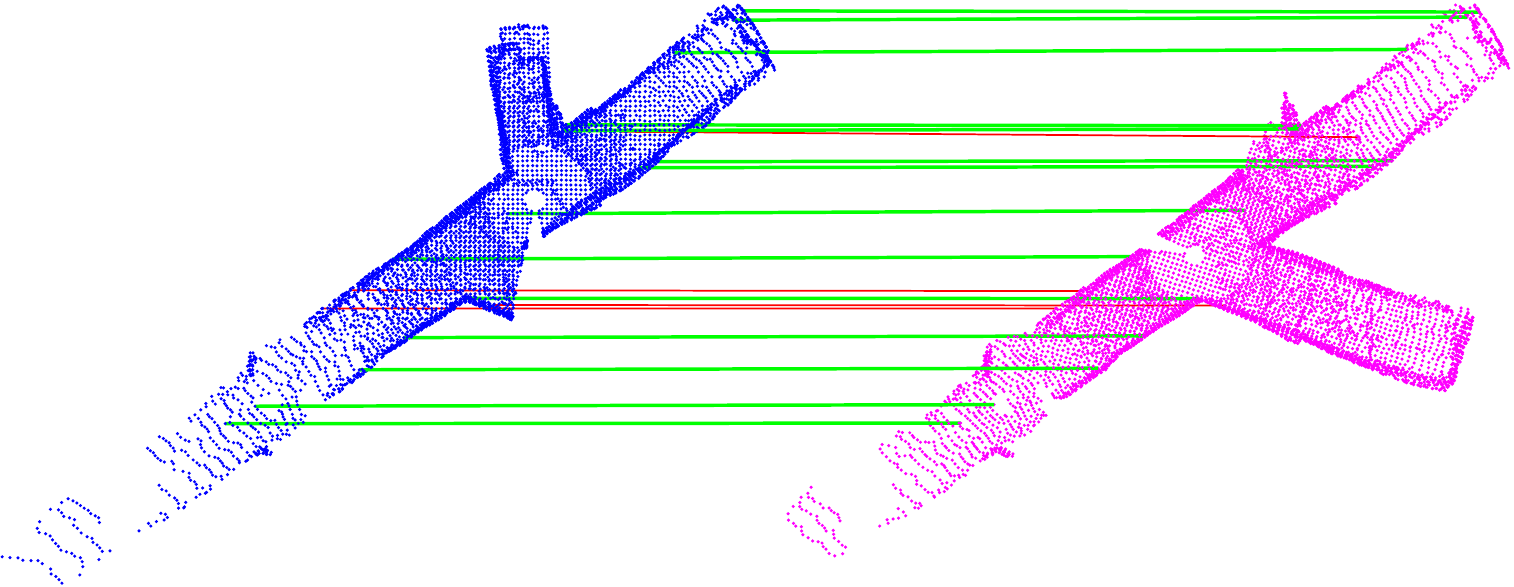}&
				\includegraphics[height=2.4cm]{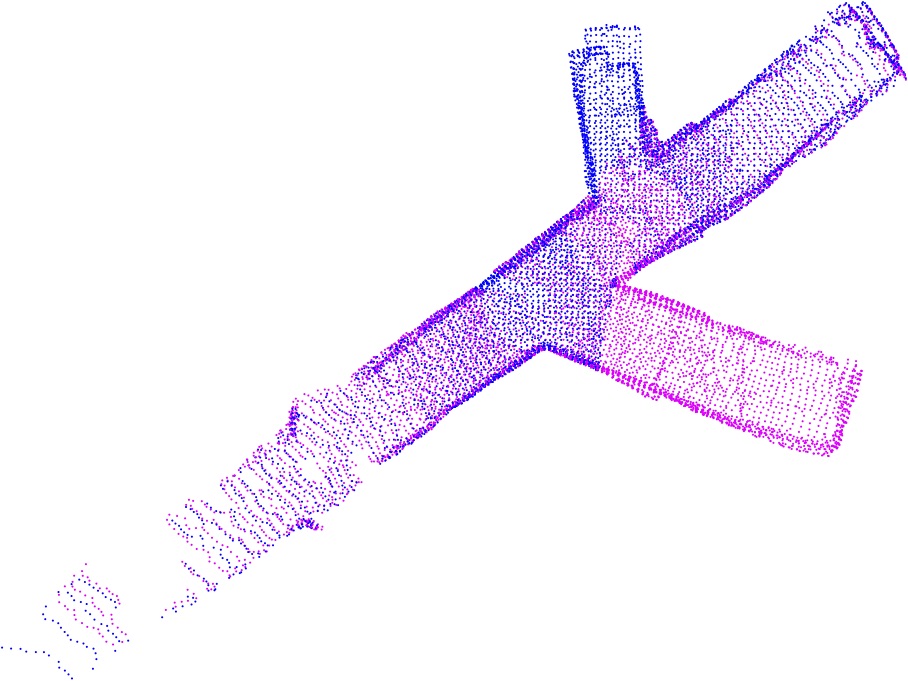}
				\\
\end{tabularx}	
	
	\caption{Qualitative results of GORE for 6 DoF Euclidean registration with $N=500$ on the mining dataset. Column 1: Input correspondences (true inliers are represented by green lines, and true outliers by red lines). Column 2: Data remaining after GORE. Column 3: Registration using approximate solution $\tilde{\bT}$ produced by GORE. }
	\label{fig:6dofresultsmining}	
\end{figure*}

\begin{figure*}
	\begin{tabularx}{\textwidth}{c|C{.6}|C{.6}|C{.7}}
		& {\footnotesize Input correspondences $\cH$} & {\footnotesize Remaining data $\cH^\prime$ after GORE} & {\footnotesize Registration using approximate solution $\tilde{\bT}$ from GORE}\\ 
		\hline
&&&\\
		\rotatebox[origin=l]{90}{\emph{vaihingen-a}}&
		\includegraphics[height=5cm]{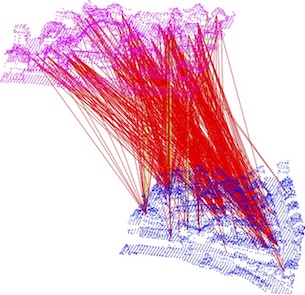}&
		\includegraphics[height=5cm]{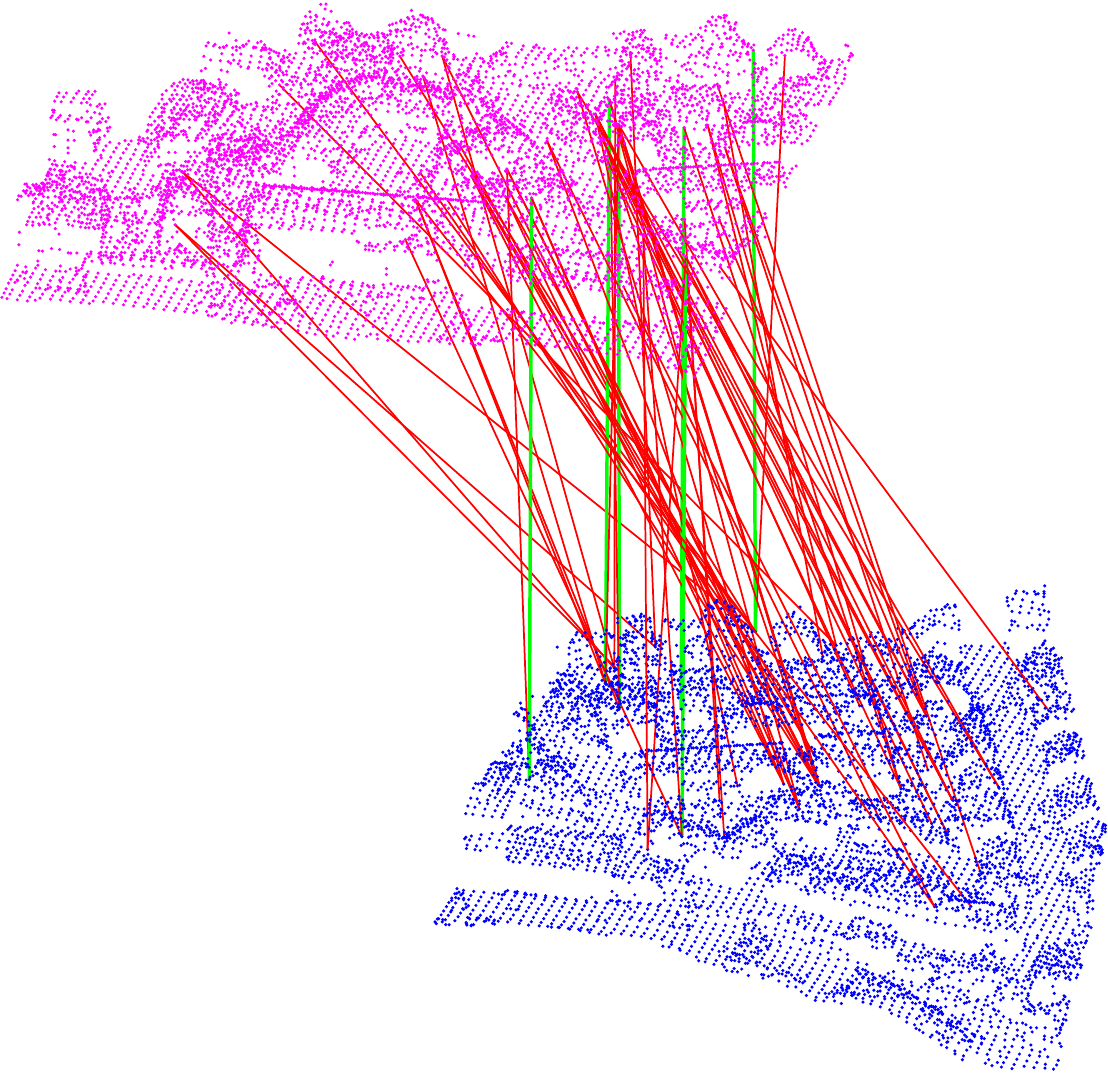}&
		\includegraphics[height=2.5cm]{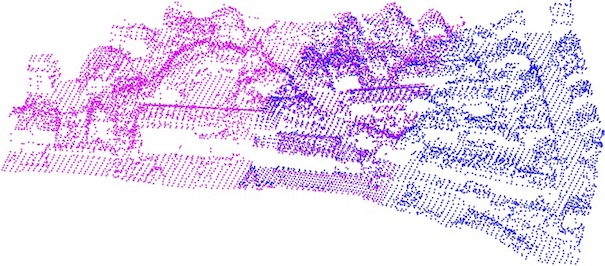}
		\\
		\rotatebox[origin=l]{90}{\emph{vaihingen-b}}&
		\includegraphics[height=5cm]{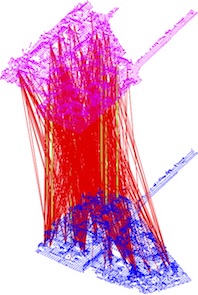}&
		\includegraphics[height=5cm]{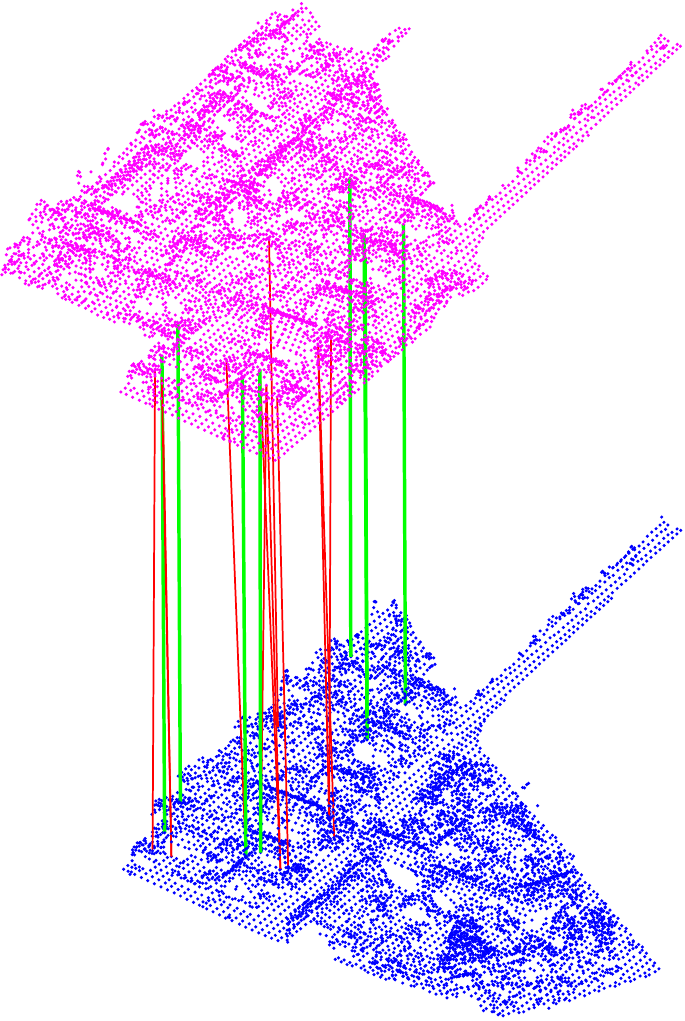}&
		\includegraphics[height=4cm]{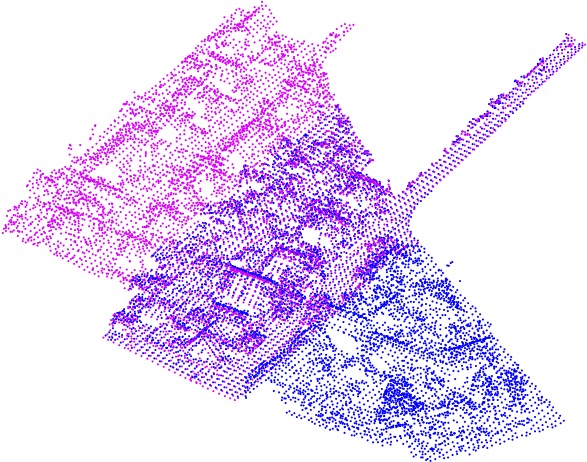} 
		\\
	\end{tabularx}	
	
	\caption{Qualitative results of GORE for 6 DoF Euclidean registration with $N=500$ on remote sensing data. Column 1: Input correspondences (true inliers are represented by green lines, and true outliers by red lines). Column 2: Data remaining after GORE. Column 3: Registration using approximate solution $\tilde{\bT}$ produced by GORE. }
	\label{fig:6dofresultsrs}	
\end{figure*}

\begin{table*} 
	\renewcommand{\arraystretch}{1.2}
	\resizebox{\textwidth}{!}{
		\begin{tabular}{|l|rr|rrrrr|rrrr|r|rrrr|rrrr|}
			\cline{1-21}
			\multicolumn{1}{|c|}{\multirow{3}{*}{Object}} & 
			\multicolumn{1}{c}  {\multirow{3}{*}{$N$}} & 
			\multicolumn{1}{c|} {\multirow{3}{*}{$\eta$}} &
			\multicolumn{5}{c|} {\multirow{2}{*}{GORE}} &
			\multicolumn{4}{c|} {\multirow{2}{*}{GORE+BnB}} &
			\multicolumn{1}{c|} {\multirow{2}{*}{BnB}} &
			\multicolumn{4}{c|} {\multirow{2}{*}{GORE+RANSAC}} &
			\multicolumn{4}{c|} {\multirow{2}{*}{RANSAC}}
			\\
			& & & 
			& & & & & 
			& & & &   
			&         
			& & & &   
			& & &     
			\\[-.5em]
			& & 
			& \multicolumn{1}{c}{$|\cI|$}   
			& \multicolumn{1}{c}{angErr} 
			& \multicolumn{1}{c}{trErr} 
			& \multicolumn{1}{c}{$|\cH^\prime|$} 
			& \multicolumn{1}{c|}{time}     
			& \multicolumn{1}{c}{$|\cI^\ast|$} 
			& \multicolumn{1}{c}{angErr} 
			& \multicolumn{1}{c}{trErr} 
			& \multicolumn{1}{c|}{time}     
			& \multicolumn{1}{c|}{time}  	
			& \multicolumn{1}{c}{$|\cI|$} 
			& \multicolumn{1}{c}{angErr} 
			& \multicolumn{1}{c}{trErr} 
			& \multicolumn{1}{c|}{time}     
			& \multicolumn{1}{c}{$|\cI|$} 
			& \multicolumn{1}{c}{angErr} 
			& \multicolumn{1}{c}{trErr} 
			& \multicolumn{1}{c|}{time}     
			\\
			& &  
			& & \multicolumn{1}{c}{(\textdegree)} & & & \multicolumn{1}{c|}{ (s)} 
			& & \multicolumn{1}{c}{(\textdegree)} & & \multicolumn{1}{c|}{ (s)}   
			& \multicolumn{1}{c|}{(s)}          
			& & \multicolumn{1}{c}{(\textdegree)} & & \multicolumn{1}{c|}{(s)}    
			& & \multicolumn{1}{c}{(\textdegree)} & & \multicolumn{1}{c|}{(s)}    
			\\
			\hline
			
			\multirow{3}{*}{{\parbox[c][4em][c]{1.5cm} {\emph{bunny}\\
						$|\cX|=7776$\\
						$|\cY|=7439$ }}}
			& 500    & 0.92   & 35     & 0.42   & 0.20   & 56     & 0.860  & 40     & 0.54   & 0.17   & 11782.48 & 14583.60 & 35     & 0.49   & 0.12   & 0.866  & 34     & 0.52   & 0.14   & 2.18   \\ 
			& 1000   & 0.96   & 37     & 0.52   & 0.16   & 59     & 1.362  & 43     & 0.47   & 0.15   & 11567.52 & 17866.86 & 37     & 0.37   & 0.11   & 1.368  & 37     & 0.39   & 0.04   & 14.87  \\ 
			& 2000   & 0.98   & 37     & 0.56   & 0.13   & 60     & 4.618  & 43     & 0.48   & 0.14   & 13984.68 & -      & 37     & 0.44   & 0.10   & 4.623  & 37     & 0.31   & 0.10   & 131.61 \\ 
						\hline 
						\multirow{3}{*}{{\parbox[c][4em][c]{1.5cm} {\emph{armadillo}\\
									$|\cX|=7558$\\
									$|\cY|=7027$ }}}
						& 500    & 0.98   & 5      & 1.22   & 0.26   & 343    & 38.161 & 9      & 0.61   & 0.30   & 1172.16 & 1313.25 & 7      & 0.64   & 0.34   & 91.751 & 7      & 2.19   & 0.36   & 219.22 \\ 
						& 1000   & 0.98   & 12     & 1.20   & 0.19   & 70     & 8.902  & 17     & 0.61   & 0.27   & 296.27 & 820.79 & 13     & 0.85   & 0.18   & 9.056  & 12     & 0.77   & 0.26   & 358.28 \\ 
						& 2000   & 0.99   & 13     & 0.92   & 0.18   & 105    & 49.397 & 19     & 0.59   & 0.28   & 706.96 & 2196.23 & 15     & 0.99   & 0.25   & 49.720 & 14     & 0.33   & 0.07   & 2136.27 \\ 
			\hline 
			\multirow{3}{*}{{\parbox[c][4em][c]{1.5cm} {\emph{dragon}\\
						$|\cX|=7774$\\
						$|\cY|=6875$ }}}
			& 500    & 0.94   & 25     & 0.54   & 0.19   & 50     & 1.807  & 32     & 0.23   & 0.13   & 122.90 & 261.02 & 27     & 0.24   & 0.17   & 1.817  & 27     & 0.34   & 0.14   & 4.04   \\ 
			& 1000   & 0.97   & 27     & 0.53   & 0.18   & 53     & 2.813  & 34     & 0.23   & 0.13   & 116.43 & 355.15 & 30     & 0.25   & 0.15   & 2.822  & 28     & 0.17   & 0.25   & 40.95  \\ 
			& 2000   & 0.98   & 28     & 0.50   & 0.19   & 57     & 4.770  & 34     & 0.24   & 0.13   & 286.85 & 1150.63 & 28     & 0.31   & 0.21   & 4.781  & 30     & 0.41   & 0.08   & 245.01 \\ 
			\hline 
			\multirow{3}{*}{{\parbox[c][4em][c]{1.5cm} {\emph{buddha}\\
						$|\cX|=7761$\\
						$|\cY|=7317$ }}}
			& 500    & 0.96   & 12     & 1.26   & 0.41   & 71     & 19.533 & 19     & 0.81   & 0.40   & 1448.80 & 3586.01 & 14     & 1.66   & 0.28   & 19.613 & 15     & 0.65   & 0.43   & 26.39  \\ 
			& 1000   & 0.98   & 13     & 1.54   & 0.25   & 79     & 25.527 & 20     & 0.83   & 0.39   & 2001.45 & 5298.21 & 15     & 2.02   & 0.30   & 25.614 & 17     & 1.05   & 0.43   & 127.38 \\ 
			& 2000   & 0.99   & 15     & 0.96   & 0.24   & 87     & 37.855 & 22     & 2.23   & 0.32   & 17615.56 & -      & 18     & 1.38   & 0.29   & 37.972 & 18     & 1.46   & 0.42   & 1171.60 \\ 
			\hline

			\multirow{3}{*}{{\parbox[c][4em][c]{1.5cm} {\emph{t-rex}\\
						$|\cX|=7862$\\
						$|\cY|=8303$ }}}
			& 500    & 0.94   & 24     & 0.40   & 0.08   & 40     & 0.788  & 28     & 0.40   & 0.06   & 145.03 & 298.14 & 24     & 0.28   & 0.10   & 0.796  & 25     & 0.15   & 0.05   & 5.93   \\ 
			& 1000   & 0.97   & 27     & 0.44   & 0.14   & 47     & 1.721  & 31     & 0.41   & 0.05   & 230.63 & 583.91 & 27     & 0.27   & 0.14   & 1.728  & 28     & 0.47   & 0.16   & 46.37  \\ 
			& 2000   & 0.98   & 28     & 0.40   & 0.14   & 49     & 3.493  & 32     & 0.41   & 0.05   & 259.95 & 1207.55 & 29     & 0.27   & 0.12   & 3.500  & 29     & 0.19   & 0.17   & 340.77 \\ 
			\hline 
			\multirow{3}{*}{{\parbox[c][4em][c]{1.5cm} {\emph{parasauro}\\
						$|\cX|=7547$\\
						$|\cY|=6782$ }}}
			& 500    & 0.96   & 17     & 0.84   & 0.06   & 46     & 4.314  & 21     & 0.45   & 0.07   & 386.74 & 708.96 & 16     & 0.62   & 0.15   & 4.341  & 17     & 0.58   & 0.09   & 14.79  \\ 
			& 1000   & 0.98   & 18     & 0.84   & 0.09   & 57     & 8.004  & 22     & 0.46   & 0.07   & 621.52 & 1536.69 & 17     & 0.94   & 0.13   & 8.038  & 18     & 0.43   & 0.31   & 150.96 \\ 
			& 2000   & 0.99   & 19     & 0.83   & 0.05   & 58     & 16.216 & 23     & 0.46   & 0.07   & 740.10 & 2521.50 & 18     & 0.66   & 0.12   & 16.246 & 18     & 0.83   & 0.14   & 1209.45 \\ 
						\hline 
						\multirow{3}{*}{{\parbox[c][4em][c]{1.5cm} {\emph{chef}\\
									$|\cX|=7811$\\
									$|\cY|=7698$ }}}
						& 500    & 0.96   & 16     & 1.71   & 0.68   & 57     & 5.529  & 22     & 0.85   & 0.56   & 8658.81 & 13717.03 & 18     & 1.54   & 0.71   & 5.563  & 18     & 1.74   & 0.77   & 17.12  \\ 
						& 1000   & 0.98   & 17     & 1.28   & 0.56   & 62     & 5.809  & 23     & 0.90   & 0.57   & 8665.13 & 14060.27 & 19     & 1.56   & 0.73   & 5.836  & 20     & 1.80   & 0.64   & 107.41 \\ 
						& 2000   & 0.99   & 20     & 1.06   & 0.56   & 65     & 8.120  & 25     & 1.86   & 0.60   & 10175.85 & 17813.93 & 20     & 1.57   & 0.71   & 8.149  & 20     & 1.76   & 0.86   & 750.75 \\ 
						\hline 
			\multirow{3}{*}{{\parbox[c][4em][c]{1.5cm} {\emph{chicken}\\
						$|\cX|=7803$\\
						$|\cY|=8333$ }}}
			& 500    & 0.96   & 16     & 0.42   & 0.14   & 46     & 2.609  & 20     & 0.71   & 0.34   & 557.54 & 1168.93 & 17     & 0.42   & 0.18   & 2.623  & 16     & 0.46   & 0.13   & 21.53  \\ 
			& 1000   & 0.98   & 18     & 0.46   & 0.37   & 56     & 6.853  & 21     & 0.76   & 0.35   & 889.70 & 2227.40 & 18     & 0.42   & 0.21   & 6.885  & 17     & 0.20   & 0.10   & 126.27 \\ 
			& 2000   & 0.99   & 16     & 0.47   & 0.18   & 70     & 12.865 & 21     & 0.76   & 0.36   & 1153.00 & 4305.64 & 17     & 0.39   & 0.24   & 12.923 & 18     & 0.64   & 0.42   & 1030.21 \\ 
			
			\hline 
			\multirow{3}{*}{{\parbox[c][4em][c]{1.5cm} {\emph{mining-a}\\
						$|\cX|=7433$\\
						$|\cY|=7355$ }}}
			& 500    & 0.98   & 9      & 0.53   & 0.14   & 16     & 2.669  & 12     & 1.96   & 0.16   & 15.05  & 42.59  & 10     & 0.48   & 0.14   & 2.676  & 8      & 1.87   & 0.20   & 140.92 \\ 
			& 1000   & 0.99   & 10     & 0.83   & 0.14   & 18     & 9.885  & 12     & 1.97   & 0.16   & 29.99  & 105.16 & 10     & 0.62   & 0.15   & 9.893  & 10     & 0.27   & 0.13   & 633.52 \\ 
			& 2000   & 0.99   & 12     & 0.89   & 0.14   & 21     & 27.076 & 14     & 2.03   & 0.16   & 50.16  & 356.89 & 12     & 0.69   & 0.14   & 27.089 & 12     & 1.90   & 0.13   & 4405.42 \\ 
			\hline 
			\multirow{3}{*}{{\parbox[c][4em][c]{1.5cm} {\emph{mining-b}\\
						$|\cX|=7667$\\
						$|\cY|=5450$ }}}
			& 500    & 0.98   & 6      & 5.34   & 0.35   & 20     & 3.546  & 10     & 3.69   & 0.38   & 5.21   & 13.11  & 8      & 5.21   & 0.42   & 3.581  & 8      & 5.14   & 0.43   & 141.34 \\ 
			& 1000   & 0.99   & 6      & 4.91   & 0.33   & 33     & 20.328 & 10     & 3.72   & 0.36   & 22.00  & 56.85  & 8      & 5.46   & 0.42   & 20.380 & 8      & 5.06   & 0.44   & 1811.11 \\ 
			& 2000   & 0.99   & 7      & 5.94   & 0.33   & 38     & 98.255 & 11     & 3.74   & 0.36   & 175.37 & 497.49 & 9      & 4.76   & 0.40   & 98.327 & 9      & 4.10   & 0.49   & 7989.53 \\ 
			\hline 
						\multirow{3}{*}{{\parbox[c][4em][c]{1.5cm} {\emph{mining-c}\\
									$|\cX|=7679$\\
									$|\cY|=6452$ }}}
						& 500    & 0.97   & 11     & 0.84   & 0.24   & 18     & 0.978  & 14     & 1.55   & 0.15   & 10.91  & 43.24  & 12     & 0.95   & 0.18   & 0.985  & 10     & 1.21   & 0.11   & 72.00  \\ 
						& 1000   & 0.98   & 12     & 0.97   & 0.23   & 26     & 3.325  & 16     & 1.55   & 0.14   & 35.09  & 135.61 & 13     & 0.68   & 0.15   & 3.338  & 13     & 0.39   & 0.21   & 289.40 \\ 
						& 2000   & 0.99   & 13     & 0.99   & 0.24   & 35     & 13.537 & 17     & 1.48   & 0.14   & 288.96 & 958.63 & 14     & 0.58   & 0.17   & 13.563 & 15     & 0.53   & 0.14   & 2275.94 \\ 
						\hline

			\multirow{3}{*}{{\parbox[c][4em][c]{1.5cm} {\emph{vaihingen-a}\\
						$|\cX|=7786$\\
						$|\cY|=7190$ }}}
			& 500    & 0.99   & 5      & 1.00   & 0.36   & 75     & 10.756 & 6      & 2.02   & 0.84   & 13.19  & 18.15  & 5      & 1.47   & 0.61   & 12.784 & 5      & 1.42   & 0.46   & 575.37 \\ 
			& 1000   & 0.99   & 12     & 1.68   & 0.65   & 14     & 1.402  & 13     & 1.22   & 0.34   & 1.86   & 100.54 & 12     & 0.65   & 0.24   & 1.404  & 12     & 1.10   & 0.50   & 369.55 \\ 
			& 2000   & 0.99   & 18     & 0.73   & 0.34   & 20     & 10.631 & 20     & 0.53   & 0.38   & 10.79  & 489.81 & 19     & 0.52   & 0.19   & 10.634 & 19     & 1.31   & 0.51   & 866.33 \\ 
			\hline 
			\multirow{3}{*}{{\parbox[c][4em][c]{1.5cm} {\emph{vaihingen-b}\\
						$|\cX|=7775$\\
						$|\cY|=7706$ }}}
			& 500    & 0.98   & 7      & 1.70   & 0.43   & 20     & 3.375  & 8      & 1.48   & 0.45   & 15.92  & 43.59  & 7      & 1.00   & 0.24   & 3.407  & 7      & 0.54   & 0.12   & 209.01 \\ 
			& 1000   & 0.99   & 10     & 1.27   & 0.31   & 22     & 2.318  & 12     & 2.53   & 0.58   & 6.76   & 98.72  & 10     & 1.17   & 0.29   & 2.326  & 11     & 0.47   & 0.31   & 465.22 \\ 
			& 2000   & 0.99   & 12     & 0.57   & 0.36   & 22     & 20.325 & 14     & 1.58   & 0.52   & 35.57  & 574.45 & 12     & 2.06   & 0.43   & 20.332 & 12     & 0.96   & 0.15   & 3596.55 \\ 
			\hline 
			
		\end{tabular}
	} 
	\vspace{.5em}
	\caption{6 DoF Euclidean registration results.}
	\label{tab:comparison6dof}
\end{table*}

\begin{table*} 
		\renewcommand{\arraystretch}{1.2}
	\resizebox{\textwidth}{!}{
		\begin{tabular}{|l|rr|rrrrr|rrrr|r|rrrr|rrrr|}
			\cline{1-21}
			\multicolumn{1}{|c|}{\multirow{3}{*}{Object}} & 
			\multicolumn{1}{c}  {\multirow{3}{*}{$N$}} & 
			\multicolumn{1}{c|} {\multirow{3}{*}{$\eta$}} &
			\multicolumn{5}{c|} {\multirow{2}{*}{GORE}} &
			\multicolumn{4}{c|} {\multirow{2}{*}{GORE+BnB}} &
			\multicolumn{1}{c|} {\multirow{2}{*}{BnB}} &
			\multicolumn{4}{c|} {\multirow{2}{*}{GORE+RANSAC}} &
			\multicolumn{4}{c|} {\multirow{2}{*}{RANSAC}}
			\\
			& & & 
			& & & & & 
			& & & &   
			&         
			& & & &   
			& & &     
			\\[-.5em]
			& & 
			& \multicolumn{1}{c}{$|\cI|$}   
			& \multicolumn{1}{c}{angErr} 
			& \multicolumn{1}{c}{trErr} 
			& \multicolumn{1}{c}{$|\cH^\prime|$} 
			& \multicolumn{1}{c|}{time}     
			& \multicolumn{1}{c}{$|\cI^\ast|$} 
			& \multicolumn{1}{c}{angErr} 
			& \multicolumn{1}{c}{trErr} 
			& \multicolumn{1}{c|}{time}     
			& \multicolumn{1}{c|}{time}  	
			& \multicolumn{1}{c}{$|\cI|$} 
			& \multicolumn{1}{c}{angErr} 
			& \multicolumn{1}{c}{trErr} 
			& \multicolumn{1}{c|}{time}     
			& \multicolumn{1}{c}{$|\cI|$} 
			& \multicolumn{1}{c}{angErr} 
			& \multicolumn{1}{c}{trErr} 
			& \multicolumn{1}{c|}{time}     
			\\
			& &  
			& & \multicolumn{1}{c}{(\textdegree)} & & & \multicolumn{1}{c|}{ (s)} 
			& & \multicolumn{1}{c}{(\textdegree)} & & \multicolumn{1}{c|}{ (s)}   
			& \multicolumn{1}{c|}{(s)}          
			& & \multicolumn{1}{c}{(\textdegree)} & & \multicolumn{1}{c|}{(s)}    
			& & \multicolumn{1}{c}{(\textdegree)} & & \multicolumn{1}{c|}{(s)}    
			\\
			\hline

			
			\multirow{3}{*}{{\parbox[c][4em][c]{1.8cm} {\emph{office}\\
						$|\cX|=262242$\\
						$|\cY|=275943$ }}}
			& 500    & 0.96   & 17     & 1.88   & 1.32   & 49     & 1.941  & 22     & 2.88   & 0.91   & 1619.00 & 2455.34 & 17     & 3.04   & 0.93   & 1.951  & 20     & 2.44   & 0.90   & 11.53  \\ 
			& 1000   & 0.97   & 27     & 2.03   & 1.64   & 97     & 8.770  & 31     & 1.38   & 1.66   & 4677.07 & 7416.09 & 27     & 2.66   & 1.59   & 8.793  & 25     & 1.83   & 1.29   & 41.64  \\ 
			& 2000   & 0.98   & 32     & 2.19   & 1.32   & 149    & 27.576 & 41     & 2.18   & 1.36   & 7675.66 & 13737.20 & 33     & 2.78   & 0.83   & 27.626 & 32     & 2.01   & 1.44   & 184.51 \\ 
			\hline 
			\multirow{3}{*}{{\parbox[c][4em][c]{1.8cm} {\emph{kitchen}\\
						$|\cX|=273943$\\
						$|\cY|=268131$ }}}
			& 500    & 0.98   & 10     & 1.48   & 3.30   & 23     & 6.536  & 11     & 2.63   & 3.45   & 10.05  & 30.30  & 9      & 3.84   & 3.71   & 6.546  & 9      & 4.51   & 3.67   & 92.77  \\ 
			& 1000   & 0.98   & 13     & 2.77   & 3.38   & 55     & 16.536 & 17     & 3.31   & 3.40   & 87.89  & 208.61 & 14     & 3.79   & 3.46   & 16.570 & 15     & 3.68   & 3.41   & 173.84 \\ 
			& 2000   & 0.99   & 17     & 2.30   & 3.29   & 69     & 10.441 & 21     & 1.61   & 3.15   & 268.38 & 904.20 & 17     & 3.41   & 3.25   & 10.472 & 16     & 3.60   & 3.03   & 1250.42 \\ 
			\hline

			\multirow{3}{*}{{\parbox[c][4em][c]{1.8cm} {\emph{cap}\\
						$|\cX|=307200$\\
						$|\cY|=307200$ }}}
			& 500    & 0.98   & 9      & 18.05  & 0.42   & 19     & 0.271  & 11     & 8.83   & 0.19   & 8.32   & 10.96  & 10     & 18.24  & 0.39   & 0.279  & 10     & 7.69   & 0.08   & 67.17  \\ 
			& 1000   & 0.98   & 14     & 13.04  & 0.23   & 48     & 1.697  & 18     & 6.34   & 0.23   & 50.36  & 59.73  & 16     & 9.33   & 0.11   & 1.717  & 16     & 9.78   & 0.10   & 168.03 \\ 
		& 2000   & 0.99   & 24     & 14.77  & 0.28   & 38     & 5.754  & 28     & 11.80  & 0.42   & 605.19 & 949.31 & 24     & 12.36  & 0.25   & 5.759  & 24     & 10.92  & 0.28   & 468.36 \\ 
		
			\hline 
			\multirow{3}{*}{{\parbox[c][4em][c]{1.8cm} {\emph{cereal-box}\\
						$|\cX|=307200$\\
						$|\cY|=307200$ }}}
			& 500    & 0.95   & 24     & 5.06   & 0.17   & 25     & 0.237  & 25     & 5.25   & 0.14   & 0.25   & 6.00   & 24     & 6.94   & 0.15   & 0.239  & 24     & 5.43   & 0.13   & 5.56   \\ 
			& 1000   & 0.96   & 35     & 3.41   & 0.17   & 38     & 0.919  & 37     & 3.48   & 0.17   & 37.84  & 89.49  & 35     & 4.65   & 0.17   & 0.921  & 35     & 4.49   & 0.13   & 14.66  \\ 
			& 2000   & 0.97   & 68     & 3.31   & 0.20   & 71     & 2.342  & 69     & 4.25   & 0.14   & 51.58  & 169.70 & 68     & 2.69   & 0.15   & 2.344  & 68     & 2.94   & 0.15   & 19.16  \\ 
			\hline 
		\end{tabular}
	} 
	\vspace{.5em}
	\caption{ RGB-D registration results.}
	\label{tab:comparison6dofrgbd}
\end{table*}

\renewcommand{\thesection}{D} 
\section{Source Code}

Please download Matlab demo and implementation from 

\url{https://cs.adelaide.edu.au/~aparra/project/gore/sup.zip}

For the demo, Matlab with point cloud support is required. Please read the readme file first.

\bibliographystyle{IEEEtran}

\bibliography{gore.bib}